\documentclass[twoside,11pt,draft]{article}

\usepackage{jmlr2e}

\usepackage{amsmath}

\usepackage{amsthm}

\usepackage{cleveref}
\crefformat{footnote}{#2\footnotemark[#1]#3}
\usepackage{algorithm}
\usepackage{algpseudocode}

\newcommand{\termref}[1]{\textit{#1}}

\newcommand{\mathsc}[1]{\text{\normalfont\textsc{#1}}}
\newcommand{\relmiddle}[1]{\mathrel{}\middle#1\mathrel{}}

\newcommand{\bbNzero}{{\mathbb{N}_{0}}}
\newcommand{\bbN}{{\mathbb{N}}}
\newcommand{\bbR}{{\mathbb{R}}}

\newcommand{\modulo}{\mathbin{\mathrm{mod}}}

\newcommand{\veci}{\boldsymbol{i}}
\newcommand{\vecx}{\boldsymbol{x}}
\newcommand{\vecy}{\boldsymbol{y}}
\newcommand{\veclambda}{\boldsymbol{\lambda}}
\newcommand{\vecF}{\boldsymbol{F}}

\DeclareMathOperator{\id}{id}
\DeclareMathOperator{\head}{head}
\DeclareMathOperator{\tail}{tail}
\DeclareMathOperator{\src}{src}
\DeclareMathOperator{\snk}{snk}
\DeclareMathOperator{\op}{op}

\newcommand{\calG}{\mathcal{G}}
\newcommand{\BC}{\mathrm{BC}}

\newcommand{\zerothproj}{\mathsf{P}_0}
\newcommand{\firstproj}{\mathsf{P}_1}
\newcommand{\upperset}{\mathord{\uparrow}}
\newcommand{\lowerset}{\mathord{\downarrow}}
\newcommand{\AC}{\mathrm{AC}}

\newtheoremstyle{MyTheoremStyle}
  {\topsep}
  {\topsep}
  {\normalfont}
  {}
  {\bfseries}
  {}
  { }
  {}
\theoremstyle{MyTheoremStyle}

\newtheorem{MyDefinition}{Definition}
\newtheorem{MyLemma}{Lemma}
\newtheorem{MyTheorem}{Theorem}
\newtheorem{MyCorollary}{Corollary}
\newtheorem{MyExample}{Example}
\newtheorem{MyFramework}{Framework}
\newtheorem{MyCondition}{Condition}

\allowdisplaybreaks

\jmlrheading{1}{2000}{1-48}{4/00}{10/00}{XXXXXX}{Ai Azuma, Masashi Shimbo, and Yuji Matsumoto}

\ShortHeadings{Forward and Forward-Backward Algorithms}{Azuma, Shimbo, and Matsumoto}
\firstpageno{1}

\begin{document}

\title{An Algebraic Formalization of Forward and Forward-backward Algorithms}

\author{\name Ai Azuma \email ai-a@is.naist.jp \\
        \name Masashi Shimbo \email shimbo@is.naist.jp \\
        \name Yuji Matsumoto \email matsu@is.naist.jp \\
        \addr Graduate School of Information Science\\
        Nara Institute of Science and Technology\\
        8916-5 Takayama, Ikoma, Nara 630-0192, Japan}

\editor{EDITOR NAME}

\maketitle

\begin{abstract}
In this paper, we propose an algebraic formalization of the two important classes of dynamic programming algorithms called forward and forward-backward algorithms. They are generalized extensively in this study so that a wide range of other existing algorithms is subsumed. Forward algorithms generalized in this study subsume the ordinary forward algorithm on trellises for sequence labeling, the inside algorithm on derivation forests for CYK parsing, a unidirectional message passing on acyclic factor graphs, the forward mode of automatic differentiation on computation graphs with addition and multiplication, and so on. In addition, we reveal algebraic structures underlying complicated computation with forward algorithms. By the aid of the revealed algebraic structures, we also propose a systematic framework to design complicated variants of forward algorithms. Forward-backward algorithms generalized in this study subsume the ordinary forward-backward algorithm on trellises for sequence labeling, the inside-outside algorithm on derivation forests for CYK parsing, the sum-product algorithm on acyclic factor graphs, the reverse mode of automatic differentiation (a.k.a. back propagation) on computation graphs with addition and multiplication, and so on. We also propose an algebraic characterization of what can be computed by forward-backward algorithms and elucidate the relationship between forward and forward-backward algorithms.
\end{abstract}

\begin{keywords}
forward-backward algorithm, inside-outside algorithm, sum-product algorithm, back propagation, semiring
\end{keywords}

\section{Introduction}

In this paper, we propose an algebraic formalization of the two important classes of dynamic programming algorithms on computation over commutative semirings. One of the classes is called forward algorithms, in which the order of computation is consistent with the dependencies among intermediate values. The other is called forward-backward algorithms, in which forward and backward passes are combined.

Algebraic generalizations of ``forward algorithms'' are formalized for many kinds of data structures, but they are developed independently. Here, the term ``forward algorithms'' includes not only the ordinary forward algorithm~\citep{rabiner1989tutorial} on trellises~\citep{forney1973viterbi} for sequence labeling like hidden Markov models (HMMs) or linear-chain conditional random fields~\citep[CRFs,][]{lafferty2001conditional}, but also the inside algorithm on derivation forests for CYK parsing, a properly scheduled unidirectional message passing on acyclic factor graphs, etc.%
\footnote{These confusing uses of the terms ``forward algorithms'' and ``forward-backward algorithms''  in this paper may bring discomfort. However, the formalization in this paper justifies them.\label{footnote:forward_and_forward_backward}}
Examples of data structures on which computation is algebraically generalized include trellises for sequence labeling, the set of derivations by weighted deduction system or logic programming~\citep{goodman1999semiring,lopez2009translation,eisner2011dyna,kimmig2011algebraic}, junction trees~\citep{aji2000generalized}, factor graphs~\citep{kschischang2001factor}, directed graphs~\citep{mohri2002semiring}, directed hypergraphs%
\footnote{To be more precise, a subset of directed hypergraphs called \termref{B-graphs}~\citep{gallo1993directed}.}%
~\citep{klein2004parsing,huang2008advanced}, binary decision diagrams~\citep[BDDs,][]{wilson2005decision}, and sum-product networks~\citep{friesen2016sum}.

Previous studies have not considered any algebraic formalization of ``bidirectional counterparts'' of forward algorithms. Hereinafter, they are called ``forward-backward algorithms.''\cref{footnote:forward_and_forward_backward}
Forward-backward algorithms include the ordinary forward-backward algorithm \citep{rabiner1989tutorial} for the ordinary forward algorithm on trellises for sequence labeling, the inside-outside algorithm~\citep{lari1990estimation} for the inside algorithm on derivation forests for CYK parsing, the sum-product algorithm~\citep{kschischang2001factor} for a unidirectional message-passing on acyclic factor graphs, and so on. In the current status of such one-sided research stream, we overlook an integrated and organic linkage between forward and forward-backward algorithms.

In some machine learning tasks, complicated variants of forward or forward-backward algorithms are necessary to be designed. However, there is no systematic framework to design such algorithms. Examples of such complicated computation include ``forward-only computation'' of the Baum-Welch algorithm~\citep{tan1993adaptive,sivaprakasam1995forward,turin1998unidirectional,miklos2005linear,churbanov2008implementing}, the entropy gradient of CRFs~\citep{mann2007efficient}, the gradient of entropy or risk of acyclic hypergraphs~\citep{li2009first}, Hessian-vector products of CRFs~\citep{tsuboi2011fast}, cross moments of factor graphs~\citep{ilic2012computation}. It is a great loss to the research community that  individual studies independently manage to develop these algorithms.

\begin{figure}[t]
\centering
\includegraphics[width=\columnwidth,draft=false]{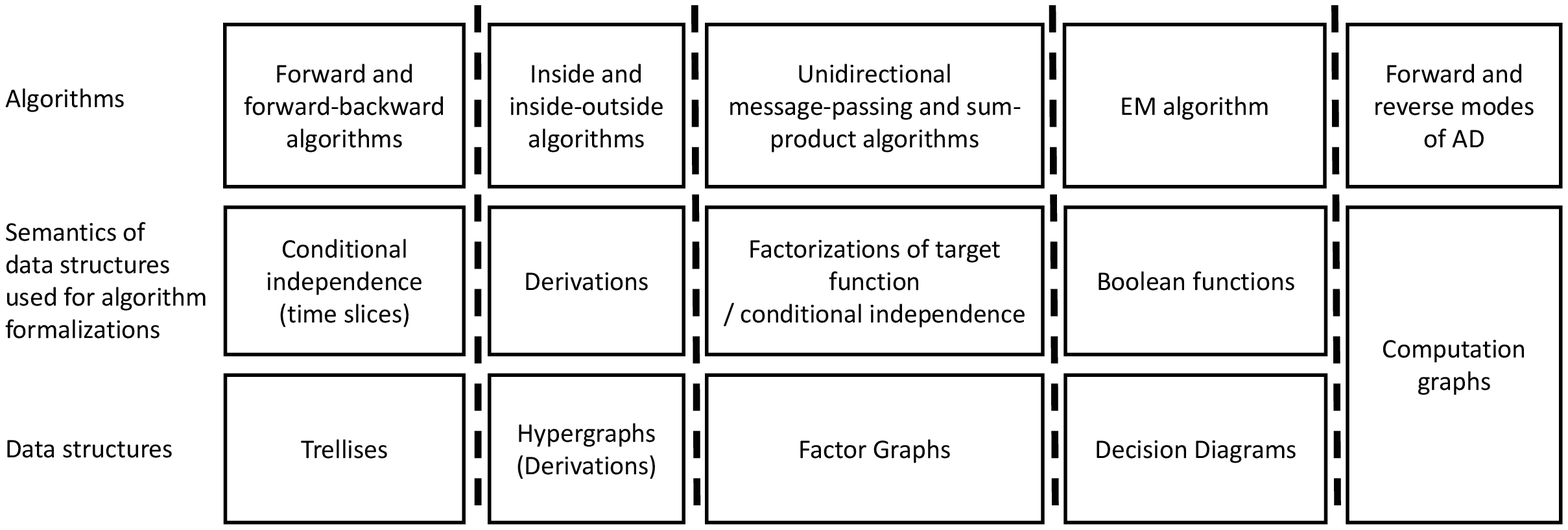}
\caption[Before]{Our grasp of the current status of the formalizations of target algorithms}
\label{fig:introduction_before}
\end{figure}%
\begin{figure}[t]
\centering
\includegraphics[width=0.8\columnwidth,draft=false]{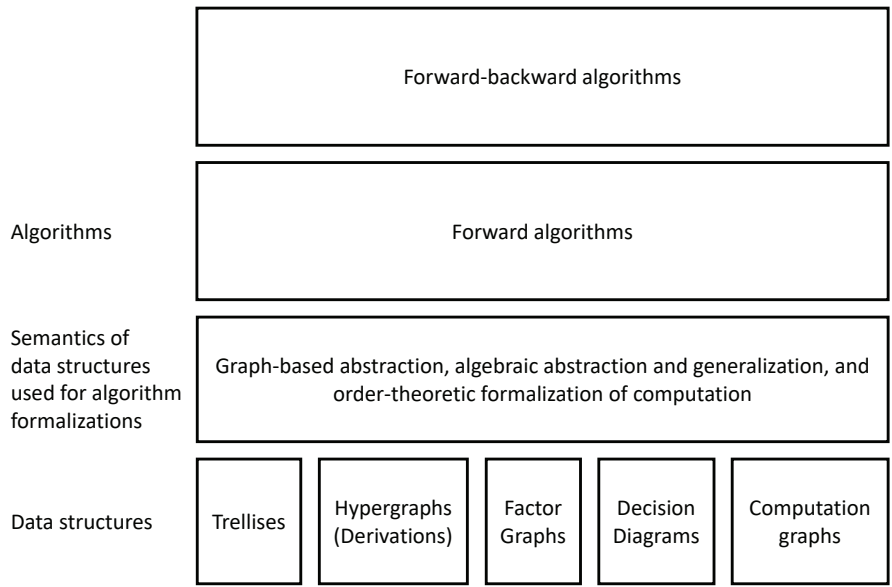}
\caption[After]{A rough sketch of the formalization presented in this paper}
\label{fig:introduction_after}
\end{figure}%
Here, we outline our approach in this paper by comparing our formalization with the usual formalizations of target algorithms. In Fig. \ref{fig:introduction_before}, we illustrate our grasp of the current status of the formalizations of target algorithms. Each of them is formalized independently. The reason is mainly because each formalization is built on top of the individual semantics of the target data structure. In contrast, we present only one formalization. Figure \ref{fig:introduction_after} illustrates a rough sketch of the formalization presented in this paper. Before we formalize forward and forward-backward algorithms, we introduce a unified abstraction of computation on a variety of data structures. The abstraction completely separates the formalization of the algorithms from details of data structures. Accordingly, we need only one formalization of forward and forward-backward algorithms while maximizing the range of their application. In addition, there is one more important point in Fig. \ref{fig:introduction_after}. Forward-backward algorithms are built on top of forward algorithms. The formalization in this way provides new insight into forward-backward algorithms and the relationship between forward and forward-backward algorithms.

Our contribution in this paper is roughly four-folded. First, we propose a computation model that can represent arbitrary computation consisting of a finite number of applications of additions and/or multiplications of a commutative semiring. The formalization presented in this paper is applicable to computation on various kinds of data structures including trellises for sequence labeling, derivation forests or acyclic hypergraphs for CYK parsing, acyclic factor graphs, a variety of decision diagrams, and so on. This wide applicability is due to the versatility of the proposed computation model on which the formalization is built. Second, algebraic structures underlying complicated computation with forward algorithms are revealed. Third, we propose a systematic framework to design complicated variants of forward algorithms by the aid of the revealed underlying algebraic structures. The framework allows us to compose a complicated and difficult-to-design forward algorithm from primitive and easy-to-design forward algorithms. Fourth, an algebraic formalization of forward-backward algorithms is proposed. It naturally reveals a relationship between forward and forward-backward algorithms. In particular, it turns out that what can be computed by forward-backward algorithms is a specific case of what can be computed by forward algorithms. This fact immediately implies that what can be computed by forward-backward algorithms can be always computed by forward algorithms. In addition, what can be computed by some instances of forward algorithms can be also done by forward-backward algorithms. The transformations between forward and forward-backward algorithms can be done in a completely systematic way, which even include a systematic transformation from the forward mode of automatic differentiation (AD) to the reverse mode (a.k.a. back propagation). We also identify a time-space trade-off between corresponding forward and forward-backward algorithms.

This paper is organized as follows: Section \ref{sec:forward_algorithms} describes an algebraic formalization of forward algorithms, and Section \ref{sec:forward_backward_algorithms} describes an algebraic formalization of forward-backward algorithms. In more detail, Section \ref{subsec:commutative_semiring_and_computation_graph} introduces a computation model on which we build formalizations throughout this paper. Sections \ref{subsec:commutative_semiring_computation_graph_parametrized_by_monoid_homomorphism} and \ref{subsec:tensor_product_of_semialgebras_for_forward_algorithm} reveal algebraic structures underlying complicated forward algorithms. Section \ref{subsec:tensor_product_of_semialgebras_for_forward_algorithm} also offers a systematic framework to compose complicated and difficult-to-design forward algorithms from primitive and easy-to-design forward algorithms. Section \ref{subsec:backward_invariants_and_forward_backward_algorithms} formalizes forward-backward algorithms in an algebraic way. Section \ref{subsec:conditions_favoring_forward_backward_algorithms} compares forward and forward-backward algorithms from the point of view of time-space trade-offs. Section \ref{subsec:checkpoints_for_forward_backward_algorithms} provides a brief note on so-called \termref{checkpoints}, which trade off time and space complexity in forward-backward algorithms. Section \ref{sec:conclusion} concludes this paper.

\section{Forward Algorithms}
\label{sec:forward_algorithms}

\subsection{Commutative Semiring and Computation Graph}
\label{subsec:commutative_semiring_and_computation_graph}

Unless otherwise stated, when the extensional definition of a set, say, $X = \left\{x_1, \dots, x_n\right\}$, is given, we henceforth assume that the listed elements are pairwise distinct ($x_i \neq x_j$ for $i \neq j$ in other words) and thus $\left|X\right| = n$.

Let $\bbNzero$ denote the set of all non-negative integers and $\bbN$ the set of all positive integers.

In this subsection, we first introduce some notations related to directed acyclic graphs (dags) and then definitions related to the theory of semiring. They are combined to formalize a model of computation over a commutative semiring.

In this paper, parallel arcs are allowed in dags. For this reason, the arc set $E$ of a dag $G = (V, E)$ is a set equipped with the head function $\head: E \to V$ and tail function $\tail: E \to V$ that map an arc to its head and tail, respectively. For a dag $G = (V, E)$, we denote:
\begin{itemize}
\item the set of all \termref{in-arcs} of a node $v \in V$ by $E^{-}_{G}(v)$, i.e., $E^{-}_{G}(v) = \left\{e \in E \relmiddle| \head(e) = v\right\}$,
\item the set of all \termref{out-arcs} of a node $v \in V$ by $E^{+}_{G}(v)$, i.e., $E^{+}_{G}(v) = \left\{e \in E \relmiddle| \tail(e) = v\right\}$,
\item the set of all \termref{source nodes} (i.e., nodes without any in-arc) by $\src(G)$, i.e., $\src(G) = \left\{v \in V \relmiddle| E^{-}_{G}(v) = \emptyset\right\}$, and
\item the set of all \termref{sink nodes} (i.e., nodes without any out-arc) by $\snk(G)$, i.e., $\snk(G) = \left\{v \in V \relmiddle| E^{+}_{G}(v) = \emptyset\right\}$.
\end{itemize}

\begin{MyDefinition}[Commutative Monoid]
Let $M \neq \emptyset$ be a set, $\cdot$ a binary operation on $M$, and $1_M$ an element of $M$. Then $\left(M, \cdot, 1_M\right)$ is called a \termref{commutative monoid} if and only if it satisfies, for every $a, b, c \in M$:
\begin{itemize}
\item $1_M$ is the \termref{identity element}, i.e., $1_M \cdot a = a \cdot 1_M = a$,
\item the operation obeys the \termref{commutative law}, i.e., $a \cdot b = b \cdot a$, and
\item the operation obeys the \termref{associative law}, i.e., $a \cdot (b \cdot c) = (a \cdot b) \cdot c$.
\end{itemize}
\label{def:commutative_monoid}
\end{MyDefinition}

When there will be no confusion, an algebraic structure is often denoted simply by its underlying set. This leads to objectionable notations, e.g., $M = \left(M, \cdot, 1_M\right)$.

Definition \ref{def:commutative_monoid} is ``multiplicatively-written,'' that is to say, the operation is denoted by a multiplication-suggestive symbol, and the identity element is denoted by ``$1$.'' However, a commutative monoid may be ``additively-written'' as the situation demands, i.e., the operation is denoted by an addition-suggestive symbol such as ``$+$,'' and the identity element is denoted by ``$0$.''

For an additively-written commutative monoid $M = \left(M, +, 0_M\right)$, the \termref{summation symbol $\sum$} is used. To be precise, let $X$ be a set, and let each element $x \in X$ be associated with an element $a_x \in M$. Even if $X$ is an infinite set, we assume $X' = \left\{a_x \relmiddle| x \in X \text{ and } a_x \neq 0_M\right\}$ is a finite set. Then we define
\begin{equation}
\sum_{x \in X} a_x = \begin{cases}
  0_M                 & \text{if $X' = \emptyset$,} \\
  \sum_{x \in X'} a_x & \text{otherwise.}
\end{cases}
\label{def:summation}
\end{equation}
In particular, if $\left|X'\right| = n$ and there exists $a \in M$ such that $a_x = a$ for every $a_x \neq 0_M$, then \eqref{def:summation} defines the \termref{$n$ repetitions} of $a$, denoted by $na$. For a multiplicatively-written commutative monoid $\left(M', \cdot, 1_{M'}\right)$, the \termref{product symbol $\prod$} is also defined in a similar fashion. Note, however, $\prod_{x \in X} a_x = 1_{M'}$ if $X' = \emptyset$. The \termref{$n$-th power} of $a \in M'$, denoted by $a^n$, is defined likewise.

\begin{MyDefinition}[Monoid Homomorphism]
Let $M = \left(M, \cdot, 1_{M}\right)$ and $M' = \left(M', \odot, 1_{M'}\right)$ be commutative monoids. Then a mapping $f \colon M \to M'$ is called a \termref{(monoid) homomorphism} from $M$ to $M'$ if and only if it satisfies, for every $a, b \in M$:
\begin{itemize}
\item $f\!\left(a \cdot b\right) = f\!\left(a\right) \odot f\!\left(b\right)$, and
\item $f\!\left(1_{M}\right) = 1_{M'}$.
\end{itemize}
\end{MyDefinition}

\begin{MyDefinition}[Commutative Semiring]
Let $S \neq \emptyset$ be a set, $+$ and $\cdot$ binary operations on $S$, and $0_S$ and $1_S$ elements of $S$. Then $\left(S, +, \cdot, 0_{S}, 1_{S}\right)$ is called a \termref{commutative semiring} if and only if it satisfies:
\begin{itemize}
\item $\left(S, +, 0_{S}\right)$ is a commutative monoid,
\item $\left(S, \cdot, 1_{S}\right)$ is a commutative monoid,
\item the two operations are connected by the \termref{distributive law}, i.e., $a \cdot \left(b + c\right) = \left(a \cdot b\right) + \left(a \cdot c\right)$ and $\left(a + b\right) \cdot c = \left(a \cdot c\right) + \left(b \cdot c\right)$ for every $a, b, c \in S$, and
\item $0_S$ is \termref{absorbing}, i.e., $0_{S} \cdot a = a \cdot 0_{S} = 0_{S}$ for every $a \in S$.
\end{itemize}
\end{MyDefinition}

For a commutative semiring $S = \left(S, +, \cdot, 0_S, 1_S\right)$, $\left(S, +, 0_S\right)$ and $\left(S, \cdot, 1_S\right)$ are called the \termref{additive monoid} and \termref{multiplicative monoid} of $S$, respectively. For every $a \in S$, the $n$ repetitions of $a$, denoted by $na$, (resp. the $n$-th power of $a$, denoted by $a^n$) is defined in terms of one defined on the additive (resp. multiplicative) monoid of $S$.

Now we are ready to introduce a model of computation over a commutative semiring. The computation model to be introduced is so-called computation graphs~\citep[a.k.a. Kantorovich graphs,][]{rall1981automatic} but where the kinds of operations are limited to two binary ones, i.e., addition and multiplication of the commutative semiring. Each value involved in the computation is attached to a source node of a dag, and each internal node of the dag designates either addition or multiplication operation of the commutative semiring. Intermediate values during the computation are associated with nodes and arcs by the \termref{forward variable}, and dependencies among them are represented by arcs.

\begin{MyDefinition}[Commutative Semiring Computation Graph\footnote{The underlying graph of a commutative semiring computation graph is allowed to be disconnected. Disconnected commutative semiring computation graphs are necessary to model computation on ``disconnected'' data structures (e.g., disconnected factor graphs).}]
Let $G = \left(V, E\right)$ be a finite dag, $\op$ a mapping from $V \setminus \src\!\left(G\right)$ to $\left\{\text{``$+$''}, \text{``$\cdot$''}\right\}$, $S$ a commutative semiring, and $\xi$ a mapping from $\src\!\left(G\right)$ to $S$. Then the quadruple $\left(G, \op, S, \xi\right)$ is called a \termref{commutative semiring computation graph}.
\end{MyDefinition}

Unless otherwise noted, we henceforth use the term ``computation graph'' to denote a commutative semiring computation graph for the sake of brevity.

\begin{MyDefinition}[Forward Variable\footnote{For purely technical reasons, forward variables (and backward variables in Definition \ref{def:beta}) are defined not only on nodes but also on arcs. In Section \ref{sec:forward_backward_algorithms}, we formalize forward-backward algorithms with the partially ordered set (poset) naturally induced by computation graphs. If the poset were induced only by nodes, the poset would belong to the class of arbitrary posets for which some problems are intractable. In contrast, defining forward variables on both nodes and arcs makes the induced posets fall into a tractable subset called \termref{chain-antichain-complete} posets, \termref{edge-induced} posets, \termref{N-free} posets, or \termref{quasi-series-parallel} posets, which are equivalent to each other~\citep{mohring1989computationally}. In fact, the proofs of some statements in Section \ref{sec:forward_backward_algorithms} rely on the tractability of this subclass.\label{footnote:tractable_posets}}]
Let $G = \left(V, E\right)$ be a finite dag, $S = \left(S, +, \cdot, 0_S, 1_S\right)$ a commutative semiring, and $\mathcal{G} = \left(G, \op, S, \xi\right)$ a computation graph. Then the \termref{forward variable} of $\mathcal{G}$, denoted by $\alpha_{\mathcal{G}}$, is a mapping from $V \cup E$ to $S$ that is defined by, for every node $v \in V$ and every arc $e \in E$,
\begin{equation}
\begin{aligned}
\alpha_{\mathcal{G}}\!\left(v\right) &= \begin{cases}
  \xi\!\left(v\right)
    & \quad \text{if $v \in \src(G)$,}         \\
  \sum_{e \in E^{-}_{G}(v)} \alpha_{\mathcal{G}}(e)
    & \quad \text{if $\op(v) = \text{``$+$''}$,}       \\
  \prod_{e \in E^{-}_{G}(v)} \alpha_{\mathcal{G}}(e)
    & \quad \text{otherwise (i.e., if $\op(v) = \text{``$\cdot$''}$),}
\end{cases} \\
\alpha_{\mathcal{G}}\!\left(e\right) &= \alpha_{\mathcal{G}}(\tail(e)) \; .
\end{aligned}
\label{eq:alphadef}
\end{equation}
\label{def:alpha}
\end{MyDefinition}

It is obvious that the above mutually recursive definition of $\alpha_{\mathcal{G}}$ on $V \cup E$ is well-defined since $G$ is a finite dag.

\termref{Forward algorithms} are algorithms to compute values of the forward variable of computation graphs. Leaving aside ``scheduling problems,'' we can readily compute values of the forward variable of a computation graph since forward variables are constructively defined in Definition \ref{def:alpha}. We postpone the presentation of the pseudo-code of forward algorithms with a full discussion of scheduling problems and others until Section \ref{sec:forward_backward_algorithms} because some additional notions are necessary to be introduced.

Let $\mathcal{G} = \left(G, \op, S, \xi\right)$ be a computation graph, and $\src\!\left(G\right) = \left\{s_1, \dots, s_n\right\}$. By induction on the recursive definition \eqref{eq:alphadef} in Definition \ref{def:alpha}, it is easy to show that, for every node and arc $t \in V \cup E$, the forward variable $\alpha_{\mathcal{G}}\!\left(t\right)$ is of the form
\begin{equation}
\alpha_{\mathcal{G}}\!\left(t\right) =
\sum_{\veci \in {\mathbb{N}_0}^n} c_{t, \veci}
\left(\xi\!\left(s_1\right)\right)^{i_1} \cdots \left(\xi\!\left(s_n\right)\right)^{i_n} \; ,
\label{eq:plain_alpha}
\end{equation}
where $\veci = (i_1, \dots, i_n)$, and $c_{t, \veci} \in \mathbb{N}_0$ for every $t \in V \cup E$ and $\veci \in {\mathbb{N}_0}^n$ but only finitely many of the coefficients $c_{t, \veci}$ in the summand are different from $0$, and the values of $c_{t, \veci}$ are dependent only on $G$, $\op$, and $t$, i.e., they are independent of $S$ and $\xi$. Note that while the right-hand side in \eqref{eq:plain_alpha} is an infinite sum over ${\mathbb{N}_0}^n$ at first glance, it is actually a finite sum and well-defined because of the condition imposed on $c_{t, \veci}$.

The fact that $\alpha_{\mathcal{G}}\!\left(t\right)$ takes on the form of \eqref{eq:plain_alpha} can be rephrased as follows. The pair $\left(G, \op\right)$ solely determines a polynomial in the indeterminates $x_1, \dots, x_n$ over $\mathbb{N}_0$, which is of the form
\begin{equation}
\sum_{\veci \in {\mathbb{N}_0}^n} c_{t, \veci} {x_1}^{i_1} \cdots {x_n}^{i_n}
\label{eq:form_of_free_forward_variable}
\end{equation}
for every node and arc $t \in V \cup E$. On the other hand, the pair $\left(S, \xi\right)$ specifies ``substitution'' of the polynomial, or ``replacing $x_i$ by $\xi(s_i)$ for every $i$'' in other words, and $\alpha_{\mathcal{G}}\!\left(t\right)$ is equal to the result of the substitution on every $t \in V \cup E$.%
\footnote{%
  See \citet[Chapter II]{hebisch1998algebraic} for the formal definitions of \termref{indeterminates}, \termref{polynomials}, \termref{substitution}, etc., in particular, Definition II.1.1, Theorems II.1.3, II.1.6, II.1.8, and Remark II.1.10.}

The modularity offered by algebraic abstraction is based on the division of the roles between the two pairs $\left(G, \op\right)$ and $\left(S, \xi\right)$ in a given computation graph $\left(G, \op, S, \xi\right)$. On the one hand, the pair $\left(G, \op\right)$ represents intermediate procedure that applies addition- and multiplication-like operations to given values. In this procedure, the details of the underlying set and the two operations are completely abstracted away; they can be anything that obeys the axioms of commutative semiring. On the other hand, the pair $\left(S, \xi\right)$ specifies these details and gives the abstract computation specified by $\left(G, \op\right)$ concrete meaning.

In order to formalize this modularity, we define the \termref{free forward variable}%
\footnote{%
  The term ``free forward variable'' is named after the fact that $\mathbb{N}_0\!\left[x_1, \dots, x_n\right]$ is the \termref{free} commutative semiring on $\left\{x_1, \dots, x_n\right\}$.}
on a computation graph by using the \termref{polynomial semiring over $\bbNzero$}. Let $\mathbb{N}_0\!\left[x_1, \dots, x_n\right]$ denote the set of all polynomials in the indeterminates $x_1, \dots, x_n$ over $\mathbb{N}_0$, that is,
\begin{equation*}
\begin{aligned}
&\mathbb{N}_0\!\left[x_1, \dots, x_n\right] = \\
&\qquad\left\{
{\textstyle\sum_{\veci \in {\mathbb{N}_0}^n} c_{\veci} {x_1}^{i_1} \cdots {x_n}^{i_n}}
\relmiddle|
\text{$\veci = (i_1, \dots, i_n)$, and
  $c_{\veci} \in {\mathbb{N}_0}$ but almost all $c_{\veci}$ are $0$.}
\right\} \; .
\end{aligned}
\end{equation*}
Then we can equip $\mathbb{N}_0\!\left[x_1, \dots, x_n\right]$ with the equality relation, addition, and multiplication in the usual way to make it a commutative semiring. The resulting semiring is called the \termref{polynomial semiring in the indeterminates $x_1, \dots, x_n$ over $\bbNzero$}.\footnote{%
  See \citet[Theorem II.1.3, Definition II.1.4, and Remark II.1.10]{hebisch1998algebraic} for the formal definitions of the equality relation, addition, and multiplication.}

\begin{MyDefinition}[Free Forward Variable]
Let $G = \left(V, E\right)$ be a finite dag, $\op$ a mapping from $V \setminus \src\!\left(G\right)$ to $\left\{\text{``$+$''}, \text{``$\cdot$''}\right\}$, $\src\!\left(G\right) = \left\{s_1, \dots, s_n\right\}$, and $\chi \colon \src\!\left(G\right) \to \bbNzero\!\left[x_1, \dots, x_n\right]$ a function that maps each $s_i \in \src\!\left(G\right)$ to $x_i$. Then, the \termref{free forward variable of $\left(G, \op\right)$ with respect to $\chi$}, denoted by $\alpha_{\left(G, \op, \chi\right)}$, is a mapping $\alpha_{\left(G, \op, \chi\right)} \colon V \cup E \to \bbNzero\!\left[x_1, \dots, x_n\right]$ defined by
\begin{equation*}
\alpha_{\left(G, \op, \chi\right)}\!\left(t\right) =
\alpha_{\left(G, \op, \bbNzero\left[x_1, \dots, x_n\right], \chi\right)}\!\left(t\right)
\end{equation*}
for every node and arc $t \in V \cup E$.
\end{MyDefinition}

By using the free forward variable of a given computation graph $\left(G, \op, S, \xi\right)$, the division of the roles between the two pairs $\left(G, \op\right)$ and $\left(S, \xi\right)$ can be summarized in the following lemma.

\begin{MyLemma}[Substitution Principle of Free Forward Variable]
Let $\mathcal{G} = \left(G, \op, S, \xi\right)$ be a computation graph, and $\src\!\left(G\right) = \left\{s_1, \dots, s_n\right\}$. Assume that one defines $\chi \colon \src\!\left(G\right) \to \bbNzero\!\left[x_1, \dots, x_n\right]$ mapping $s_i$ to $x_i$ for every $s_i \in \src\!\left(G\right)$, and obtains the following form of the free forward variable of $\left(G, \op\right)$ w.r.t. $\chi$
\begin{equation*}
\alpha_{\left(G, \op, \chi\right)}\!\left(t\right)
= \sum_{\veci \in \bbNzero^n} c_{t, \veci} {x_1}^{i_1} \cdots {x_n}^{i_n}
\end{equation*}
for every node and arc $t \in V \cup E$, then
\begin{equation*}
\alpha_{\mathcal{G}}\!\left(t\right) = \sum_{\veci \in \bbNzero^n} c_{t, \veci}
\left(\xi\!\left(s_1\right)\right)^{i_1} \cdots \left(\xi\!\left(s_n\right)\right)^{i_n} \; .
\end{equation*}
\label{lemma:substitution_principle_of_free_forward_variable}
\end{MyLemma}
\begin{proof}
The statement can be established by induction on the finite dag $G$.
\end{proof}

\begin{figure}[t]
\centering
\includegraphics[width=0.7\columnwidth,draft=false]{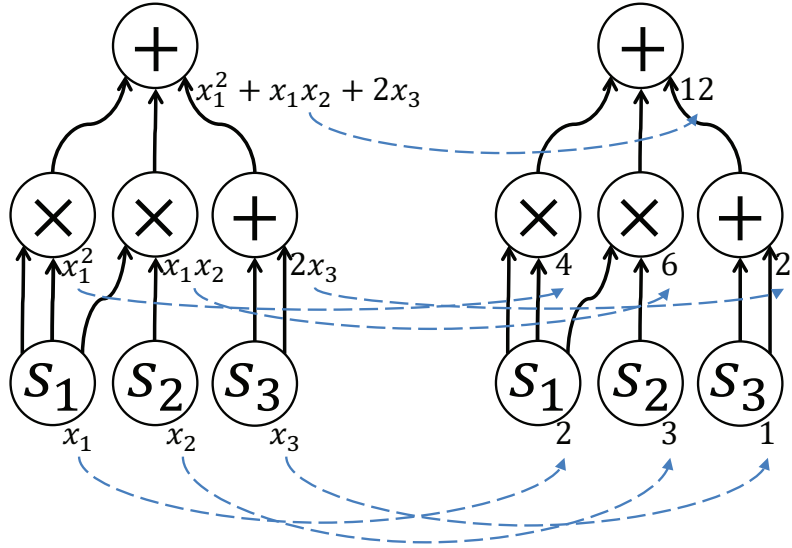}
\caption[An Example of Substitution Principle of Free Forward Variable]{The right dag is an example computation graph with $S = \mathbb{R}, \xi\!\left(s_1\right) = 2, \xi\!\left(s_2\right) = 3, \xi\!\left(s_3\right) = 1$, and the left one is the corresponding graph with the free forward variable ($\chi\!\left(s_i\right) = x_i$). The values of the forward variables on the arcs are omitted for the sake of brevity. Dashed arrows between the values of the forward variables of both the graphs represent the ``substitution of $x_i$ with $\xi\!\left(s_i\right)$.''}
\label{fig:substitution_principle_of_free_forward_variable}
\end{figure}

Figure \ref{fig:substitution_principle_of_free_forward_variable} shows the graphical illustration of the substitution principle of the free forward variable for an example computation graph.

Introducing free forward variables allows us to analyze how the computation changes depending on various structures equipped in $S$ and $\xi$ without getting into details of the computation structure represented by $\left(G, \op\right)$. In what follows, when a computation graph $\left(G, \op, S, \xi\right)$ is considered, details of $\left(G, \op\right)$ are specified only by the free forward variable, and a pure focus is placed on how each structure equipped in $S$ and $\xi$ affects the computation. The free forward variable provides information about the computation structure specified by $\left(G, \op\right)$ to the extent that is necessary and sufficient for the development of this and the subsequent sections.

Moreover, free forward variable combined with ``substitution'' (i.e., ``replacing $x_i$ by $\xi\!\left(s_i\right)$ for every $i$'') can model any computation insofar as the computation consists of a finite number of applications of additions and/or multiplications that obey the axioms of a commutative semiring. Therefore, this computation model subsumes such diverse computations as the ordinary forward algorithm on trellises for sequence labeling, the inside algorithm on derivation forests or hypergraphs for CYK parsing, a unidirectional message-passing on acyclic factor graphs.

In what follows, for various data structures, we illustrate the corresponding computation graphs.

\begin{MyExample}[Computation Graph for Sequence Labeling]
\label{ex:cg_for_sequence_labeling}
The upper diagram in Fig. \ref{fig:trellis_and_computation_graph} shows an example trellis for sequence labeling with $3$ states numbered $0$, $1$, and $2$, and the lower directed graph $\left(G, \op\right)$ in Fig. \ref{fig:trellis_and_computation_graph} shows the computation graph (without specifying the domain $S$ and the values on source nodes $\xi\!\left(s_i\right)$) corresponding to the ordinary forward algorithm on the upper trellis. In the computation graph, the source node $s_0$ conceptually represents the residence in the state $0$ at time $t = 0$. The same applies to $s_1$ and $s_2$. The source node $s_3$ conceptually represents the transition from the state $0$ at time $t = 0$ to the same state at time $t = 1$. The same applies to $s_i$ for every $i \in \left\{4, \dots, 11\right\}$. The source node $s_{12}$ conceptually represents the residence in the state $0$ at time $t = 1$. The same applies to $s_{13}$ and $s_{14}$. The same goes for all following source nodes.

\begin{figure}[t]
\centering
\includegraphics[width=0.8\columnwidth,draft=false]{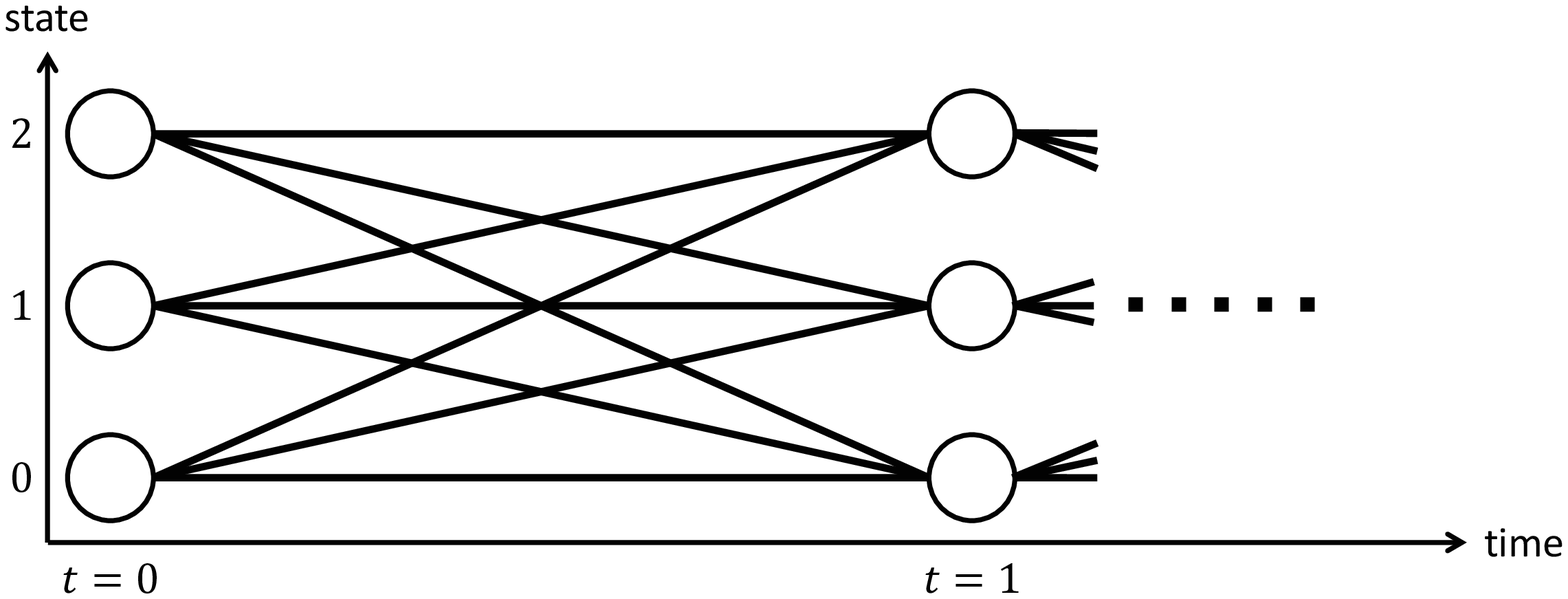}
\includegraphics[width=0.8\columnwidth,draft=false]{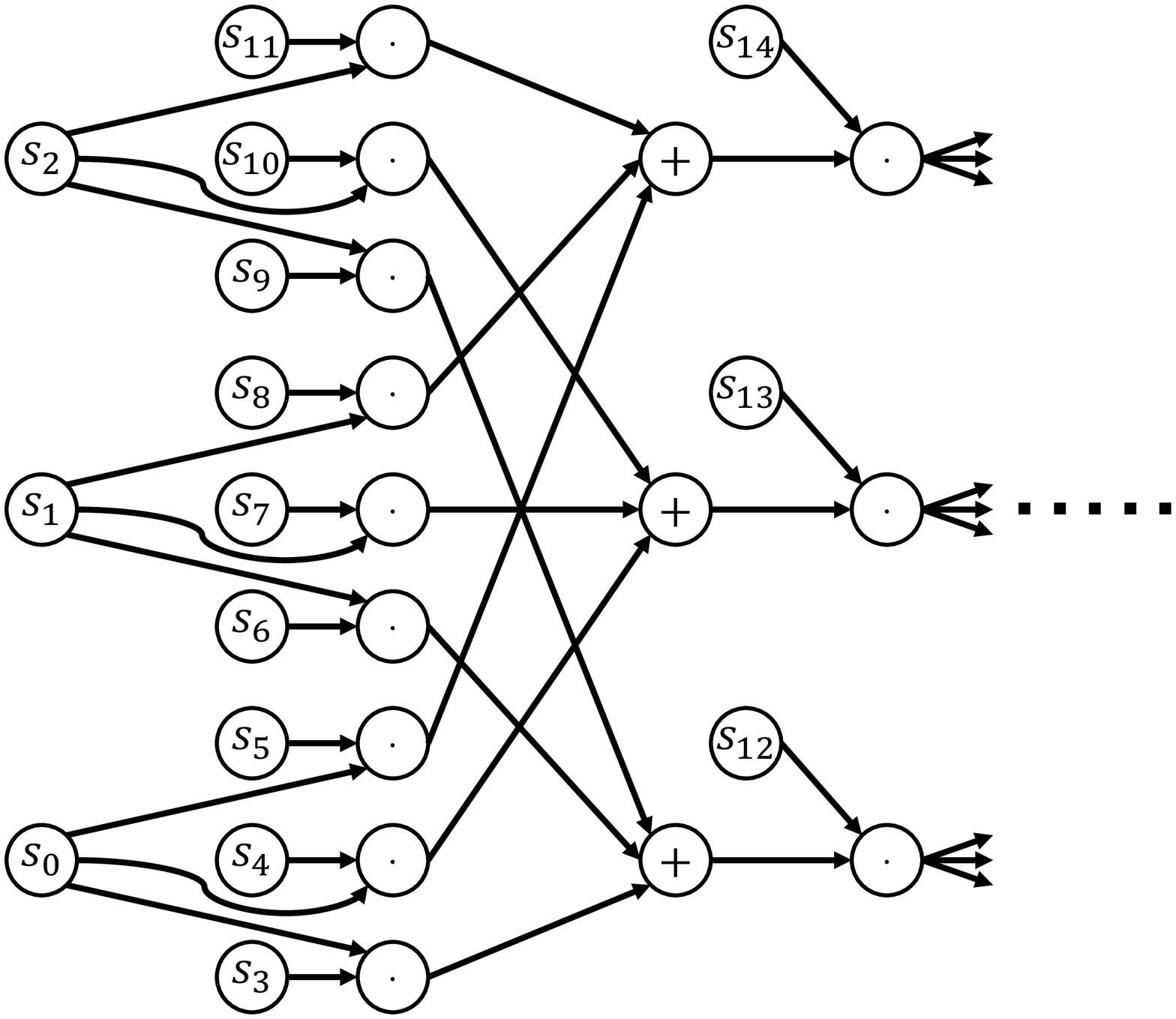}
\caption[An Example Trellis and Corresponding Computation Graph]{The upper diagram is an example trellis for sequence labeling (e.g., HMMs or CRFs) with $3$ states numbered $0$, $1$, and $2$. The lower directed graph is the computation graph corresponding to the ordinary forward algorithm on the upper trellis.}
\label{fig:trellis_and_computation_graph}
\end{figure}
Consider the free forward variable of the lower computation graph in Fig. \ref{fig:trellis_and_computation_graph} with respect to $\chi \colon s_i \mapsto x_i$. In the sum of the values of the free forward variable over all sink nodes, which obviously takes on the form \eqref{eq:form_of_free_forward_variable}, each term corresponds to a sequence (in other words, a joint assignment of states throughout the full period of time) in the upper trellis in Fig. \ref{fig:trellis_and_computation_graph}. Moreover, each factor in a term corresponds to the residence in a state at a moment in time or a transition in the sequence that the term corresponds to.

For example, each of all terms involving the factor $x_0$ conceptually represents a sequence passing through the state $0$ at time $t = 0$ in the upper trellis in Fig. \ref{fig:trellis_and_computation_graph}. For another example, each of all terms involving the factor $x_3$ conceptually represents a sequence containing the arc (transition) between the state $0$ at time $t = 0$ and the same state at time $t = 1$ in the upper trellis in Fig \ref{fig:trellis_and_computation_graph}.

Further consider an example of the concrete meaning of a ``substitution.'' Let\linebreak $\xi\colon \src(G) \to \bbR$ be defined by
\begin{equation*}
\xi\!\left(s_i\right) = \begin{cases}
  P\!\left(T = 0, S = i\right) P\!\left(O = o_0 \relmiddle| S = i\right)
  \hspace{97pt}\text{if $i \in \left\{0, 1, 2\right\}$,} \\
  \begin{aligned}
    &\!P\!\left(T = t + 1, S = k \relmiddle| T = t, S = j\right) \\
    &\hspace{126pt}\text{if $12t + 3 \leq i \leq 12t + 11$ for some $t \in \bbNzero$,} \\
    &\hspace{126pt}\text{\phantom{if }$j = \left(\left(i - 3\right) \modulo 12 - \left(i \modulo 3\right)\right) \!/ 3$, and} \\
    &\hspace{126pt}\text{\phantom{if }$k = \left(i - 3\right) \modulo 12$,}
  \end{aligned}\\
  P\!\left(O = o_t \relmiddle| S = \left(i \modulo 12\right)\right)
  \hspace{35pt}\text{if $12t \leq i \leq 12t + 2$ for some $t \in \bbN$},
\end{cases}
\end{equation*}
where $T$ is a discrete variable ranging over the full period of time, $S$ is a variable ranging over all the state numbers, $O$ is a variable ranging over a given domain of discrete observations, $o_i$ is the observation at time $t = i$, and $\modulo$ is the modulo operator. Then, it is obvious what the substitution of the free forward variable by $\xi$ means. That is, the values of the forward variable of the computation graph $\left(G, \op, \bbR, \xi\right)$ completely correspond to the ordinary forward variables of the HMMs on the upper trellis in Fig. \ref{fig:trellis_and_computation_graph} for a given observation sequence $\left(o_0, o_1, \dots\right)$.
\end{MyExample}

\begin{MyExample}[Computation Graph for Acyclic Factor Graph]
The upper diagram in Fig. \ref{fig:factor_graph_and_computation_graph} illustrates an example acyclic factor graph borrowed from \citet{kschischang2001factor}. The directed arrows in the upper factor graph indicate an example \termref{generalized forward/backward (GFB) schedule}. The lower directed graph is the computation graph corresponding to the unidirectional message passing on the upper factor graph with respect to the GFB schedule, provided that all variables of the factor graph take on the binary values $0$ or $1$.

Table \ref{table:factor_graph_cscg} shows what each source node of the computation graph in Fig. \ref{fig:factor_graph_and_computation_graph} conceptually represents. Each source node represents either an assignment of a value to a variable or an evaluation of a factor with specific arguments.

Consider the free forward variable of the lower computation graph in Fig. \ref{fig:factor_graph_and_computation_graph} with respect to $\chi\colon s_i \mapsto x_i$. In the value of the free forward variable on the last $+$ node, which obviously takes on the form \eqref{eq:form_of_free_forward_variable}, each term corresponds to a \termref{configuration} of the upper factor graph. Moreover, each factor in a term corresponds to either an assignment of a value to a variable or an evaluation of a factor with specific arguments that is consistent with the configuration the term  corresponds to.

Further consider an example of the concrete meaning of a ``substitution.'' Let the codomain of all factors be the set of real numbers, $S = \bbR$, and $\xi$ be defined such that $\xi(s_i) = 1$ if and only if $s_i$ represents an assignment of a value to a variable, and if $s_i$ represents an evaluation of a factor with specific arguments then $\xi(s_i)$ is equal to the value of the evaluation. Then, the value of the forward variable of $\left(G, \op, \bbR, \xi\right)$ on the last $+$ node is equal to the sum of values over all the possible configurations in the upper factor graph.
\begin{figure}[p]
\centering
\includegraphics[width=0.6\columnwidth,draft=false]{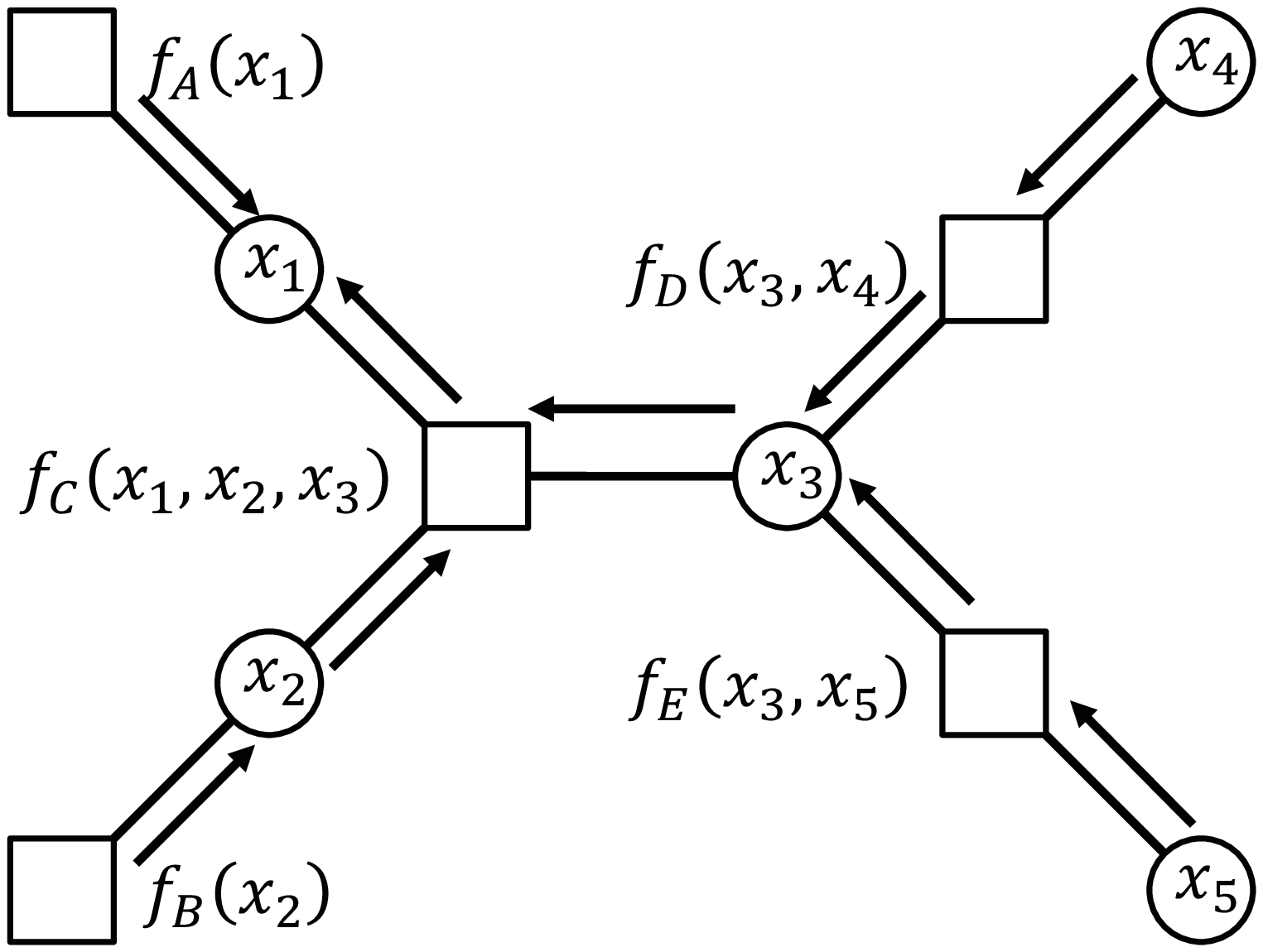}
\includegraphics[width=0.9\columnwidth,draft=false]{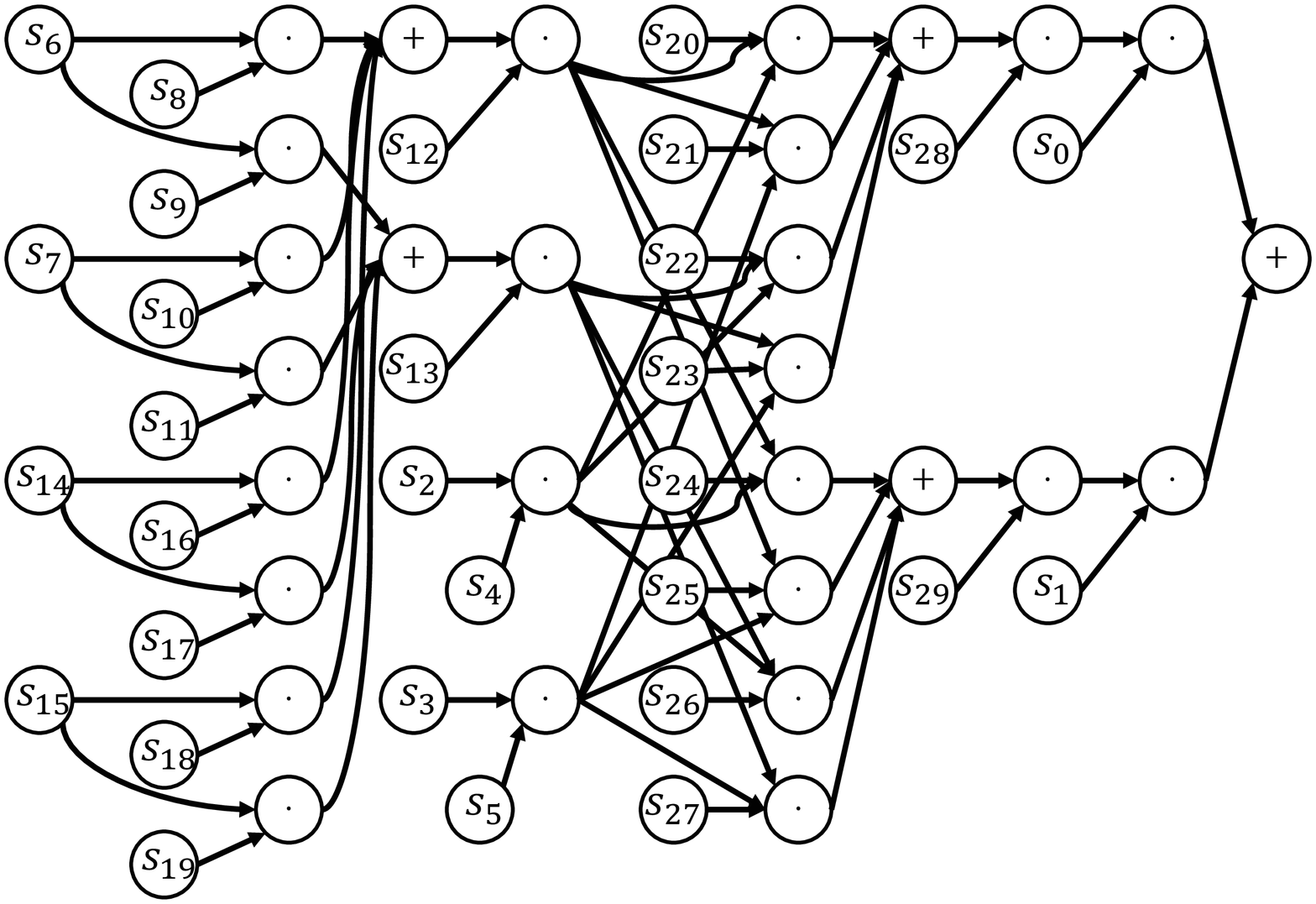}
\caption[An Example Factor Graph and Corresponding Computation Graph]{The upper diagram is an example acyclic factor graph. The directed arrows in the upper diagram indicate an example generalized forward/backward (GFB) schedule. The lower directed graph is the computation graph corresponding to the unidirectional message passing on the upper factor graph with respect to the GFB.}
\label{fig:factor_graph_and_computation_graph}
\end{figure}%
\label{ex:cg_for_acyclic_factor_graph}
\end{MyExample}
\begin{table}[t]
\centering
\begin{tabular}{|l|l||l|l|}
\hline
$s_0$    & $f_A(x_1 = 0)$          & $s_{16}$ & $f_E(x_3 = 0, x_5 = 0)$ \\
\hline
$s_1$    & $f_A(x_1 = 1)$          & $s_{17}$ & $f_E(x_3 = 1, x_5 = 0)$ \\
\hline
$s_2$    & $f_B(x_2 = 0)$          & $s_{18}$ & $f_E(x_3 = 0, x_5 = 1)$ \\
\hline
$s_3$    & $f_B(x_2 = 1)$          & $s_{19}$ & $f_E(x_3 = 1, x_5 = 1)$ \\
\hline
$s_4$    & $x_2 = 0$               & $s_{20}$ & $f_C(x_1 = 0, x_2 = 0, x_3 = 0)$ \\
\hline
$s_5$    & $x_2 = 1$               & $s_{21}$ & $f_C(x_1 = 0, x_2 = 0, x_3 = 1)$ \\
\hline
$s_6$    & $x_4 = 0$               & $s_{22}$ & $f_C(x_1 = 0, x_2 = 1, x_3 = 0)$ \\
\hline
$s_7$    & $x_4 = 1$               & $s_{23}$ & $f_C(x_1 = 0, x_2 = 1, x_3 = 1)$ \\
\hline
$s_8$    & $f_D(x_3 = 0, x_4 = 0)$ & $s_{24}$ & $f_C(x_1 = 1, x_2 = 0, x_3 = 0)$ \\
\hline
$s_9$    & $f_D(x_3 = 1, x_4 = 0)$ & $s_{25}$ & $f_C(x_1 = 1, x_2 = 0, x_3 = 1)$ \\
\hline
$s_{10}$ & $f_D(x_3 = 0, x_4 = 1)$ & $s_{26}$ & $f_C(x_1 = 1, x_2 = 1, x_3 = 0)$ \\
\hline
$s_{11}$ & $f_D(x_3 = 1, x_4 = 1)$ & $s_{27}$ & $f_C(x_1 = 1, x_2 = 1, x_3 = 1)$ \\
\hline
$s_{12}$ & $x_3 = 0$               & $s_{28}$ & $x_1 = 0$                        \\
\hline
$s_{13}$ & $x_3 = 1$               & $s_{29}$ & $x_1 = 1$ \\
\hline
$s_{14}$ & $x_5 = 0$               &          & \\
\hline
$s_{15}$ & $x_5 = 1$               &          & \\
\hline
\end{tabular}
\caption{What each source node of the computation graph in Fig. \ref{fig:factor_graph_and_computation_graph} conceptually represents.}
\label{table:factor_graph_cscg}
\end{table}%

\begin{MyExample}[Computation Graph for Zero-suppressed Binary Decision Diagram]
\label{ex:cg_for_zdd}
The left diagram in Fig. \ref{fig:zdd_and_computation_graph} illustrates an example zero-suppressed binary decision diagram~\citep[ZDD,][]{minato1993zero,knuth2009art}. The right directed graph is the computation graph representing the Boolean function expressed by the left ZDD.

The computation graph corresponding to a ZDD such as the one in Fig. \ref{fig:zdd_and_computation_graph} can be systematically constructed. It can be constructed as we traverse the ZDD from the top node to the bottom ($0$ and $1$) nodes.

In the computation graph, each source node conceptually represents an argument of the Boolean function expressed by the ZDD. $s_0$ in the computation graph in Fig \ref{fig:zdd_and_computation_graph} conceptually represents $A$ in the ZDD, $s_1$ for $B$, and $s_2$ for $C$.

Consider that the free forward variable of the right computation graph in Fig. \ref{fig:zdd_and_computation_graph} with respect to $\chi\colon s_i \mapsto x_i$. The value of the free forward variable on the sink node is equal to $x_0x_1 + x_0x_2 + x_2$. This is the polynomial representation of the Boolean function. Each term in the polynomial represents a specific joint assignment of the Boolean values to the arguments of the Boolean function that yields $1$. For example, the term $x_0 x_1$ represents the joint assignment $A = 1$, $B = 1$, and $C = 0$.

Further consider that an example of the concrete meaning of a ``substitution.'' Let $\xi$ be a function mapping each source node of the computation graph to a non-negative real number. Then, the value of the forward variable $\left(G, \op, \bbR, \xi\right)$ on the sink node is equal to the summation of the weights of all joint assignments of the Boolean values to the arguments of the Boolean function that yield $1$, each of whose weight, in turn, is expressed by the production of the weights $\xi$ over all arguments that are substituted by $1$. This value is fundamental in the combination of logic-based formalisms with statistical inference~\citep[e.g.,][]{poole1993probabilistic,sato2001parameter,de2007problog}.

Although we only show a computation graph for a ZDD as an example, we can construct computation graphs representing the Boolean functions expressed by a variety of decision diagrams, including (reduced and ordered) BDDs~\citep{akers1978binary,bryant1992symbolic}, case-factor diagrams~\citep{mcallester2004case}, and/or multi-valued decision diagrams~\citep{mateescu2008and}, and so on.
\end{MyExample}

\begin{figure}[t]
\centering
\includegraphics[width=0.5\columnwidth,draft=false]{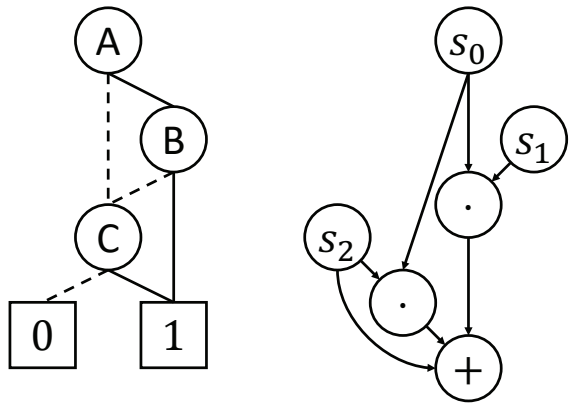}
\caption[An Example ZDD and Corresponding Computation Graph]{The left diagram is an example ZDD. The right directed graph is the computation graph representing the Boolean function expressed by the left ZDD.}
\label{fig:zdd_and_computation_graph}
\end{figure}%
\subsection{Commutative Semiring Computation Graph Parametrized by Monoid Homomorphism}
\label{subsec:commutative_semiring_computation_graph_parametrized_by_monoid_homomorphism}

Our goal in this subsection is to elucidate another aspect of modularity offered by an algebraic abstraction that is not described in past research. If an additional condition is imposed on $\xi$ of a computation graph $\left(G, \op, S, \xi\right)$, values of the forward variable of the computation graph can be shown to be equal to a form that is ``globally'' changed from the original form \eqref{eq:plain_alpha}.

Let us introduce a motivating example. Consider the dynamic programming computation of the normalization constant for linear-chain CRFs. The normalization constant is of the form
\begin{equation}
\sum_{\vecy} \exp\!\left(
  \sum_{i} \veclambda \cdot \vecF\!\left(\vecy, \vecx, i\right)\right) \; ,
\label{eq:crf_z}
\end{equation}
where $i$ ranges over positions in the sequence. The efficiency of dynamic programming is based on a step-by-step factorization of the target formula by using the distributive law of semiring. At first glance, the distributive law does not appear to be applicable to formula \eqref{eq:crf_z} since the function $\exp$ is applied to each summand. However, one can transform formula \eqref{eq:crf_z} into an equivalent one to suit the distributive law, i.e.,
\begin{equation*}
\sum_{\vecy} \exp\!\left(\veclambda \cdot \vecF\!\left(\vecy, \vecx, 1\right)\right) \cdot
             \exp\!\left(\veclambda \cdot \vecF\!\left(\vecy, \vecx, 2\right)\right) \cdots
\end{equation*}
by using the addition law of the exponential function.

The transformation of a formula like \eqref{eq:crf_z} into an equivalent one compatible with dynamic programming is very important for many machine learning tasks. In structured prediction, in particular, the number of summands may grow as an exponential function of the problem size. Therefore, naive computation of a formula like \eqref{eq:crf_z} is intractable in those cases.

This subsection presents an algebraic abstraction of the transformation mentioned above. That is, we present an algebraic condition that allows transformation of some formulas that are seemingly unsuitable for application of the distributive law and efficient computation by dynamic programming into equivalent ones that suit them. This algebraic abstraction covers some interesting and important problem instances that were not formalized in an algebraic way in past research.

\begin{MyDefinition}[Sextuple Specifying Commutative Semiring Computation Graph Parametrized by Monoid Homomorphism]
A sextuple $\left(G, \op, M, \phi, S, f\right)$ is said to \termref{specify the $f$-parametrized commutative semiring computation graph $\left(G, \op, S, f \circ \phi\right)$} if and only if:
\begin{itemize}
  \item $G = \left(V, E\right)$ is a finite dag,
  \item $\op$ is a mapping from $V \setminus \src\!\left(G\right)$ to $\left\{\text{``$+$''}, \text{``$\cdot$''}\right\}$,
  \item $M$ is a commutative monoid,
  \item $\phi$ is a mapping from $\src\!\left(G\right)$ to $M$,
  \item $S = \left(S, +, \cdot, 0_S, 1_S\right)$ is a commutative semiring, and
  \item $f$ is a monoid homomorphism from $M$ to the multiplicative monoid $\left(S, \cdot, 1_S\right)$ of $S$.
\end{itemize}
\end{MyDefinition}

\begin{MyTheorem}
Let a sextuple $(G, \op, M, \phi, S, f)$ specify the $f$-parametrized computation\linebreak graph $\calG = (G, \op, S, f \circ \phi)$, and $\src(G) = \{s_1, \dots, s_n\}$. Further assume that one defines $\chi\colon \src(G) \to \bbNzero[x_1, \dots, x_n]$ mapping $s_i$ to $x_i$ for every $s_i \in \src(G)$, and obtains the following form of the free forward variable of $(G, \op)$ w.r.t. $\chi$
\begin{equation*}
\alpha_{\left(G, \op, \chi\right)}\!\left(t\right) =
\sum_{\veci \in {\mathbb{N}_0}^n} c_{t, \veci} {x_1}^{i_1} \cdots {x_n}^{i_n}
\end{equation*}
for every node and arc $t \in V \cup E$. Then,
\begin{equation}
\alpha_{\mathcal{G}}\!\left(t\right) =
  \sum_{\veci \in {\mathbb{N}_0}^n}
    c_{t, \veci} f\!\left(\left(\phi\!\left(s_1\right)\right)^{i_1} \cdots
                          \left(\phi\!\left(s_n\right)\right)^{i_n}\right) \; .
\label{eq:parametrized_alpha}
\end{equation}
\label{thm:parametrized_alpha}
\end{MyTheorem}
\begin{proof}
The statement follows from Lemma \ref{lemma:substitution_principle_of_free_forward_variable} and the homomorphism of $f$.
\end{proof}

An important implication of Theorem \ref{thm:parametrized_alpha} is as follows. The number of summands of the summation in the right-hand side of \eqref{eq:parametrized_alpha} might be exponentially large compared to the size of $G$, and the function $f$ appears to interfere with a factorization of each summand. Therefore, at first glance, the summand of the summation of the right-hand side of \eqref{eq:parametrized_alpha} cannot be factorized nor computed efficiently by using the distributive law of the semiring. Actually, Eq. \eqref{eq:parametrized_alpha} means that it can be computed efficiently with time and space complexity proportional to the size of $G$ since $\alpha_{\mathcal{G}}$ in the left-hand side is defined and computed by recursion on $G$.

In what follows, we give some important or interesting examples of parametrized computation graphs.

\begin{MyExample}[Commutative Semiring Computation Graph Parametrized by Identity Function]
Let $S = (S, +, \cdot, 0_S, 1_S)$ be a commutative semiring, and the sextuple $(G, \op, (S, \cdot, 1_S), \phi, S, \id_S)$ specify the $\id_{S}$-parametrized computation graph where $\id_{S}$ is the identity function on $S$. Note that $\id_{S}$ is the trivial automorphism on the multiplicative monoid of $S$. The $\id_{S}$-parametrized computation graph is equivalent to the computation graph $\left(G, \op, S, \phi\right)$.
\label{ex:id_pcscg}
\end{MyExample}

Example \ref{ex:id_pcscg} shows that any computation graph can be interpreted as a computation graph parametrized by a suitable monoid homomorphism, and a formalization in terms of parametrized computation graphs is a pure extension of one in terms of plain computation graphs.

\begin{MyExample}[$\exp$-parametrized Commutative Semiring Computation Graph]
Let $\bbR = (\bbR, +, \cdot, 0, 1)$ be the ordinary semiring of real numbers, and $(G, \op, (\bbR, +, 0), \phi, \bbR, \exp)$ specify the $\exp$-parametrized computation graph $\mathcal{G} = \left(G, \op, \mathbb{R}, \exp \circ\> \phi\right)$. Note that the exponential function $\exp$ is a monoid homomorphism from $\left(\mathbb{R}, +, 0\right)$ to the multiplicative monoid $\left(\mathbb{R}, \cdot, 1\right)$ of $\mathbb{R}$, that is, $\exp\!\left(a + b\right) = \exp\!\left(a\right) \cdot \exp\!\left(b\right)$ for every $a, b \in \mathbb{R}$, which is nothing other than the addition law of the exponential function. Let $\src\!\left(G\right) = \left\{s_1, \dots, s_n\right\}$. Further assume that one defines $\chi \colon \src\!\left(G\right) \to \bbNzero\!\left[x_1, \dots, x_n\right]$ mapping $s_i$ to $x_i$ for every $s_i \in \src\!\left(G\right)$, and obtains the following form of the free forward variable of $\left(G, \op\right)$ w.r.t. $\chi$
\begin{equation*}
\alpha_{\left(G, \op, \chi\right)}\!\left(t\right) =
\sum_{\veci \in {\mathbb{N}_0}^n} c_{t, \veci} {x_1}^{i_1} \cdots {x_n}^{i_n}
\end{equation*}
for every node and arc $t \in V \cup E$. Then,
\begin{equation*}
\alpha_{\mathcal{G}}\!\left(t\right) = \sum_{\veci \in {\mathbb{N}_0}^n}
  c_{t, \veci} \exp\!\left(i_1 \phi\!\left(s_1\right) + \dots + i_n \phi\!\left(s_n\right)\right) \; .
\end{equation*}
\label{ex:exp_pcscg}
\end{MyExample}

Example \ref{ex:exp_pcscg} corresponds to the example cited in the beginning of this subsection.

\begin{MyExample}[$\left(\cos, \sin\right)$-parametrized Commutative Semiring Computation Graph]
Let $\mathcal{R}^2 = (\bbR^2, +, \cdot, (0, 0), (1, 0))$ be the commutative semiring equipped with the addition and multiplication defined by
\begin{equation*}
\begin{aligned}
\left(a, b\right) + \left(c, d\right)     &= \left(a + c, b + d\right) \text{ and} \\
\left(a, b\right) \cdot \left(c, d\right) &= \left(ac - bd, ad + bc\right) \; ,
\end{aligned}
\end{equation*}
respectively. We can easily confirm that $\mathcal{R}^2$ is a commutative semiring by the isomorphism between this semiring and the ordinary semiring of complex numbers. Further let $(\cos, \sin)\colon \bbR \to \bbR^2$ be a function that maps each $\theta \in \bbR$ to $(\cos(\theta), \sin(\theta))$, and let $(G, \op, (\bbR, +, 0), \phi, \mathcal{R}^2, (\cos, \sin))$ specify the $(\cos, \sin)$-parametrized computation graph $\calG = (G, \op, \mathcal{R}^2, (\cos, \sin) \circ \phi)$. Note that $(\cos, \sin)$ is a monoid homomorphism from $(\bbR, +, 0)$ to the multiplicative monoid of $\mathcal{R}^2$, that is, $(\cos, \sin)(a + b) = ((\cos, \sin)(a)) \cdot ((\cos, \sin)(b))$ for every $a, b \in \bbR$, which follows from the addition law of the trigonometric functions $\cos$ and $\sin$. Let $\src(G) = \{s_1, \dots, s_n\}$. Now further assume that one defines $\chi\colon \src(G) \to \bbNzero[x_1, \dots, x_n]$ mapping $s_i$ to $x_i$ for every $s_i \in \src(G)$, and obtains the following form of the free forward variable of $(G, \op)$ w.r.t. $\chi$
\begin{equation*}
\alpha_{\left(G, \op, \chi\right)}\!\left(t\right) =
\sum_{\veci \in {\mathbb{N}_0}^n} c_{t, \veci} {x_1}^{i_1} \cdots {x_n}^{i_n}
\end{equation*}
for every node and arc $t \in V \cup E$. Then,
\begin{equation*}
\alpha_{\mathcal{G}}\!\left(t\right) = \sum_{\veci \in {\mathbb{N}_0}^n}
c_{t, \veci} \left(
\left(\cos, \sin\right)\!\left(
i_1 \phi\!\left(s_1\right) + \dots + i_n \phi\!\left(s_n\right)
\right)
\right) \; .
\end{equation*}
\label{ex:cos_sin_pcscg}
\end{MyExample}

\begin{MyDefinition}[Binomial Convolution Semiring]
Let $S = \left(S, +, \cdot, 0_S, 1_S\right)$ be a commutative semiring, and $n$ a non-negative integer. For every $\left(a_i\right)_{i = 0, \dots, n},\, \left(b_i\right)_{i = 0, \dots, n} \in S^{n + 1}$, the binary operations $+$ and $\diamond$ on $S^{n + 1}$ are defined by
\begin{equation}
\left(\left(a_i\right)_{i = 0, \dots, n}\right) + \left(\left(b_i\right)_{i = 0, \dots, n}\right) =
\left(a_i + b_i\right)_{i = 0, \dots, n}
\label{eq:bc_add}
\end{equation}
and, using binomial coefficients $\binom{n}{k} = \frac{n!}{k!(n - k)!} \in \mathbb{N}_0$,
\begin{equation}
\left(\left(a_i\right)_{i = 0, \dots, n}\right) \diamond
\left(\left(b_i\right)_{i = 0, \dots, n}\right) =
\left(
  \sum_{j \in \left\{k \in \mathbb{N}_0 \relmiddle| k \leq i\right\}}
    \binom{i}{j} \left(a_j \cdot b_{i - j}\right)
\right)_{\!\!\!\!i = 0, \dots, n} \; ,
\label{eq:bc_mult}
\end{equation}
respectively. Further let $0_{\BC_S^n} = \left(0_S, \dots, 0_S\right)$ and $1_{\BC_S^n} = \left(1_S, 0_S, \dots, 0_S\right)$. Then $\BC_{S}^{n} = \left(S^{n + 1}, +, \diamond, 0_{\BC_S^n}, 1_{\BC_S^n}\right)$ is called the \termref{$n$-th order binomial convolution semiring over $S$}.
\label{def:binomial_convolution_semiring}
\end{MyDefinition}

\begin{MyLemma}
Let $S$ be a commutative semiring and $n$ a non-negative integer. Then the $n$-th order binomial convolution semiring $\BC_{S}^{n}$ over $S$ is a commutative semiring.
\label{lemma:bcs}
\end{MyLemma}

Note that $\BC_{\mathbb{R}}^{1}$ is isomorphic to the ordinary semiring of dual numbers~\citep{yaglom1968complex}. It is also known as the (first-order) expectation semiring~\citep{eisner2001expectation}. $\BC_S^1$, where $S$ is a commutative semiring, plays a vital part in the formalization of forward-backward algorithms in Section \ref{sec:forward_backward_algorithms}.

\begin{MyExample}[Parametrized Commutative Semiring Computation Graph for Sequence of Powers]
Let $S = (S, +, \cdot, 0_S, 1_S)$ be a commutative semiring, $n$ a non-negative integer, $\mathcal{P}_S^n\colon S \to S^{n + 1}$ a function that maps each $a \in S$ to $(a^i)_{i = 0, \dots, n}$, and the sextuple $(G, \op, \left(S, +, 0_S\right), \phi, \BC_{S}^{n}, $ $\mathcal{P}_S^n)$ specify the $\mathcal{P}_{S}^{n}$-parametrized computation graph $\mathcal{G} = \left(G, \op, \BC_{S}^{n}, \mathcal{P}_{S}^{n} \circ \phi\right)$. Note that $\mathcal{P}_{S}^{n}$ is a monoid homomorphism from $\left(S, +, 0_S\right)$ to the multiplicative monoid $\left(S^{n + 1}, \diamond, 1_{\BC_S^n}\right)$ of $\BC_{S}^{n}$, that is, $\mathcal{P}_{S}^{n}\!\left(a + b\right) = \left(\mathcal{P}_{S}^{n}\!\left(a\right)\right) \diamond \left(\mathcal{P}_{S}^{n}\!\left(b\right)\right)$ for every $a, b \in S$, which is nothing other than the binomial theorem.%
\footnote{
  More precisely, this is the binomial theorem generalized to any commutative semiring. The proof of the generalization is omitted because it is trivial~\citep[cf.][Exercise I.2.12]{hebisch1998algebraic}.}
Let $\src\!\left(G\right) = \left\{s_1, \dots, s_m\right\}$. Further assume that one defines $\chi \colon \src(G) \to \bbNzero\!\left[x_1, \dots, x_m\right]$ mapping $s_i$ to $x_i$ for every $s_i \in \src\!\left(G\right)$, and obtains the following form of the free forward variable of $\left(G, \op\right)$ w.r.t. $\chi$
\begin{equation*}
\alpha_{\left(G, \op, \chi\right)}\!\left(t\right) =
\sum_{\veci \in {\mathbb{N}_0}^m} c_{t, \veci} {x_1}^{i_1} \cdots {x_m}^{i_m}
\end{equation*}
for every node and arc $t \in V \cup E$. Then,
\begin{equation}
\begin{split}
\alpha_{\mathcal{G}}\!\left(t\right)
&= \sum_{\veci \in {\mathbb{N}_0}^{m}}
c_{t, \veci} \mathcal{P}_{S}^{n}\!\left(
i_1 \phi\!\left(s_1\right) + \cdots + i_m \phi\!\left(s_m\right)\right) \\
&= \sum_{\veci \in {\mathbb{N}_0}^{m}}
c_{t, \veci} \!\left(
\left(i_1 \phi\!\left(s_1\right) + \cdots + i_m \phi\!\left(s_m\right)\right)^j
\right)_{\!\!j = 0, \dots, n} \\
&= \left(\sum_{\veci \in {\mathbb{N}_0}^{m}}
c_{t, \veci} \!\left(i_1 \phi\!\left(s_1\right) + \cdots + i_m \phi\!\left(s_m\right)\right)^j
\right)_{\!\!\!\!j = 0, \dots, n} \; .
\end{split}
\label{eq:alpha_power_sequence}
\end{equation}
Note that the first equality in \eqref{eq:alpha_power_sequence} uses Theorem \ref{thm:parametrized_alpha}, and the last equality uses the fact that the addition of $\BC_{S}^n$ is done component-wise.
\label{ex:monomial_series_pcscg}
\end{MyExample}

\begin{MyExample}[Parametrized Commutative Semiring Computation Graph for Sequence of Higher-order Derivatives]
Let $n$ be a non-negative integer, $l$ a positive integer, $\mathcal{F}_{C^{n}, l}$ the set of all $n$-differentiable functions having $l$ independent variables from an open subset $D$ of $\mathbb{R}^l$ to $\mathbb{R}$, the binary operation $\cdot$ on $\mathcal{F}_{C^{n}, l}$ defined pointwise (i.e., for every $f, g \in \mathcal{F}_{C^n, l}$, $f \cdot g$ is defined as a function that maps each $\vecx \in \bbR^{l}$ to $f\!\left(\vecx\right) \cdot g\!\left(\vecx\right)$), $1_{\mathcal{F}_{C^{n}, l}} \in \mathcal{F}_{C^n, l}$ the constant function whose value is always $1$, $k \in \left\{1, \dots, l\right\}$, $\vecx_0 \in D$, $\Delta_{k, \vecx_0}^{n} \colon \mathcal{F}_{C^n, l} \to \mathbb{R}^{n + 1}$ a function that maps each $f \in \mathcal{F}_{C^{n}, l}$ to $\left(\left.\!\frac{\partial^i}{\partial {x_k}^i} f\!\left(\vecx\right)\right|_{\vecx = \vecx_0}\right)_{i = 0, \dots, n}$ where $\frac{\partial^0}{\partial {x_k}^0} f\!\left(\vecx\right) = f\!\left(\vecx\right)$ and $\vecx = \left(x_1, \dots, x_l\right)$, and the sextuple $\left(G, \op, \left(\mathcal{F}_{C^{n}, l}, \cdot, 1_{\mathcal{F}_{C^{n}, l}}\right), \phi, \BC_{\mathbb{R}}^n, \Delta_{k, \vecx_0}^n\right)$ specify the $\Delta_{k, \vecx_{0}}^{n}$-parametrized computation graph $\mathcal{G} = \left(G, \op, \BC_{\bbR}^{n}, \Delta_{k, \vecx_0}^{n} \!\circ \phi\right)$. Note that $\mathcal{F}_{C^{n}, l} = \left(\mathcal{F}_{C^{n}, l}, \cdot, 1_{\mathcal{F}_{C^{n}, l}}\right)$ is a commutative monoid, and $\Delta_{k, \vecx_0}^n$ is a monoid homomorphism from $\mathcal{F}_{C^{n}, l}$ to the multiplicative monoid $\left(\bbR^{n + 1}, \diamond, 1_{\BC_{\bbR}^{n}}\right)$ of $\BC_{\bbR}^{n}$, that is, $\Delta_{k, \vecx_0}^n\!\!\left(f \cdot g\right) = \left(\Delta_{k, \vecx_0}^n\!\!\left(f\right)\right) \diamond \left(\Delta_{k, \vecx_0}^n\!\!\left(g\right)\right)$ for every $f, g \in \mathcal{F}_{C^{n}, l}$, which is nothing other than the general Leibniz rule followed by an evaluation at the point $\vecx = \vecx_0$. Let $\src\!\left(G\right) = \left\{s_1, \dots, s_m\right\}$. Now further assume that one defines $\chi \colon \src\!\left(G\right) \to \bbNzero\!\left[y_1, \dots, y_m\right]$ mapping $s_i$ to $y_i$ for every $s_i \in \src\!\left(G\right)$, and obtains the following form of the free forward variable of $\left(G, \op\right)$ w.r.t. $\chi$
\begin{equation*}
\alpha_{\left(G, \op, \chi\right)}\!\left(t\right) =
\sum_{\veci \in {\mathbb{N}_0}^m} c_{t, \veci} {y_1}^{i_1} \cdots {y_m}^{i_m}
\end{equation*}
for every node and arc $t \in V \cup E$. Then,
\begin{equation}
\begin{split}
\alpha_{\mathcal{G}}(t)
  &= \sum_{\veci \in {\mathbb{N}_0}^m} c_{t, \veci} \Delta_{k, \vecx_0}^{n}\!\!\left(
       \left(\phi\!\left(\vecx; s_1\right)\right)^{i_1} \cdots
       \left(\phi\!\left(\vecx; s_m\right)\right)^{i_m}
     \right) \\
  &= \sum_{\veci \in {\mathbb{N}_0}^m} c_{t, \veci} \!\left(
       \left.
         \!\frac{\partial^j}{\partial {x_k}^j} \!\left(
           \left(\phi\!\left(\vecx; s_1\right)\right)^{i_1} \cdots
           \left(\phi\!\left(\vecx; s_m\right)\right)^{i_m}
         \right)
       \right|_{\vecx = \vecx_0}
     \right)_{\!\!\!j = 0, \dots, n} \\
  &= \left(
       \left.
         \!\frac{\partial^j}{\partial {x_k}^j} \!\left(
           \sum_{\veci \in {\mathbb{N}_0}^m} c_{t, \veci}
             \!\left(\phi\!\left(\vecx; s_1\right)\right)^{i_1} \cdots
               \left(\phi\!\left(\vecx; s_m\right)\right)^{i_m}
         \right)
       \!\right|_{\vecx = \vecx_0}
     \right)_{\!\!\!\!j = 0, \dots, n} \; .
\end{split}
\label{eq:alpha_diff_seq}
\end{equation}
Note that the first equality in \eqref{eq:alpha_diff_seq} uses Theorem \ref{thm:parametrized_alpha}, and the last equality uses linearity of differentiation and the fact that the addition of $\BC_{\mathbb{R}}^{n}$ is done component-wise.
\label{ex:n_differentiation_pcscg}
\end{MyExample}

Note that the instance of forward algorithms on the parametrized computation graph $\left(G, \op, \BC_{\mathbb{R}}^{1}, \Delta_{k, \vecx_{0}}^1 \circ \phi\right)$ specified by the sextuple $\left(G, \op, \mathcal{F}_{C^1, l}, \phi, \BC_{\mathbb{R}}^{1}, \Delta_{k, \vecx_{0}}^1\right)$, which is obtained by setting $n = 1$ in Example \ref{ex:n_differentiation_pcscg}, is equivalent to the forward mode of AD~\citep{griewank2008evaluating} on the computation graph $\left(G, \op, \left(\mathbb{R}, +, \cdot, 0, 1\right), \phi\right)$.

\subsection{Tensor Product of Semialgebras for Forward Algorithms}
\label{subsec:tensor_product_of_semialgebras_for_forward_algorithm}

In the previous subsection, we introduced the notion of a parametrized computation graph. The goal of this subsection is to provide a systematic way to ``compose'' a new parametrized computation graph from those that have the same computation structure $\left(G, \op\right)$. In the composed graph, the values of its forward variable are the ``composition'' of those of the original computation graphs.

In short, our contribution in this subsection is to reveal algebraic structures underlying complicated computation with forward algorithms and to construct a systematic framework to compose a complicated and difficult-to-design forward algorithm from primitive and easy-to-design forward algorithms.

To illustrate our motivation, let us introduce computation with the second-order expectation semiring~\citep{li2009first}. Roughly speaking, for a computation graph with the following free forward variable
\begin{equation*}
\alpha_{\left(G, \op, \chi\right)}\!\left(t\right)
=\sum_{\veci \in \bbNzero^n} c_{t, \veci}{x_{1}}^{i_1} \cdots {x_{n}}^{i_n} \; ,
\end{equation*}
the second-order expectation semiring is used to compute values of the form
\begin{equation}
\sum_{\veci \in \bbNzero^n}
  c_{t, \veci}
  \left(\left(\mu\!\left(s_1\right)\right)^{i_1} \cdots \left(\mu\!\left(s_n\right)\right)^{i_n}\right)
  \left(i_1\phi\!\left(s_1\right) + \cdots + i_n\phi\!\left(s_n\right)\right)
  \left(i_1\psi\!\left(s_1\right) + \cdots + i_n\psi\!\left(s_n\right)\right) \; .
\label{eq:form_by_second_order_expectation_semiring}
\end{equation}
The second-order expectation semiring is a commutative semiring on quadruples of real numbers equipped with the following addition and multiplication,
\begin{equation}
\begin{aligned}
&\left(p_{1}, r_{1}, s_{1}, t_{1}\right) + \left(p_{2}, r_{2}, s_{2}, t_{2}\right)
 =\left(p_{1} + p_{2}, r_{1} + r_{2}, s_{1} + s_{2}, t_{1} + t_{2}\right) \\
&\left(p_{1}, r_{1}, s_{1}, t_{1}\right) \cdot \left(p_{2}, r_{2}, s_{2}, t_{2}\right) \\
&\qquad=\left(p_{1}p_{2}, p_{1}r_{2} + p_{2}r_{1}, p_{1}s_{2} + p_{2}s_{1},
              p_{1}t_{2} + p_{2}t_{1} + r_{1}s_{2} + r_{2}s_{1}\right) \; .
\end{aligned}
\label{eq:second_order_semiring_adhoc}
\end{equation}
However, the derivation of this semiring in \citet{li2009first} is ad-hoc and limited to the above form.

Actually, the complex Eqs. \eqref{eq:second_order_semiring_adhoc} can be derived systematically. First observe that the summand in formula \eqref{eq:form_by_second_order_expectation_semiring} consists of $c_{t, \veci}$ and three other factors. Now consider summations with simpler summands that consist of $c_{t, \veci}$ and only one of the other three factors, i.e., sums of the form $\sum_{\veci \in \bbNzero^n} c_{t, \veci} \left(\mu\!\left(s_1\right)\right)^{i_1} \cdots \left(\mu\!\left(s_n\right)\right)^{i_n}$, $\sum_{\veci \in \bbNzero^n} c_{t, \veci} \left(i_1\phi\!\left(s_1\right) + \cdots + i_n\phi\!\left(s_n\right)\right)$, and $\sum_{\veci \in \bbNzero^n} c_{t, \veci} \left(i_1\psi\!\left(s_1\right) + \cdots + i_n\psi\!\left(s_n\right)\right)$. These simpler summations can be computed easily, as we have already specified the computation graphs for these summations in Example \ref{ex:id_pcscg}, and Example \ref{ex:monomial_series_pcscg} with $S = \mathbb{R}$ and $n = 1$. The goal of this subsection is to give a way that allows systematic derivation of \eqref{eq:second_order_semiring_adhoc} from the computation graphs for these simpler summations, and to elucidate the underlying abstract structure. At the end of this subsection, we replicate the derivation of \eqref{eq:second_order_semiring_adhoc} on a comprehensive mathematical foundation, and show how to construct a systematic way to compute \eqref{eq:form_by_second_order_expectation_semiring}.

Let us state our goal more formally. When two sextuples $(G, \op, M, \phi, S, f)$ and $(G, \op,\allowbreak M', \psi, S', g)$ are already known to specify the $f$- and $g$-parametrized computation graphs $\left(G, \op, S, f \circ \phi\right)$ and $\left(G, \op, S', g \circ \psi\right)$, respectively, and the form of the free forward variable of $\left(G, \op\right)$ with respect to $\chi \colon \src\!\left(G\right) \to \bbNzero\!\left[x_1, \dots, x_n\right]$ mapping $s_i$ to $x_i$ for every $s_i \in \src\!\left(G\right) = \left\{s_1, \dots, s_n\right\}$ is known to be equal to
\begin{equation*}
\alpha_{\left(G, \op, \chi\right)}\!\left(t\right) =
\sum_{\veci \in {\mathbb{N}_0}^n} c_{t, \veci} {x_1}^{i_1} \cdots {x_n}^{i_n}
\end{equation*}
for every node and arc $t \in V \cup E$, then the goal is to construct a systematic way to compute
\begin{equation}
\sum_{\veci \in {\mathbb{N}_0}^n} c_{t, \veci}
  B\!\left(
    f\!\left(\left(\phi\!\left(s_1\right)\right)^{i_1} \cdots
             \left(\phi\!\left(s_n\right)\right)^{i_n}\right),
    g\!\left(\left(\psi\!\left(s_1\right)\right)^{i_1} \cdots
             \left(\psi\!\left(s_n\right)\right)^{i_n}\right)
  \right) \; ,
\label{eq:forward_variable_tensor_product}
\end{equation}
where $B$ is any \termref{bilinear mapping}. Moreover, for $m$ sextuples $(G, \op, M_i, \phi_i, S_i, f_i) \;$$(i = 1, \dots, m)$ specifying the $f_i$-parametrized computation graph $\left(G, \op, S_i, f_i \circ \phi_i\right)$, respectively, we also construct a systematic way to compute
\begin{equation}
\sum_{\veci \in \bbNzero^n} c_{t, \veci}
  \mathcal{M}\!\left(
    f_1\!\left(\left(\phi_1\!\left(s_1\right)\right)^{i_1} \cdots
               \left(\phi_1\!\left(s_n\right)\right)^{i_n}\right),
    \dots,
    f_m\!\left(\left(\phi_m\!\left(s_1\right)\right)^{i_1} \cdots
               \left(\phi_m\!\left(s_n\right)\right)^{i_n}\right)
  \right) \; ,
\label{eq:forward_variable_multilinear}
\end{equation}
where $\mathcal{M}$ is any \termref{$m$-linear mapping}.

The above-mentioned goal is feasible under some reasonable assumptions. However, there are things to rigorously formalize in order to achieve the goal. Therefore, we introduce the definitions of \termref{semimodule}, \termref{basis}, \termref{semialgebra}, \termref{structure constants}, \termref{bilinear mapping}, \termref{tensor product}, \termref{$n$-linear mapping}, and so on. Semimodule just models ``vector-like'' objects. A semimodule generalizes the concept of ordinary vector space (over a field), wherein the corresponding scalars are elements of a semiring. Bases of a semimodule are a concept analogous to that of ordinary vector space. A semialgebra is a semimodule equipped with ``multiplication between vectors.'' Structure constants provide a primitive description of ``multiplication between vectors'' for a semialgebra. A bilinear mapping roughly performs a ``multiplication-like'' operation between elements of two semimodules and yields an element of another semimodule. Therefore, bilinear mappings include not only an ordinary multiplication as a closed binary operation but also various multiplication-like operations involving vectors such as scalar product, inner product, outer product, and so on. The notion of tensor product for semialgebras offers a systematic way to ``compose'' the multiplication structures of  given semialgebras as well as their domains. Once the tensor product of semialgebras is constructed, we can easily compute any bilinear mapping on the semialgebras via a succinct linear mapping. $n$-linear mapping is the ``$n$-ary extension'' of bilinear mapping.

Hereinafter, there is an application limitation of formalization. We consider only a class of semirings called \termref{cancellative} semirings. A ring is always a cancellative semiring. Thus, the ordinary semiring of real numbers and its variants, including Examples \ref{ex:exp_pcscg}, \ref{ex:cos_sin_pcscg}, and \ref{ex:n_differentiation_pcscg}, and Examples \ref{ex:id_pcscg} and \ref{ex:monomial_series_pcscg} with $S = \mathbb{R}$ are cancellative semirings. They are frequently used in a learning or optimization phase in machine learning tasks. In contrast, non-cancellative semirings, including Boolean semiring, max-plus (tropical) semiring and its variants, are not covered in the subsequent part of this paper, although they are frequently used during a prediction phase.%
\footnote{
  See \citet[Chapters 15 and 16]{golan1999semirings} for the case of non-cancellative semirings. In particular, for non-cancellative semirings that have an element $x$ such that $a + x = x$ holds for every element $a$ (e.g., Boolean semiring and max-plus semiring), the tensor product of semimodules, which is a core notion in the subsequent part, degenerates to a trivial structure, and formalization becomes also trivial and meaningless.}

Definitions, theorems, etc. introduced in this subsection are also used to formalize forward-backward algorithms in an algebraic way in the next section.

\begin{MyDefinition}[Semimodule,\footnote{In this paper, we only define semimodules over a commutative semiring. Therefore, it is irrelevant to distinguish between a \termref{left} $S$-semimodule $M$ and the \termref{right} one if we set $m\sigma = \sigma m$ for every $\sigma \in S$ and $m \in M$. Note that an $S$-semimodule is an \termref{$(S, S)$-bisemimodule}~\citep{golan1999semirings} in this setting.} {\normalfont\citealt{hebisch1998algebraic,golan1999semirings}}]
Let $M = (M, +,\allowbreak 0_M)$ be a commutative monoid, $S = (S, +, \cdot, 0_S, 1_S)$ a commutative semiring, and let there exist a mapping from $S \times M$ to $M$, denoted by the juxtaposition of an element of $S$ and an element of $M$, and called \termref{scalar multiplication}. Then $M$ is called a \termref{semimodule over $S$} or an \termref{$S$-semimodule} if and only if the following conditions are satisfied for every $\sigma, \tau \in S$ and $a, b \in M$:
\begin{itemize}
  \item $\sigma (a + b) = \sigma a + \sigma b$,
  \item $(\sigma + \tau) a = \sigma a + \tau a$,
  \item $(\sigma \cdot \tau) a = \sigma (\tau a)$,
  \item $1_S a = a$, and
  \item $0_S a = 0_M$.
\end{itemize}
\end{MyDefinition}

\begin{MyDefinition}[$S$-homomorphism, {\normalfont\citealt{hebisch1998algebraic,golan1999semirings}}]
  Let $S$ be a commutative semiring, and let $\left(M, +, 0_M\right)$ and $\left(N, +, 0_N\right)$ be $S$-semimodules. Then a mapping $f \colon M \to N$ is called an \termref{$S$-homomorphism} or a \termref{linear mapping} if and only if the following conditions are satisfied for every $\sigma \in S$ and $m, m' \in M$:
  \begin{itemize}
    \item $f\!\left(m + m'\right) = f\!\left(m\right) + f\!\left(m'\right)$, and
    \item $f\!\left(\sigma m\right) = \sigma f\!\left(m\right)$.
  \end{itemize}
\end{MyDefinition}

\begin{MyDefinition}[Basis of Semimodule, {\normalfont\citealt{hebisch1998algebraic,golan1999semirings}}]
Let $\left(M, +, 0_M\right)$ be a semimodule over a commutative semiring $\left(S, +, \cdot, 0_S, 1_S\right)$, and $U \neq \emptyset$ a subset of $M$. Then $a \in M$ is called a \termref{linear combination of elements $u \in U$ (over $S$)} if and only if $a = \sum_{u \in U} \sigma_u u$ holds for $\sigma_u \in S$ but only finitely many of the coefficients $\sigma_{u}$ are different from $0_S$. If every element in $M$ can be obtained in this way, $U$ is said to \termref{generate $\left(M, +, 0_M\right)$ by linear combinations}. Further, $U$ is called \termref{linearly independent (over $S$)} if $\sum_{u \in U} \sigma_u u = \sum_{u \in U} \tau_u u$ for only finitely many non-zero coefficients $\sigma_u, \tau_u \in S$ implies $\sigma_u = \tau_u$ for every $u \in U$. Finally, $U$ is called a \termref{basis} of $\left(M, +, 0_M\right)$ if and only if $U$ is linearly independent and generates $\left(M, +, 0_M\right)$ by linear combinations.
\end{MyDefinition}

\begin{MyDefinition}[Extension by Linearity]
Let $M = \left(M, +, 0_M\right)$ and $N = \left(N, +, 0_N\right)$ be semimodules over a commutative semiring $S$, and $U \subseteq M$ a basis of $M$. Further, for every $m \in M$, let $m = \sum_{u \in U} \sigma_{m, u} u \; \left(\sigma_{m, u} \in S\right)$ be its unique expression as a linear combination of elements of $U$. Then, a mapping $f\colon U \to N$ is said to be \termref{extended by linearity} to a mapping $g\colon M \to N$ if and only if $g$ is defined by $g\!\left(m\right) = \sum_{u \in U} \sigma_{m, u} f\!\left(u\right)$ for every $m \in M$. $g$ is well-defined since the linear combination is unique. $g$ is also called the \termref{extension by linearity} of $f$. Clearly, $g$ is an $S$-homomorphism.
\end{MyDefinition}

\begin{MyDefinition}[Bilinear ($S$-balanced) Mapping, {\normalfont\citealt{golan1999semirings}}]
Let $S$ be a commutative semiring, and let $M$, $N$, and $P$ be $S$-semimodules. Then a mapping $B \colon M \times N \to P$ is \termref{bilinear} or \termref{$S$-balanced} if and only if, for every $m, m' \in M$, $n, n' \in N$, and $\sigma \in S$, we have:
\begin{itemize}
  \item $B\left(m + m', n\right) = B\left(m, n\right) + B\left(m', n\right),$
  \item $B\left(m, n + n'\right) = B\left(m, n\right) + B\left(m, n'\right),$ and
  \item $B\left(\sigma m, n\right) = B\left(m, \sigma n\right) = \sigma B\left(m, n\right).$
\end{itemize} 
\end{MyDefinition}

\begin{MyDefinition}[Cancellativeness, {\normalfont\citealt{hebisch1998algebraic,golan1999semirings}}]
  A commutative semiring $S = \left(S, +, \cdot, 0_{S}, 1_{S}\right)$ is called \termref{cancellative} if and only if, for every $a \in S$, $a + b = a + c$ for some $b, c \in S$ implies $b = c$. A semimodule $M = \left(M, +, 0_{M}\right)$ is \termref{cancellative} if and only if, for every $m \in M$, $m + m' = m + m''$ for some $m', m'' \in M$ implies $m' = m''$.
\end{MyDefinition}

\begin{MyDefinition}[Tensor Product of Semimodules, {\normalfont\citealt{takahashi1982bordism,golan1999semirings}}\footnote{This definition is a specific case of the general definition of the tensor product of semimodules and corresponds to \citet[Corollary 4.4]{takahashi1982bordism} and \citet[Proposition 16.15]{golan1999semirings}. The general definition \citep{takahashi1982bordism,golan1999semirings} does not assume cancellativeness of $P$.}]
Let $S$ be a commutative semiring, let $M$ and $N$ be $S$-semimodules, and $P$ a cancellative $S$-semimodule. The \termref{tensor product of $M$ and $N$ over $S$}, denoted by $M \otimes_S N$, is an $S$-semimodule equipped with a bilinear mapping $\otimes \colon M \times N \to M \otimes_S N$ such that for any bilinear mapping $B \colon M \times N \to P$ there is a unique $S$-homomorphism $L \colon M \otimes_S N \to P$ making $L\!\left(m \otimes n\right) = B\!\left(m, n\right)$ for every $m \in M$ and $n \in N$. 
\label{def:tensor_product_of_semimodules}
\end{MyDefinition}

Figure \ref{fig:tensor_prod} illustrates what Definition \ref{def:tensor_product_of_semimodules} says. Roughly speaking, for any bilinear mapping $B\colon M \times N \to P$, $B\!\left(m, n\right)$ can be calculated by way of $M \otimes_{S} N$. This is not just a roundabout way of calculating $B\!\left(m, n\right)$, but of great help in analyzing a complicated $B$ because a linear-algebraic method can be used on $M \otimes_{S} N$.
\begin{figure}[t]
  \centering
  \includegraphics[width=0.3\columnwidth,draft=false]{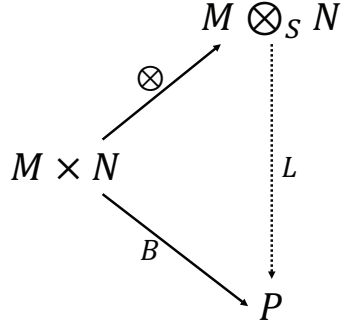}
  \caption[Tensor Product of Semimodules]{The commutative diagram for the tensor product of semimodules (Definition \ref{def:tensor_product_of_semimodules})}
  \label{fig:tensor_prod}
\end{figure}%

\begin{MyLemma}[Existence of Tensor Product of Semimodules, {\normalfont\citealt{takahashi1982bordism,golan1999semirings}}]
Let $S$ be a commutative semiring, and let $M = \left(M, +, 0_M\right)$ and $N = \left(N, +, 0_N\right)$ be $S$-semimodules. Then the tensor product $M \otimes_S N$ exists, $0_M \otimes 0_N$ is the zero element of $M \otimes_S N$, and $\left\{m \otimes n\right\}_{m \in M, n \in N}$ generates $M \otimes_S N$ by linear combinations.
\label{lemma:existence_of_tensor_product_of_semimodules}
\end{MyLemma}
\begin{proof}
See \citet[Sections 3 and 4]{takahashi1982bordism} or \citet[Chapter16]{golan1999semirings}.
\end{proof}

\begin{MyLemma}[Basis of Tensor Product of Semimodules]
Let $S$ be a cancellative commutative semiring, and let $M$ and $N$ be $S$-semimodules. Further let $U$ and $V$ be bases of $M$ and $N$, respectively. Then $\left\{u \otimes v\right\}_{u \in U,\, v \in V}$ is a basis of the tensor product $M \otimes_{S} N$.
\label{lemma:basis_of_tensor_product_of_semimodules}
\end{MyLemma}

\begin{MyDefinition}[Semialgebra, {\normalfont\citealt{hebisch1998algebraic}}]
Let $\left(S, +, \cdot, 0_S, 1_S\right)$ be a commutative semiring. Then $\left(A, +, \cdot, 0_A\right)$ is called a \termref{semialgebra over $S$} or an \termref{$S$-semialgebra} if and only if the following conditions are satisfied:
\begin{itemize}
  \item $\left(A, +, 0_A\right)$ is an $S$-semimodule,
  \item $\cdot$ is a binary operation on $A$, called \termref{multiplication}, and $\left(A, +, \cdot, 0_A\right)$ satisfies the \termref{distributive law}, i.e., $a \cdot (b + c) = (a \cdot b) + (a \cdot c)$ and $(a + b) \cdot c = (a \cdot c) + (b \cdot c)$ for every $a, b, c \in A$, and
  \item there exists a basis $U$ of $\left(A, +, 0_A\right)$ such that $(\sigma u) \cdot (\tau v) = (\sigma \cdot \tau)(u \cdot v)$ holds for every $\sigma, \tau \in S$ and $u, v \in U$. Such a basis $U$ is called a \termref{semialgebra basis of $\left(A, +, \cdot, 0_A\right)$}.
\end{itemize}
\end{MyDefinition}

An $S$-semialgebra $\left(A, +, \cdot, 0_A\right)$ is called \termref{unital} if and only if there exists the identity element of the multiplication. It is called \termref{commutative} or \termref{associative} if and only if the multiplication is commutative or associative, respectively. Hereinafter, a unital semialgebra $(A, +, \cdot, 0_A)$ where $1_A \in A$ is the identity element of the multiplication is denoted by $(A, +, \cdot, 0_A, 1_A)$.

Note that a commutative semiring $S = \left(S, +, \cdot, 0_{S}, 1_{S}\right)$ can be always considered as a commutative unital associative $S$-semialgebra with a semialgebra basis $\left\{1_{S}\right\}$. At the same time, a commutative unital associative semialgebra itself can be always considered as a commutative semiring.

In this paper, we are only interested in semialgebras that are themselves commutative semirings, that is, commutative unital associative semialgebras.

\begin{MyDefinition}[Structure Constants, {\normalfont\citealt{hebisch1998algebraic}}]
Let $\left(A, +, \cdot, 0_A\right)$ be a semialgebra over a commutative semiring $S = \left(S, +, \cdot, 0_S, 1_S\right)$, and $U$ a semialgebra basis of $\left(A, +, \cdot, 0_A\right)$. Then, for every $u, v \in U$, we obtain the unique linear combination of elements of $U$ for the product of $u$ and $v$
\begin{equation*}
u \cdot v = \sum_{w \in U} \sigma_{u, v}^w w \; ,
\end{equation*}
where $\sigma_{u, v}^w \in S$ but only finitely many of the coefficients $\sigma_{u, v}^w$ are different from $0_S$. $\sigma_{u, v}^w$ are called the \termref{structure constants of the $S$-semialgebra $\left(A, +, \cdot, 0_A\right)$ with respect to the semialgebra basis $U$}.
\end{MyDefinition}

The structure constants of a semialgebra serve as an alternative definition of the multiplication of the semialgebra. This means that, if the definition of the multiplication is given then we easily obtain the structure constants of the semialgebra, and conversely, if the structure constants are given then we can determine the multiplication between every two elements of the semialgebra by using the linear combinations of the basis for the operands.

\begin{MyDefinition}[Tensor Product of Semialgebras]
Let $S$ be a cancellative commutative semiring, $A = (A, +, \cdot, 0_A, 1_A)$ and $A' = (A', +, \cdot, 0_{A'}, 1_{A'})$ commutative unital associative $S$-semialgebras, and $U$ and $V$ semialgebra bases of $A$ and $A'$, respectively. Further let $\sigma_{u, u'}^{u''} \left(u, u', u'' \in U\right)$ (resp. $\tau_{v, v'}^{v''} \left(v, v', v'' \in V\right)$) be the structure constants of $A$ (resp. $A'$) with respect to $U$ (resp. $V$). Let $t$ and $t'$ be two elements of $A \otimes_S A'$, and let $t = \sum_{\left(u, v\right) \in U \times V} \rho_{u, v} \left(u \otimes v\right)$, and $t' = \sum_{\left(u, v\right) \in U \times V} \rho'_{u, v} \left(u \otimes v\right)$ be their unique expressions as linear combinations of the basis $\left\{u \otimes v\right\}_{u \in U, v \in V}$. We define the binary operation $t \cdot t'$ of the two elements by
\begin{equation}
t \cdot t' = \sum_{\left(u, v\right) \in U \times V}
               \sum_{\left(u', v'\right) \in U \times V}
                 \sum_{\left(u'', v''\right) \in U \times V}
                   \left(\rho_{u, v} \cdot \rho'_{u', v'} \cdot
                         \sigma_{u, u'}^{u''} \cdot \tau_{v, v'}^{v''}\right)
                   \left(u'' \otimes v''\right) \; .
\label{eq:multiplication_of_tensor_product_of_semialgebras}
\end{equation}
The tensor product of $S$-semimodules $A \otimes_{S} A'$ equipped with this operation is called the \termref{tensor product of semialgebras $A$ and $A'$ over $S$}.
\label{def:tensor_product_of_semialgebras}
\end{MyDefinition}

\begin{MyLemma}
Let $S, A, A', U, V, \sigma_{u, u'}^{u''}, \tau_{v, v'}^{v''},$ and the operation $\cdot$ be defined as in Definition \ref{def:tensor_product_of_semialgebras}. Then, the following statements hold:
\begin{itemize}
  \item $A \otimes_{S} A' = \left(A \otimes_{S} A', +, \cdot, 0_{A} \otimes 0_{A'}, 1_{A} \otimes 1_{A'}\right)$ is a commutative unital associative $S$-semialgebra,
  \item $W = \left\{u \otimes v\right\}_{u \in U, v \in V}$ is a semialgebra basis of $A \otimes_{S} A'$,
  \item the structure constants of $A \otimes_{S} A'$ with respect to the basis $W$ are $\omega_{u \otimes v, u' \otimes v'}^{u'' \otimes v''} = \sigma_{u, u'}^{u''} \cdot \tau_{v, v'}^{v''} \; \left(u \otimes v, u' \otimes v', u'' \otimes v'' \in W\right)$, and
  \item for every $a, b \in A$ and $a', b' \in A'$,
  \begin{equation}
  \left(a \otimes a'\right) \cdot \left(b \otimes b'\right) =
    \left(a \cdot b\right) \otimes \left(a' \cdot b'\right) \; .
  \label{eq:multiplication_between_elementary_tensors}
  \end{equation}
\end{itemize}
\label{lemma:tensor_product_of_semialgebras}
\end{MyLemma}

The following theorem shows how to compute formula \eqref{eq:forward_variable_tensor_product}.

\begin{MyTheorem}
Let $(S, +, \cdot, 0_S, 1_S)$ be a cancellative commutative semiring, $A = (A, +, \cdot, 0_A,\allowbreak 1_A)$ and $A' = (A', +, \cdot, 0_{A'}, 1_{A'})$ commutative unital associative $S$-semialgebras, $U$ and $U'$ semialgebra bases of $A$ and $A'$, respectively, and the sextuples $(G, \op, (M, \cdot, 1_M), \phi, A, f)$ and $(G, \op, (N, \cdot, 1_N), \psi, A', g)$ specify the $f$- and $g$-parametrized computation graphs $(G, \op, A,\allowbreak f \circ \phi)$ and $(G, \op, A', g \circ \psi)$, respectively. Note that $M \times N = (M \times N, \cdot, (1_M, 1_N))$ is a commutative monoid equipped with component-wise multiplication $(m, n) \cdot (m', n') = (m \cdot m', n \cdot n')$ for every $m, m' \in M$ and $n, n' \in N$. Further let $\phi \times \psi \colon \src(G) \to M \times N$ be a function that maps each $v \in \src(G)$ to $(\phi(v), \psi(v))$, $f \otimes g \colon M \times N \to A \otimes_S A'$ a function that maps each $(m, n) \in M \times N$ to $f(m) \otimes g(n)$, and $\src(G) = \{s_1, \dots, s_n\}$. Now further assume that one defines $\chi\colon \src(G) \to \bbNzero[x_1, \dots, x_n]$ mapping $s_i$ to $x_i$ for every $s_i \in \src(G)$, and obtains the following form of the free forward variable of $(G, \op)$ w.r.t. $\chi$
\begin{equation*}
\alpha_{(G, \op, \chi)}(t) =
\sum_{\veci \in \bbNzero^n} c_{t, \veci}{x_1}^{i_1} \cdots {x_n}^{i_n} 
\end{equation*}
for every node and arc $t \in V \cup E$. Then,
\begin{equation*}
\alpha_{\calG}(t) = \sum_{\veci \in \bbNzero^n} c_{t, \veci}
\!\left(
  \left(
    f\!\left(\left(\phi\!\left(s_1\right)\right)^{i_1} \cdots
             \left(\phi\!\left(s_n\right)\right)^{i_n}\right)
  \right) \otimes
  \left(
    g\!\left(\left(\psi\!\left(s_1\right)\right)^{i_1} \cdots
             \left(\psi\!\left(s_n\right)\right)^{i_n}\right)
  \right)
\right) \; ,
\end{equation*}
where $\calG = (G, \op, A \otimes_S A', (f \otimes g) \circ (\phi \times \psi))$ is the computation graph specified by the sextuple $(G, \op, M \times N, \phi \times \psi, A \otimes_S A', f \otimes g)$. Moreover, for any bilinear mapping $B$ from $A \times A'$ to a cancellative $S$-semimodule $P$, we can construct an $S$-homomorphism $L \colon A \otimes_S A' \to P$ such that
\begin{equation*}
L\!\left(\alpha_{\mathcal{G}}\!\left(t\right)\right) =
\sum_{\veci \in \bbNzero^n} c_{t, \veci}
B\!\left(f\!\left(\left(\phi\!\left(s_1\right)\right)^{i_1} \cdots
\left(\phi\!\left(s_n\right)\right)^{i_n}\right),
g\!\left(\left(\psi\!\left(s_1\right)\right)^{i_1} \cdots
\left(\psi\!\left(s_n\right)\right)^{i_n}\right)\right) \; .
\end{equation*}
In fact, the extension by linearity of the mapping $L'(u \otimes u') = B(u, u')$ for every $u \in U$ and $u' \in U'$ has this very effect.
\label{thm:alpha_for_tensor_product_of_semialgebras}
\end{MyTheorem}

Let $A, A', A''$ be $S$-semialgebras, and $U, U', U''$ semialgebra bases of $A, A', A''$, respectively. Then, there exists the trivial isomorphism between $(A \otimes_{S} A') \otimes_{S} A''$ and $A \otimes_{S} \left(A' \otimes_{S} A''\right)$, that is, the mapping $u \otimes \left(u' \otimes u''\right) \mapsto \left(u \otimes u'\right) \otimes u'' \; \left(u \in U, u' \in U', u'' \in U''\right)$ is extended by linearity to the isomorphism. Therefore, tensor product is essentially associative, and we can omit parentheses indicating the association order of multiple tensor products.

By noting this fact, and using Theorem \ref{thm:alpha_for_tensor_product_of_semialgebras} and Lemmas \ref{lemma:basis_of_tensor_product_of_semimodules} and \ref{lemma:tensor_product_of_semialgebras}, we obtain the following statement how to compute formula \eqref{eq:forward_variable_multilinear}.

\begin{MyCorollary}
Let $S$ be a cancellative commutative semiring, $n$ a non-negative integer, $A_i = (A_{i}, +, \cdot, 0_{A_i}, 1_{A_i})$ $(i = 1, \dots, n)$ commutative unital associative $S$-semialgebras, $U_i \; (i = 1, \dots, n)$ a semialgebra basis of $A_i$, the sextuples $(G, \op, M_i, \phi_i, A_i, f_i) \;$$(i = 1, \dots, n)$ specify the $f_i$-parametrized computation graph $(G, \op, A_i, f_i \circ \phi_i)$, respectively, $\phi_1 \times \cdots \times \phi_n\colon \src(G) \to M_1 \times \cdots \times M_n$ a function that maps each $v \in \src(G)$ to $(\phi_i(v))_{i = 1, \dots, n}$, and $f_1 \otimes \cdots \otimes f_n$ a function that maps each $(m_i)_{i = 1, \dots, n} \in M_1 \times \cdots \times M_n$ to $(f_1(m_1)) \otimes \cdots \otimes (f_n(m_n))$. Then the sextuple $(G, \op, M_1 \times \cdots \times M_n, \phi_1 \times \cdots \times \phi_n, A_1 \otimes_S \cdots \otimes_S A_n, f_1 \otimes \cdots \otimes f_n)$ specifies the parametrized computation graph $\calG = (G, \op, A_1 \otimes_S \cdots \otimes_S A_n, (f_1 \otimes \cdots \otimes f_n) \circ (\phi_1 \times \cdots \times \phi_n))$. Let $\src(G) = \{s_1, \dots, s_\ell\}$. Now further assume that one defines $\chi\colon \src(G) \to \bbNzero[x_1, \dots, x_\ell]$ mapping $s_i$ to $x_i$ for every $s_i \in \src(G)$, and obtains the following form of the free forward variable of $(G, \op)$ w.r.t. $\chi$
\begin{equation*}
\alpha_{\left(G, \op, \chi\right)}\!\left(t\right) =
\sum_{\veci \in \bbNzero^\ell} c_{t, \veci} {x_{1}}^{i_{1}} \cdots {x_\ell}^{i_\ell}
\end{equation*}
for every node and arc $t \in V \cup E$. Then,
\begin{equation*}
\begin{split}
&\alpha_{\mathcal{G}}\!\left(t\right) = \\
&  \sum_{\veci \in \bbNzero^\ell} c_{t, \veci}\!\left(
\left(f_{1}\!\left(\left(\phi_{1}\!\left(s_{1}\right)\right)^{i_{1}} \cdots
\left(\phi_{1}\!\left(s_\ell\right)\right)^{i_\ell}\right)\right) \otimes
\cdots \otimes
\left(f_{n}\!\left(\left(\phi_{n}\!\left(s_{1}\right)\right)^{i_{1}} \cdots
\left(\phi_{n}\!\left(s_\ell\right)\right)^{i_\ell}\right)\right)
\right) \; .
\end{split}
\end{equation*}
Moreover, for any $n$-linear mapping $\mathcal{M}$ from $M_1 \times \cdots \times M_n$ to a cancellative $S$-semimodule $P$ (i.e., a mapping that is linear if all but one of its arguments are fixed, that is, $\mathcal{M}(\dots, m_i + {m_i}', \dots) = \mathcal{M}(\dots, m_i, \dots) + \mathcal{M}\!\left(\dots, {m_i}', \dots\right)$ and $\mathcal{M}(\dots, \sigma m_i, \dots) = \sigma\mathcal{M}(\dots, m_i, \dots)$ hold for every $m_i, {m_i}' \in M_{i}$ and $\sigma \in S$), we can construct an $S$-homomorphism $L\colon A_1 \otimes_S \cdots \otimes_S A_n \to P$ such that
\begin{equation*}
\begin{split}
&L\!\left(\alpha_{\mathcal{G}}\!\left(t\right)\right) = \\
&  \sum_{\veci \in \bbNzero^\ell} c_{t, \veci} \mathcal{M}\!\left(
\left(f_{1}\!\left(\left(\phi_{1}\!\left(s_{1}\right)\right)^{i_{1}} \cdots
\left(\phi_{1}\!\left(s_\ell\right)\right)^{i_\ell}\right)\right), \cdots,
\left(f_{n}\!\left(\left(\phi_{n}\!\left(s_{1}\right)\right)^{i_{1}} \cdots
\left(\phi_{n}\!\left(s_\ell\right)\right)^{i_\ell}\right)\right)
\right) \; .
\end{split}
\end{equation*}
In fact, the extension by linearity of the mapping $L'\!\left(u_1 \otimes \cdots \otimes u_n\right) = \mathcal{M}\!\left(u_1, \dots, u_n\right)$ for every $u_i \in U_i$ has this very effect.
\label{corollary:alpha_for_multiple_tensor_product_of_semialgebras}
\end{MyCorollary}
\begin{proof}
By induction with Theorem \ref{thm:alpha_for_tensor_product_of_semialgebras}, and Lemmas \ref{lemma:basis_of_tensor_product_of_semimodules} and \ref{lemma:tensor_product_of_semialgebras}, we can easily prove the statement.
\end{proof}

Here, we present a systematic framework to design complicated forward algorithms.
\begin{MyFramework}[Framework to Design Complicated Forward Algorithms]
\ 
\begin{enumerate}
\item Identify the underlying abstract computation structure $\left(G, \op\right)$.
\item If computation problem on $\left(G, \op\right)$ at hand is of the form
\begin{equation}
\begin{aligned}
&\sum_{\veci \in \bbNzero^n} c_{t, \veci}
   L_{1}\!\left(
     f_{1}\!\left(
       \left(\phi_{1}\!\left(s_{1}\right)\right)^{i_1} \cdots
       \left(\phi_{1}\!\left(s_{n}\right)\right)^{i_n}
     \right)
   \right) \cdots \\
&\phantom{\sum_{\veci \in \bbNzero^n} c_{t, \veci} \quad}
   L_{m}\!\left(
     f_{m}\!\left(
       \left(\phi_{m}\!\left(s_{1}\right)\right)^{i_1} \cdots
       \left(\phi_{m}\!\left(s_{n}\right)\right)^{i_n}
     \right)
   \right) \; ,
\end{aligned}
\label{eq:complicated_forward_algorithms}
\end{equation}
where $t \in V \cup E$, $\veci = \left(i_1, \dots, i_n\right)$, $S$ is a cancellative commutative semiring, $P$ is a unital commutative associative $S$-semialgebra, and, for every $j \in \left\{1, \dots, m\right\}$, $A_j$ is a unital commutative associative $S$-semialgebra, $\left(G, \op, M_j, \phi_j, A_j, f_j\right)$ specifies the $f_j$-parametrized computation graph $\left(G, \op, A_j, f_j \circ \phi_j\right)$, and $L_j\colon A_j \to P$ is an $S$-homomorphism,
\item then define the $m$-linear mapping $\mathcal{M}\!\left(a_1, \dots, a_m\right) = L_1\!\left(a_1\right) \cdots L_m\!\left(a_m\right)$ for every $a_j \in A_j$, and construct the extension by linearity ${L'}^*\colon A_1 \otimes_S \cdots \otimes_S A_m \to P$ of the mapping $L'\!\left(u_1 \otimes \cdots \otimes u_m\right) = \mathcal{M}\!\left(u_1, \dots, u_m\right)$ for every $u_j \in U_j$, where $U_j$ is a semialgebra basis of $A_j$.
\item Thus, ${L'}^*\!\left(\alpha_{\mathcal{G}}\!\left(t\right)\right)$ is equal to \eqref{eq:complicated_forward_algorithms} by Corollary \ref{corollary:alpha_for_multiple_tensor_product_of_semialgebras}, where $\mathcal{G}$ is the parametrized computation graph $\mathcal{G} = \left(G, \op, A_{1} \times \cdots \times A_{m}, \left(f_1 \otimes \cdots \otimes f_m\right) \circ \left(\phi_1 \times \cdots \phi_m\right)\right)$ specified by the sextuple $\left(G, \op, M_1 \times \cdots \times M_m, \phi_1 \times \cdots \times \phi_m, A_1 \otimes_S \cdots \otimes_S A_m, f_1 \otimes \cdots \otimes f_m\right)$.
\end{enumerate}
\label{framework}
\end{MyFramework}

In this framework, it is important to have a large ``portfolio'' of parametrized computation graphs, including Examples \ref{ex:id_pcscg} through \ref{ex:n_differentiation_pcscg}, because it is a key to identify a computation problem at hand as one that takes on the form \eqref{eq:complicated_forward_algorithms}.

The rest of this subsection demonstrates how to apply this framework to concrete contexts.

\begin{MyExample}[Marginalization and Expectation Problems]
\label{ex:marginalization_and_expectation_problems}
Consider an abstract computation structure $\left(G, \op\right)$, and the following computation problem
\begin{equation}
\sum_{\veci \in \bbNzero^n} c_{t, \veci}
  \left(
    \left(\phi\!\left(s_1\right)\right)^{i_1}
      \cdots \left(\phi\!\left(s_n\right)\right)^{i_n}
  \right)
  \left(
    i_1 \psi\!\left(s_1\right) + \cdots + i_n \psi\!\left(s_n\right)
  \right) \; ,
\label{eq:unnormalized_first_order_moment}
\end{equation}
where $\src\!\left(G\right) = \left\{s_1, \dots, s_n\right\}$, $\veci = \left(i_1, \dots, i_n\right)$, $t \in V \cup E$, $c_{t, \veci} \in \bbNzero$ for every $t$ and $\veci$, $\phi$ and $\psi$ are a function that maps each element of $\src\!\left(G\right)$ to a real number. Both marginalization and expectation problems take on the form \eqref{eq:unnormalized_first_order_moment}. They are important computation problems in many inference procedures.

For example, see Example \ref{ex:cg_for_sequence_labeling}. Let $\phi$ in \eqref{eq:unnormalized_first_order_moment} be equal to $\xi$ defined in Example \ref{ex:cg_for_sequence_labeling}. Then, as explained in Example \ref{ex:cg_for_sequence_labeling}, the factor involving $\phi$ in the summand in \eqref{eq:unnormalized_first_order_moment} is equal to the joint probability of a sequence in HMMs. Therefore, for example, if $\psi$ is defined as the indicator function that returns $1$ on $s_0$ of the computation graph in Fig. \ref{fig:trellis_and_computation_graph} and $0$ on all the other source nodes, then \eqref{eq:unnormalized_first_order_moment} is equal to the marginal probability with respect to the residence in the state $0$ at time $t = 0$. Likewise, marginalization with respect to any state or transition takes on the form \eqref{eq:unnormalized_first_order_moment}.

Expectation problem also takes on the form \eqref{eq:unnormalized_first_order_moment}. The factor involving $\psi$ in the summand in \eqref{eq:unnormalized_first_order_moment} is equal to the ``count'' of a feature defined by $\psi$ in a sequence. Therefore, \eqref{eq:unnormalized_first_order_moment} can be interpreted as the feature expectation (first-order moment) defined by a feature $\psi$ with respect to the probability weight defined by the factor involving $\phi$.

Here, note that formula \eqref{eq:unnormalized_first_order_moment} is of a form that Framework \ref{framework} applies to. For \eqref{eq:complicated_forward_algorithms}, by setting $m = 2$, $f_1 = \id_\bbR$, $L_1 = \id_\bbR$, $f_2 = \mathcal{P}_\bbR^1$, and $L_2\colon \bbR^2 \to \bbR$, where $\mathcal{P}_\bbR^1$ is that defined in Example \ref{ex:monomial_series_pcscg}, and $L_2$ is a function that maps each $(a_0, a_1)$ to $a_1$, \eqref{eq:complicated_forward_algorithms} becomes equal to \eqref{eq:unnormalized_first_order_moment}.

Thus, we systematically construct the instance of forward algorithms that is used to compute \eqref{eq:unnormalized_first_order_moment}. From Framework \ref{framework}, ${L'}^*\!\left(\alpha_{\mathcal{G}}\!\left(t\right)\right)$ is equal to formula \eqref{eq:unnormalized_first_order_moment}, where $\mathcal{G} = \left(G, \op, \bbR \otimes_{\bbR} \BC_\bbR^1, \left(\id_\bbR \otimes \mathcal{P}_\bbR^1\right) \circ \left(\phi \times \psi\right) \right)$ is the computation graph specified by the sextuple $\left(G, \op, (\bbR, \cdot, 1) \times (\bbR, +, 0), \phi \times \psi, \bbR \otimes_{\bbR} \BC_\bbR^1, \id_\bbR \otimes \mathcal{P}_\bbR^1\right)$, $\hat{e}_0 = \left(1, 0\right) \in \bbR^2$, $\hat{e}_1 = \left(0, 1\right) \in \bbR^2$, and ${L'}^*\colon \bbR \otimes_\bbR \BC_\bbR^1 \to \bbR$ is the $\bbR$-homomorphism that is obtained by extension by linearity of the mapping $L'\colon \left\{1 \otimes \hat{e}_i\right\}_{i = 0, 1} \to \bbR$ such that $L'\!\left(1 \otimes \hat{e}_0\right) = 0$ and $L'\!\left(1 \otimes \hat{e}_1\right) = 1$. Note that the resulting algebraic structure is isomorphic to the (first-order) expectation semiring~\citep{eisner2001expectation,li2009first}.

The above-mentioned arguments can be applied to data structures other than trellises. The same applies to acyclic factor graphs (e.g., by repeating the above arguments for Example \ref{ex:cg_for_acyclic_factor_graph}), hypergraphs, a variety of decision diagrams (e.g., by repeating the above arguments for Example \ref{ex:cg_for_zdd}), and so on.
\end{MyExample}

Example \ref{ex:marginalization_and_expectation_problems} may be a little surprising. It says that both marginalization and expectation problems can be worked out only by forward passes on computation graphs. It does not appear to agree with the fact that these computation problems are usually solved by the combination of forward and backward passes (e.g., the ordinary forward-backward algorithm on trellises, the inside-outside algorithm on CYK derivations, the sum-product algorithm on acyclic factor graphs, the EM algorithm on decision diagrams~\citep{ishihata2008propositionalizing}, and so on). However, ``forward-only'' algorithms for these computations has been already proposed for certain kinds of data structures, e.g., the forward-only algorithm for HMMs~\citep{tan1993adaptive,sivaprakasam1995forward,turin1998unidirectional,miklos2005linear,churbanov2008implementing} and the forward algorithm with the (first-order) expectation semiring on hypergraphs~\citep{li2009first}. Example \ref{ex:marginalization_and_expectation_problems} generalizes these algorithms to any computation consisting of a finite number of applications of additions and/or multiplications on various kinds of data structures.

With a slight modification to Example \ref{ex:marginalization_and_expectation_problems}, we also obtain the forward-only algorithm to compute a feature expectation for CRFs on trellises and factor graphs and the log-linear model on various data structures. That is, it is obtained by replacing the factor involving $\phi$ in \eqref{eq:unnormalized_first_order_moment} with $\exp(i_1 \phi(s_1) + \cdots + i_n \phi(s_n))$ and considering the parametrized computation graph specified by the sextuple $(G, \op, (\bbR, +, 0) \times (\bbR, +, 0), \phi \times \psi, \bbR \otimes_\bbR \BC_\bbR^1, \exp \otimes \mathcal{P}_\bbR^1)$.

Now, let us return to the example of the second-order expectation semiring introduced in the beginning of this subsection. We replicate the algebraic structure and computation result of the second-order semiring, but simply as an instance of Framework \ref{framework}.

\begin{MyExample}[Second-order Expectation Semiring]
First observe that formula \eqref{eq:form_by_second_order_expectation_semiring} is an instance of formula \eqref{eq:complicated_forward_algorithms}. In fact, for \eqref{eq:complicated_forward_algorithms}, by setting $m = 3$, $\phi_1 = \mu$, $f_1 = \id_\bbR$, $L_1 = \id_\bbR$, $\phi_2 = \phi$, $f_2 = \mathcal{P}_\bbR^1$, $L_2\!\left(r_0, r_1\right) = r_1$ for every $r_0, r_1 \in \bbR$, $\phi_3 = \psi$, $f_3 = \mathcal{P}_\bbR^1$, $L_3 = L_2$, \eqref{eq:complicated_forward_algorithms} becomes equal to \eqref{eq:form_by_second_order_expectation_semiring}.

Thus, we can systematically construct the instance of forward algorithms that computes \eqref{eq:form_by_second_order_expectation_semiring}. From Framework \ref{framework}, ${L'}^*(\alpha_\calG(t))$ is equal to \eqref{eq:form_by_second_order_expectation_semiring}, where $\calG = (G, \op, \bbR \otimes_\bbR \BC_\bbR^1 \otimes_\bbR \BC_\bbR^1, (\id_\bbR \otimes \mathcal{P}_\bbR^1 \otimes \mathcal{P}_\bbR^1) \circ (\mu \times \phi \times \psi))$ is the parametrized computation graph specified by the sextuple $\bigl(G, \op, (\bbR, \cdot, 1) \times (\bbR, +, 0) \times (\bbR, +, 0), \mu \times \phi \times \psi, \bbR \otimes_\bbR \BC_\bbR^1 \otimes_\bbR \BC_\bbR^1, \id_\bbR \otimes \mathcal{P}_\bbR^1 \otimes \mathcal{P}_\bbR^1\bigr)$, $\hat{e}_0 = \left(1, 0\right) \in \bbR^2$, $\hat{e}_1 = \left(0, 1\right) \in \bbR^2$, and ${L'}^*\colon \bbR \otimes_\bbR \BC_\bbR^1 \otimes_\bbR \BC_\bbR^1 \to \bbR$ is the $\bbR$-homomorphism that is obtained by extension by linearity of the mapping $L'\colon \left\{1 \otimes \hat{e}_i \otimes \hat{e}_j\right\}_{i, j = 0, 1} \to \bbR$ such that $L'\!\left(1 \otimes \hat{e}_i \otimes \hat{e}_j\right) = 1$ if $i = j = 1$ and $L'\!\left(1 \otimes \hat{e}_i \otimes \hat{e}_j\right) = 0$ otherwise.

Next, we replicate the derivation of the algebraic structure of the second-order semiring, that is, Eqs. \eqref{eq:second_order_semiring_adhoc}. The ordinary semiring of $\mathbb{R}$ can be viewed as a commutative unital associative $\mathbb{R}$-semialgebra with a semialgebra basis $\left\{1\right\}$. The structure constants of this semialgebra w.r.t. this basis is $\rho_{1, 1}^{1} = 1$. $\BC_{\mathbb{R}}^{1}$ can be viewed as a commutative unital associative $\mathbb{R}$-semialgebra with a semialgebra basis $\left\{\hat{e}_{0}, \hat{e}_{1}\right\}$. The structure constants of this semialgebra w.r.t. this basis is $\sigma_{\hat{e}_{0}, \hat{e}_{0}}^{\hat{e}_{0}} = 1$, $\sigma_{\hat{e}_{0}, \hat{e}_{0}}^{\hat{e}_{1}} = 0$, $\sigma_{\hat{e}_{0}, \hat{e}_{1}}^{\hat{e}_{0}} = 0$, $\sigma_{\hat{e}_{0}, \hat{e}_{1}}^{\hat{e}_{1}} = 1$, $\sigma_{\hat{e}_{1}, \hat{e}_{0}}^{\hat{e}_{0}} = 0$, $\sigma_{\hat{e}_{1}, \hat{e}_{0}}^{\hat{e}_{1}} = 1$, $\sigma_{\hat{e}_{1}, \hat{e}_{1}}^{\hat{e}_{0}} = 0$, and $\sigma_{\hat{e}_{1}, \hat{e}_{1}}^{\hat{e}_{1}} = 0$. Consider the tensor product of semialgebras $\mathbb{R} \otimes_{\mathbb{R}} \BC_{\mathbb{R}}^{1} \otimes_{\mathbb{R}} \BC_{\mathbb{R}}^{1}$. By noting that $\mathbb{R}$ is a cancellative commutative semiring, $W = \left\{1 \otimes \hat{e}_{i} \otimes \hat{e}_{j}\right\}_{i \in \left\{0, 1\right\},\, j \in \left\{0, 1\right\}}$ is a semialgebra basis of $\mathbb{R} \otimes_{\mathbb{R}} \BC_{\mathbb{R}}^{1} \otimes_{\mathbb{R}} \BC_{\mathbb{R}}^{1}$, so every element of $\mathbb{R} \otimes_{\mathbb{R}} \BC_{\mathbb{R}}^{1} \otimes_{\mathbb{R}} \BC_{\mathbb{R}}^{1}$ can be written in the form $\sum_{i \in \left\{0, 1\right\},\, j \in \left\{0, 1\right\}} \eta_{i, j}\!\left(1 \otimes \hat{e}_{i} \otimes \hat{e}_{j}\right)$ for some $\eta_{i, j} \in \mathbb{R}$. Therefore, the addition of $\mathbb{R} \otimes_{\mathbb{R}} \BC_{\mathbb{R}}^{1} \otimes_{\mathbb{R}} \BC_{\mathbb{R}}^{1}$ is defined by
\begin{equation*}
\begin{aligned}
&\phantom{=}\left(\eta_{0, 0}\!\left(1 \otimes \hat{e}_{0} \otimes \hat{e}_{0}\right) +
                  \eta_{0, 1}\!\left(1 \otimes \hat{e}_{0} \otimes \hat{e}_{1}\right) +
                  \eta_{1, 0}\!\left(1 \otimes \hat{e}_{1} \otimes \hat{e}_{0}\right) +
                  \eta_{1, 1}\!\left(1 \otimes \hat{e}_{1} \otimes \hat{e}_{1}\right)\right) \\
&\phantom{=}\qquad+
            \left(\eta_{0, 0}'\!\left(1 \otimes \hat{e}_{0} \otimes \hat{e}_{0}\right) +
                  \eta_{0, 1}'\!\left(1 \otimes \hat{e}_{0} \otimes \hat{e}_{1}\right) +
                  \eta_{1, 0}'\!\left(1 \otimes \hat{e}_{1} \otimes \hat{e}_{0}\right) +
                  \eta_{1, 1}'\!\left(1 \otimes \hat{e}_{1} \otimes \hat{e}_{1}\right)\right) \\
&=\left(\eta_{0, 0} + \eta_{0, 0}'\right)\!\left(1 \otimes \hat{e}_{0} \otimes \hat{e}_{0}\right) +
  \left(\eta_{0, 1} + \eta_{0, 1}'\right)\!\left(1 \otimes \hat{e}_{0} \otimes \hat{e}_{1}\right) \\
&\phantom{=}\qquad+
  \left(\eta_{1, 0} + \eta_{1, 0}'\right)\!\left(1 \otimes \hat{e}_{1} \otimes \hat{e}_{0}\right) +
  \left(\eta_{1, 1} + \eta_{1, 1}'\right)\!\left(1 \otimes \hat{e}_{1} \otimes \hat{e}_{1}\right) \; .
\end{aligned}
\end{equation*}
The structure constants of $\mathbb{R} \otimes_{\mathbb{R}} \BC_{\mathbb{R}}^{1} \otimes_{\mathbb{R}} \BC_{\mathbb{R}}^{1}$ w.r.t. $W$ are $\tau_{1 \otimes \hat{e}_{i_{1}} \otimes \hat{e}_{j_{1}}, 1 \otimes \hat{e}_{i_{2}} \otimes \hat{e}_{j_{2}}}^{1 \otimes \hat{e}_{i_{3}} \otimes \hat{e}_{j_{3}}} = \rho_{1, 1}^{1} \cdot \sigma_{\hat{e}_{i_{1}}, \hat{e}_{i_{2}}}^{\hat{e}_{i_{3}}} \cdot \sigma_{\hat{e}_{j_{1}}, \hat{e}_{j_{2}}}^{\hat{e}_{j_{3}}} \; \left(i_{1}, i_{2}, i_{3}, j_{1}, j_{2}, j_{3} \in \left\{0, 1\right\}\right)$. Therefore, the multiplication of $\mathbb{R} \otimes_{\mathbb{R}} \BC_{\mathbb{R}}^{1} \otimes_{\mathbb{R}} \BC_{\mathbb{R}}^{1}$ is defined by
\begin{equation*}
\begin{aligned}
&\phantom{=}
 \left(\eta_{0, 0}\!\left(1 \otimes \hat{e}_{0} \otimes \hat{e}_{0}\right) +
       \eta_{0, 1}\!\left(1 \otimes \hat{e}_{0} \otimes \hat{e}_{1}\right) +
       \eta_{1, 0}\!\left(1 \otimes \hat{e}_{1} \otimes \hat{e}_{0}\right) +
       \eta_{1, 1}\!\left(1 \otimes \hat{e}_{1} \otimes \hat{e}_{1}\right)\right) \\
&\phantom{=}\qquad\cdot
 \left(\eta_{0, 0}'\!\left(1 \otimes \hat{e}_{0} \otimes \hat{e}_{0}\right) +
       \eta_{0, 1}'\!\left(1 \otimes \hat{e}_{0} \otimes \hat{e}_{1}\right) +
       \eta_{1, 0}'\!\left(1 \otimes \hat{e}_{1} \otimes \hat{e}_{0}\right) +
       \eta_{1, 1}'\!\left(1 \otimes \hat{e}_{1} \otimes \hat{e}_{1}\right)\right) \\
&=
 \left(\eta_{0, 0} \eta_{0, 0}'\right)\!\left(1 \otimes \hat{e}_{0} \otimes \hat{e}_{0}\right) \\
&\phantom{=}\qquad+
 \left(\eta_{0, 0} \eta_{0, 1}' + \eta_{0, 1} \eta_{0, 0}'\right)\!\left(
   1 \otimes \hat{e}_{0} \otimes \hat{e}_{1}\right) +
 \left(\eta_{0, 0} \eta_{1, 0}' + \eta_{1, 0} \eta_{0, 0}'\right)\!\left(
   1 \otimes \hat{e}_{1} \otimes \hat{e}_{0}\right) \\
&\phantom{=}\qquad+
 \left(\eta_{0, 0} \eta_{1, 1}' + \eta_{1, 1} \eta_{0, 0}' +
       \eta_{0, 1} \eta_{1, 0}' + \eta_{1, 0} \eta_{0, 1}'\right)\!\left(
   1 \otimes \hat{e}_{1} \otimes \hat{e}_{1}
 \right) \; .
\end{aligned}
\end{equation*}
It is obvious that the algebraic structure with the above addition and multiplication is isomorphic to the one with \eqref{eq:second_order_semiring_adhoc}.
\end{MyExample}

\section{Forward-backward Algorithms}
\label{sec:forward_backward_algorithms}

\subsection{Backward Invariants and Forward-backward Algorithms}
\label{subsec:backward_invariants_and_forward_backward_algorithms}

The goal of this section is to answer the question ``what can be calculated by forward-backward algorithms.'' More precisely, computation by forward-backward algorithms is characterized in an algebraic way.

First of all, we articulate the answer to the above question proposed in this paper. Forward-backward algorithms compute $\sum_{v \in \snk(G)} \alpha_\calG(v)$ of the parametrized computation graph $\calG = (G, \op, A \otimes_S \BC_S^1, (f \otimes g) \circ (\phi \times \psi))$ that is specified by the sextuple $(G, \op, M \times N, \phi \times \psi, A \otimes_S \BC_S^1, f \otimes g)$, where $S$ is a cancellative commutative semiring, $A$ is a commutative unital associative $S$-semialgebra, and $\left(G, \op, M, \phi, A, f\right)$ and $\left(G, \op, N, \psi, \BC_S^1, g\right)$ specify the $f$- and $g$-parametrized computation graphs $\left(G, \op, A, f \circ \phi\right)$ and $\left(G, \op, \BC_S^1, g \circ \psi\right)$, respectively. Note that $\BC_{S}^{1}$ is the first-order binomial convolution semiring over $S$ (see Definition \ref{def:binomial_convolution_semiring}). In other words, computation by some instances of forward algorithms can be replaced by that of forward-backward algorithms.

Despite its name ``forward-backward algorithms,'' the proposed characterization subsumes a quite wide range of existing algorithms, including the ordinary forward-backward algorithm on trellises for sequence labeling, the inside-outside algorithm on derivation forests for CYK parsing, the sum-product algorithm on acyclic factor graphs, the EM algorithm on a variety of decision diagrams, the reverse mode of AD, and so on. In addition, not only the standard version of these algorithms but also their variants are formalized in a unified way.

Forward-backward algorithms formalized in this section are characterized by a combination of forward and backward passes on the computation graph specified by the sextuple $\left(G, \op, M \times N, \phi \times \psi, A \otimes_{S} \BC_{S}^{1}, f \otimes g\right)$ instead of a forward-only pass on the computation graph. Since the ``dimension'' of $\BC_{S}^{1}$ is equal to 2, every element of $A \otimes_{S} \BC_{S}^{1}$ can be considered as a 2-dimensional vector-like object by ignoring the $S$-semialgebra structure of $A$. Hereinafter, each of the components is called the \termref{zeroth component} and \termref{first component}, respectively. Roughly speaking, the first step of forward-backward algorithms consists of a forward pass on the computation graph to compute the zeroth component, and the second one a backward pass to compute the first component.

The characterization of forward-backward algorithms described above leads to an immediate but very important consequence. That is, computation by forward-backward algorithms can be also done by forward algorithms. This consequence has been partly referred to in Example \ref{ex:marginalization_and_expectation_problems} and the following paragraph already.

Moreover, we can turn things around. Computation of some instances of forward algorithms can be also done by forward-backward algorithms. One of the most important implications of this fact is the derivation of the reverse mode of AD (a.k.a. back propagation) from the forward mode. Because the forward mode of AD is an instance of forward algorithms (cf. Example \ref{ex:n_differentiation_pcscg}) and can be replaced with the corresponding forward-backward algorithm, the reverse mode of AD can be derived by the formalization in this subsection.

Note that $\BC_S^1$ is a central player in the subsequent part. It cannot be overemphasized that this algebra is important in the formalization of forward-backward algorithms. For example, the following characteristics of forward-backward algorithms are entirely ascribable to the multiplication structures of $\BC_S^1$: forward-backward algorithms can be split into two stages of a forward pass and backward one, and a characteristic ``sum-product'' computation pattern appears in a backward pass.

To formalize forward-backward algorithms, additional definitions, lemmas, and a theorem are introduced below. Before going into the details of the formalization, we give a rough sketch of the flow of the discussion below. First, we introduce two mappings named the \termref{zeroth projection} $\zerothproj$ and \termref{first projection} $\firstproj$ (Definition \ref{def:zeroth_and_first_projections}) to pick up each (the zeroth or first) component of elements of $A \otimes_S \BC_S^1$, and prove some properties of the projections (Lemmas \ref{lemma:zeroth_and_first_projection} and \ref{lemma:linearity_of_projection}) used in the subsequent part. Then, we show that the zeroth component of values of the forward variable of computation graphs can be computed independently from their first component (Lemma \ref{lemma:zeroth_projection_of_alpha}). This computation of the zeroth component constitutes the ``forward part'' of forward-backward algorithms. After that, we introduce definitions and statements to ``reverse'' computation of the first component of values of the forward variable. First, we transform calculations of the first component on multiplication nodes into ``backwardable'' ones (Lemma \ref{lemma:first_projection_of_prod_alpha}). Next, we introduce the backward variable $\beta_{\mathcal{G}}$ (Definition \ref{def:beta}), which is defined by backward recursion and designed such that the sum of the values of a function of $\beta_{\mathcal{G}}$ over $\src\!\left(G\right)$ is equal to the first component of the sum of the values of $\alpha_{\mathcal{G}}$ over $\snk\!\left(G\right)$. The remaining part is devoted to prove the equation $\firstproj\!\left(\sum_{v \in \snk\!\left(G\right)} \alpha_{\mathcal{G}}\!\left(v\right)\right) = \sum_{v \in \src\!\left(G\right)} \text{(a function of $\beta_{\mathcal{G}}\!\left(v\right)$)}$. We introduce terms and concepts of order theory, and consider computation graphs as partially ordered sets by using Definition \ref{def:poset_induced_by_dag}. \termref{Antichain cutsets}, which are defined in Definition \ref{def:antichain_cutset}, play a key role in the proof of interest. Antichain cutsets are ``cut sets that are crossed against the direction of a computation graph.'' We can construct a series of antichain cutsets along the direction of a computation graph. We prove the equation of interest by induction on the series of antichain cutsets. Lemma \ref{lemma:covering_antichain_cutset} constitutes induction steps of the proof, and Theorem \ref{thm:backward_invariants} is the final result. Thereupon, we show that the first component of $\sum_{v \in \snk\!\left(G\right)} \alpha_{\mathcal{G}}\!\left(v\right)$ can be computed by $\beta_{\mathcal{G}}$ in place of $\alpha_{\mathcal{G}}$. Computation of values of $\beta_{\mathcal{G}}$ by backward recursion constitutes the ``backward part'' of forward-backward algorithms.

\begin{MyDefinition}[Zeroth and First Projections of $A \otimes_S \BC_S^1$]
Let $S = (S, +, \cdot, 0_S, 1_S)$ be a cancellative commutative semiring, $A = (A, +, \cdot, 0_A, 1_A)$ a commutative unital associative $S$-semialgebra, $U$ a semialgebra basis of $A$, and $\hat{e}_0 = (1_S, 0_S), \hat{e}_1 = (0_S, 1_S)$. Then $\{\hat{e}_0, \hat{e}_1\}$ is clearly a semialgebra basis of $\BC_S^1$. If every element of $A \otimes_S \BC_S^1$ is written as a linear combination of elements of the basis $\{u \otimes \hat{e}_i\}_{u \in U,\, i \in \{0, 1\}}$, say, $\sum_{u \in U} \sum_{i \in \{0, 1\}} \sigma_{u, i}(u \otimes \hat{e}_i) = \left(\sum_{u \in U} \sigma_{u, 0} u\right) \otimes \hat{e}_0 + \left(\sum_{u \in U} \sigma_{u, 1} u\right) \otimes \hat{e}_1$, $\sigma_{u, i} \in S$ are uniquely determined, and thus so are $\sum_{u \in U} \sigma_{u, 0} u$ and $\sum_{u \in U} \sigma_{u, 1} u$. Therefore, for every element $\sum_{u \in U} \sum_{i \in \{0, 1\}} \sigma_{u, i}(u \otimes \hat{e}_i)$ of $A \otimes_S \BC_S^1$, the two mappings from $A \otimes_S \BC_S^1$ to $A$, the \termref{zeroth projection} $\zerothproj$
\begin{equation*}
\zerothproj\!\left(\sum_{u \in U} \sum_{i \in \{0, 1\}} \sigma_{u, i}(u \otimes \hat{e}_i)\right) = \sum_{u \in U} \sigma_{u, 0} u
\end{equation*}
and the \termref{first projection} $\firstproj$
\begin{equation*}
\firstproj\!\left(\sum_{u \in U} \sum_{i \in \{0, 1\}} \sigma_{u, i}(u \otimes \hat{e}_i)\right) = \sum_{u \in U} \sigma_{u, 1} u
\end{equation*}
can be defined.
\label{def:zeroth_and_first_projections}
\end{MyDefinition}

\begin{MyLemma}
Let $S$, $A$, $U$, $\hat{e}_{0}$, and $\hat{e}_{1}$ be defined as in Definition \ref{def:zeroth_and_first_projections}. Then $x = \mathsf{P}_{0}\!\left(x\right) \otimes \hat{e}_{0} + \mathsf{P}_{1}\!\left(x\right) \otimes \hat{e}_{1}$ holds for every $x \in A \otimes_{S} \BC_{S}^{1}$.
\label{lemma:zeroth_and_first_projection}
\end{MyLemma}

\begin{MyLemma}
Let $S$, $A$, and $U$ be defined as in Definition \ref{def:zeroth_and_first_projections}. Then $\mathsf{P}_{0}\!\left(x + y\right) = \mathsf{P}_{0}\!\left(x\right) + \mathsf{P}_{0}\!\left(y\right)$ and $\mathsf{P}_{1}\!\left(x + y\right) = \mathsf{P}_{1}\!\left(x\right) + \mathsf{P}_{1}\!\left(y\right)$ hold for every two elements $x, y$ of $A \otimes_{S} \BC_{S}^{1}$.
\label{lemma:linearity_of_projection}
\end{MyLemma}

\begin{MyLemma}
Let $S = (S, +, \cdot, 0_S, 1_S)$ be a cancellative commutative semiring, $A = (A, +, \cdot, 0_A,\allowbreak 1_A)$ a commutative unital associative $S$-semialgebra, and $\calG = (G, \op, A \otimes_S \BC_S^1, \xi)$ a computation graph. Further let $\calG'$ be the computation graph $\calG' = (G, \op, A, \zerothproj \circ \xi)$. Then, for every node and arc $t \in V \cup E$ of $G$, we have
\begin{equation}
\zerothproj(\alpha_{\calG}(t)) = \alpha_{\calG'}(t) \; .
\label{eq:zeroth_projection_of_alpha}
\end{equation}
\label{lemma:zeroth_projection_of_alpha}
\end{MyLemma}

\begin{MyLemma}
Let $S$, $A$, and $\mathcal{G} = \left(G, \op, A \otimes_S \BC_S^1, \xi\right)$ be defined as in Lemma \ref{lemma:zeroth_projection_of_alpha}. Then, for every node $v \in V$, we have
\begin{equation*}
\mathsf{P}_{1}\!\left(
  \prod_{e \in E_{G}^{-}\left(v\right)} \alpha_{\mathcal{G}}\!\left(e\right)
\right) =
\sum_{e \in E_{G}^{-}\left(v\right)}
  \mathsf{P}_{1}\!\left(\alpha_{\mathcal{G}}\!\left(e\right)\right) \cdot
  \left(
    \prod_{e' \in E_{G}^{-}\left(v\right) \setminus \left\{e\right\}}
      \mathsf{P}_{0}\!\left(\alpha_{\mathcal{G}}\!\left(e'\right)\right)
  \right) \; .
\end{equation*}
\label{lemma:first_projection_of_prod_alpha}
\end{MyLemma}

\begin{MyDefinition}[Backward Variable]
Let $S$, $A$, and $\mathcal{G} = \left(G, \op, A \otimes_{S} \BC_{S}^{1}, \xi\right)$ be defined as in Lemma \ref{lemma:zeroth_projection_of_alpha}. Then the \termref{backward variable} of $\mathcal{G}$, denoted by $\beta_{\mathcal{G}}$, is a mapping from $V \cup E$ to $A$ that is defined by, for every node $v \in V$ and arc $e \in E$,
\begin{equation}
\begin{aligned}
\beta_{\mathcal{G}}\!\left(v\right) &= \begin{cases}
1_{A} & \text{if $v \in \snk\!\left(G\right)$,} \\
\sum_{e' \in E_{G}^{+}\!\left(v\right)} \beta_{\mathcal{G}}\!\left(e'\right) & \text{otherwise,}
\end{cases} \\
\beta_{\mathcal{G}}\!\left(e\right) &= \begin{cases}
\beta_{\mathcal{G}}\!\left(\head\!\left(e\right)\right)
\hspace{90pt}\text{if $\op\!\left(\head\!\left(e\right)\right) = \text{``$+$''}$,} \\
\begin{aligned}
&\beta_{\mathcal{G}}\!\left(\head\!\left(e\right)\right) \cdot
   \left(\prod_{e' \in E_{G}^{-}\!\left(\head\left(e\right)\right) \setminus \left\{e\right\}}
           \mathsf{P}_{0}\!\left(\alpha_{\mathcal{G}}\!\left(e'\right)\right)\right) \\
&\hspace{142pt}\text{otherwise (i.e., $\op\!\left(\head\!\left(e\right)\right) =\> $``$\cdot$'').}
\end{aligned}
\end{cases}
\end{aligned}
\label{eq:betadef}
\end{equation}
\label{def:beta}
\end{MyDefinition}

Note that the mutually recursive definition of $\beta_{\mathcal{G}}$ on $V$ and $E$ in Definition \ref{def:beta} is well-defined since $G$ is a finite dag.

Here, let us note that a characteristic ``sum-product'' computation pattern appears in the computation of values of a backward variable $\beta_{\mathcal{G}}$. By expanding the recursive definition of $\beta_{\mathcal{G}}\!\left(v\right)$ in Eqs. \eqref{eq:betadef} in just one step, we obtain the following form
\begin{equation*}
\beta_{\mathcal{G}}\!\left(v\right) =
\sum_{e \in E_G^+\!\left(v\right)}
\beta_{\mathcal{G}}\!\left(e\right) \cdot \left(
  \prod_{e' \in E_G^-\!\left(\head\!\left(e\right)\right) \setminus \left\{e\right\}}
  \zerothproj\!\left(\alpha_{\mathcal{G}}\!\left(e'\right)\right)
\right) \; .
\end{equation*}
This computation pattern can be commonly found in the ordinary forward-backward algorithm on trellises, the inside-outside algorithm on derivations by CYK parsing, the sum-product algorithm on acyclic factor graphs (as the name suggests), the EM algorithm on a variety of decision diagrams, and even the reverse mode of AD. On trellises for sequence labeling, this sum-product computation pattern degenerates into a ``forward-backward'' computation pattern since, in the computation graph corresponding to a trellis (e.g., Fig. \ref{fig:trellis_and_computation_graph}), multiplication nodes always have exactly two parent nodes and one of the parent nodes is always a source node.

Hereinafter, formalization is given in terms of order theory rather than graph theory, as the former provides more useful terms and concepts for our purpose.

In a partially ordered set (poset hereinafter) $\mathcal{O} = \left(O, \leq\right)$, $x \in O$ is said to be \termref{covered by} $y \in O$ (or $y$ \termref{covers} $x$) if and only if $x < y$ and there does not exist any $z \in O$ such that $x < z$ and $z < y$. Let $x \prec y$ denote that $x$ is covered by $y$. For $x \in O$, the \termref{covering set} of $x$ is defined by $\left\{y \in O \relmiddle| x \prec y\right\}$. Likewise, the \termref{covered set} of $x$ is defined by $\left\{y \in O \relmiddle| y \prec x\right\}$. For every two elements $x, y \in O$, $x$ and $y$ are called \termref{comparable} if and only if either $x \leq y$ or $y \leq x$ holds, and \termref{incomparable} if and only if they are not comparable, i.e., neither $x \leq y$ nor $y \leq x$ holds. Let $X$ be a subset of $O$. $X$ is called a \termref{chain} if and only if every two elements of $X$ are comparable (in other words, $X$ is a totally ordered subset of $O$). $X$ is called an \termref{antichain} if and only if every two distinct elements of $X$ are incomparable. $\upperset X$ is defined by $\upperset X = \left\{y \in O \relmiddle| \exists x \in X \;\; \text{s.t.} \;\; x \leq y \right\}$.\footnote{A subset $U \subseteq O$ of a poset $\left(O, \leq\right)$ is called an \termref{upper set} if and only if $u \leq x$ implies $x \in U$ for every $u \in U$ and $x \in O$, and the notation $\upperset X$ where $X \subseteq O$ usually denotes the \termref{smallest upper set of $O$ containing $X$} in order theory. However, this definition is equivalent to our one (the proof is omitted). We use the latter as it is easier to understand. The same applies to $\lowerset X$ by replacing upper set with \termref{lower set}.} $\lowerset X$ is defined likewise, i.e., $\lowerset X = \left\{y \in O \relmiddle| \exists x \in X \;\; \text{s.t.} \;\; y \leq x\right\}$.

\begin{MyDefinition}[Poset Induced by DAG]
Let $G = \left(V, E\right)$ be a dag. Then a poset $\mathcal{O} = \left(V \cup E, \leq\right)$ is called the \termref{poset induced by $G$} if and only if, $x \leq y$ holds for every $x, y \in V \cup E$ if and only if there exists a directed path $\left(v_0, e_1, v_1, \dots, e_l, v_l\right)$ in $G$ and either of the following conditions holds:
\begin{itemize}
  \item there exist $i, j \in \mathbb{N}_0$ such that $i \leq j$, $x = v_i$, and $y = v_j$,
  \item there exist $i, j \in \mathbb{N}$ such that $i \leq j$, $x = e_i$, and $y = e_j$,
  \item there exist $i \in \mathbb{N}_0$ and $j \in \mathbb{N}$ such that $i < j$, $x = v_i$, and $y = e_j$, or
  \item there exist $i \in \mathbb{N}$ and $j \in \mathbb{N}_0$ such that $i \leq j$, $x = e_i$, and $y = v_j$.
\end{itemize}
In other words, $x \leq y$ holds if and only if $y$ is ``reachable'' from $x$ but ``reachability'' includes one not only among nodes but also among nodes and arcs.
\label{def:poset_induced_by_dag}
\end{MyDefinition}

\begin{MyDefinition}[Antichain Cutset, {\normalfont\citealt{rival1985antichain}}]
Let $\mathcal{O} = \left(O, \leq\right)$ be a poset, and $C$ a subset of $O$. Then $C$ is called an \termref{antichain cutset} of $\mathcal{O}$ if and only if $C$ is an antichain of $\mathcal{O}$ and $C$ intersects every maximal chain of $\mathcal{O}$.
\label{def:antichain_cutset}
\end{MyDefinition}

Let $\mathcal{O}$ be a poset and $\AC\!\left(\mathcal{O}\right)$ the set of all antichain cutsets of $\mathcal{O}$. Then $\left(\AC\!\left(\mathcal{O}\right), \leq\right)$ constitutes a poset~\citep[actually the \termref{lattice of antichain cutsets},][]{higgs1986lattices} if we set
\begin{equation*}
C \leq C' \text{ if and only if, for every } x \in C \text{, there exists } x' \in C' \text{ such that } x \leq x'
\end{equation*}
for every two antichain cutsets $C, C' \in \AC\!\left(\mathcal{O}\right)$.

As the following lemma shows, if a poset is induced by a dag, one can obtain an antichain cutset ``adjacent'' to a given antichain cutset.\footnote{%
  We can construct a poset such that there is an arbitrarily complicated substructure but no antichain cutset between two antichain cutsets. See Fig. 4 (and Fig. 5) in \citet{rival1985antichain} for example. However, in the poset induced by a dag, we can find another antichain cutset right close to every given antichain cutset. This is because the poset induced by a dag falls into a tractable subset of posets, as pointed out in the footnote \ref{footnote:tractable_posets}.}
Moreover, an invariant on forward and backward variables holds for both the adjacent antichain cutsets.

\begin{MyLemma}
Let $S$, $A$, and $\mathcal{G} = \left(G, \op, A \otimes_S \BC_S^1, \xi\right)$ be defined as in Lemma \ref{lemma:zeroth_projection_of_alpha}. Further let $\mathcal{O} = \left(V \cup E, \leq\right)$ be the poset induced by $G$, $\AC\!\left(\mathcal{O}\right)$ the set of all antichain cutsets of $\mathcal{O}$, $\left(\AC\!\left(\mathcal{O}\right), \leq\right)$ the lattice of antichain cutsets of $\mathcal{O}$, and $C \subseteq V \cup E$ an antichain cutset of $\mathcal{O}$ satisfying $C \neq \snk\!\left(G\right)$. Then one can obtain $x \in C$ such that there exists $y \in V \cup E$ satisfying $x \prec y$ and $y$ is a minimal element of $\uparrow\!\!C \setminus C$. Further let $D'$ be the covering set of $x$. Then, for every $d' \in D'$, the covered set of $d'$ is a subset of $C$. Further let $D = \left\{z \in V \cup E \relmiddle| \exists w \in D' \;\; \text{s.t.} \;\; z \prec w\right\}$. Then $C' = \left(C \cup D'\right) \setminus D$ is an antichain cutset of $\mathcal{O}$ satisfying $C < C'$. $C'$ is called the \termref{covering antichain cutset of $C$ with respect to $x$}. Moreover, we have
\begin{equation*}
\sum_{c \in C} \mathsf{P}_{1}\!\left(
                 \alpha_{\mathcal{G}}\!\left(c\right)\right)
                   \cdot \beta_{\mathcal{G}}\!\left(c\right) =
\sum_{c \in C'} \mathsf{P}_{1}\!\left(
                  \alpha_{\mathcal{G}}\!\left(c\right)\right)
                    \cdot \beta_{\mathcal{G}}\!\left(c\right) \; .
\end{equation*}
\label{lemma:covering_antichain_cutset}
\end{MyLemma}

Here, we present the main theorem of this subsection. Intuitively, the theorem says that computation of the first component of $\sum_{v \in \snk\!\left(G\right)} \alpha_{\mathcal{G}}\!\left(v\right)$ of $\mathcal{G} = \left(G, \op, A \otimes_S \BC_S^1, \xi\right)$ can be made in terms of the backward variable. This result guarantees the correctness of computation of the first component of the values of the forward variable by a backward pass in forward-backward algorithms.

\begin{MyTheorem}[Backward Invariants]
Let $S$, $A$, $\mathcal{G} = \left(G, \op, A \otimes_S BC_S^1, \xi\right)$, and $\mathcal{O}$ be defined as in Lemma \ref{lemma:covering_antichain_cutset}. Then, for every two antichain cutsets $C, C' \in \AC\!\left(\mathcal{O}\right)$,
\begin{equation*}
\sum_{c \in C} \mathsf{P}_{1}\!\left(\alpha_{\mathcal{G}}\!\left(c\right)\right)
                 \cdot \beta_{\mathcal{G}}\!\left(c\right) =
\sum_{c \in C'} \mathsf{P}_{1}\!\left(\alpha_{\mathcal{G}}\!\left(c\right)\right)
                  \cdot \beta_{\mathcal{G}}\!\left(c\right) \; .
\end{equation*}
Especially,
\begin{equation*}
\mathsf{P}_{1}\!\left(\sum_{v \in \snk\!\left(G\right)} \alpha_{\mathcal{G}}\!\left(v\right)\right) =
\sum_{v \in \src\!\left(G\right)} \mathsf{P}_{1}\!\left(\xi\!\left(v\right)\right)
                                    \cdot \beta_{\mathcal{G}}\!\left(v\right) \; .
\end{equation*}
\label{thm:backward_invariants}
\end{MyTheorem}

Now, the pseudo-codes of forward and forward-backward algorithms are shown in Algorithms \ref{alg:forward} and \ref{alg:forward_backward}, respectively. Forward algorithms just compute values of the forward variable of a computation graph. Although this variable has been constructively defined in Definition \ref{def:alpha} already, the pseudo-code of forward algorithms is also presented here in order to discuss two important problems, ``reduction of space complexity'' and ``scheduling'' in forward algorithms.

Antichain cutsets of the poset induced by a computation graph are a key to reduce space complexity in forward and forward-backward algorithms because an antichain cutset is a minimal subset of $V \cup E$ where values of the forward (resp. backward) variable should be stored at each point in time during execution of a forward (resp. backward) pass. The reason is explained as follows. If values of the forward (resp. backward) variable are stored only in a subset of $V \cup E$ that some maximal chain of the poset induced by $G$ does not intersect, values of the forward (resp. backward) variable cannot be computed on some nodes or arcs. If values of the forward (resp. backward) variable are stored in a subset of $V \cup E$ that is not an antichain of the poset induced by $G$, some of them are not necessary or can be obtained later during further computation.

\begin{algorithm}[t]
\caption{Forward Algorithm}
\begin{algorithmic}[1]
\Require $G = \left(V, E\right)$ is a finite dag.
         $\op$ is a mapping from $\src\!\left(G\right)$
         to $\left\{\text{``$+$''}, \text{``$\cdot$''}\right\}$.
         $S$ is a commutative semiring.
         $\xi$ is a mapping from $\src\!\left(G\right)$ to $S$.
         $\left(V \cup E, <\right)$ is the poset induced by $G$.
\Ensure $\mathcal{G} = \left(G, \op, S, \xi\right)$ then
        $\alpha\!\left(t\right) = \alpha_{\mathcal{G}}\!\left(t\right)$.
\Procedure{Forward}{$G$, $\op$, $S$, $\xi$}
\ForAll{$v \in \src\!\left(G\right)$}
  \State $\alpha\!\left(v\right) \gets \xi\!\left(v\right)$
\EndFor
\State $C \gets \src\!\left(G\right)$
\While{$C \neq \snk\!\left(G\right)$}
  \State Find $x \in C$ s.t. $\exists y \in\; \uparrow\!\!C \setminus C.\, x \prec y$
         and $y$ is a minimal element of $\uparrow\!\!C \setminus C$.
  \State $D' \gets \left\{d' \in V \cup E \relmiddle| x \prec d'\right\}$
  \ForAll{$t \in D'$}
    \If{$t \in V$}
      \If{$\op\!\left(t\right) = \text{``$+$''}$}
        \State $\alpha\!\left(t\right) \gets
                \sum_{e \in E_{G}^{-}\!\left(t\right)} \alpha\!\left(e\right)$
      \Else
        \State $\alpha\!\left(t\right) \gets
                \prod_{e \in E_{G}^{-}\!\left(t\right)} \alpha\!\left(e\right)$
      \EndIf
    \Else
      \State $\alpha\!\left(t\right) \gets \alpha\!\left(\tail\!\left(t\right)\right)$
    \EndIf
  \EndFor
  \State $D \gets \left\{d \in V \cup E \relmiddle| \exists d' \in D'.\, d \prec d'\right\}$
  \State $C \gets \left(C \cup D'\right) \setminus D$
\EndWhile
\EndProcedure
\end{algorithmic}
\label{alg:forward}
\end{algorithm}
Therefore, Algorithm \ref{alg:forward} is described on an antichain cutset to antichain cutset basis. In the {\normalfont\textbf{while}} loop from Line 6 through Line 22, values of the forward variable are computed with respect to each antichain cutset. Lemma \ref{lemma:covering_antichain_cutset} guarantees that each update of $C$ in Line 21 preserves the loop invariant that $C$ is an antichain cutset. If one is interested in the values of the forward variable only on $\snk\!\left(G\right)$ and does not execute forward-backward algorithms that require intermediate values of the forward variable, the values of the forward variable on the subset $D$ of the old antichain cutset in Line 20 are no longer necessary and can be discarded. By doing so, one can minimalize space complexity of forward algorithms for such a case.%
\footnote{
  ``Minimalization'' of space complexity described here should be distinguished from the problem of ``minimization'' of space complexity. The problem of ``minimization'' of space complexity of forward algorithms for such a case described here could be formalized as the minimization problem of the max number of values of the forward variable that are simultaneously stored. In short, this problem is formalized as a \termref{(one-shot) pebble game}. It is easy to show that this minimization problem is NP-hard for the class of arbitrary computation graphs~\citep[cf.][]{wu2014inapproximability} but the details are omitted because it is out of the scope of this paper.}

\begin{algorithm}[t]
\caption{Forward-Backward Algorithm}
\begin{algorithmic}[1]
\Require $G = \left(V, E\right)$ is a finite dag.
         $\op$ is a mapping from $\src\!\left(G\right)$
         to $\left\{\text{``$+$''}, \text{``$\cdot$''}\right\}$.
         $S$ is a cancellative commutative semiring.
         $A = \left(A, +, \cdot, 0_{A}, 1_{A}\right)$ is
         a commutative unital associative $S$-semialgebra.
         $\xi$ is a mapping from $\src\!\left(G\right)$ to $A \otimes_{S} \BC_{S}^{1}$.
         $\left(V \cup E, <\right)$ is the poset induced by $G$.
\Ensure $\sum_{v \in \snk\!\left(G\right)} \alpha_{\mathcal{G}}\!\left(v\right) =
           \sum_{v \in \snk\!\left(G\right)} \alpha\!\left(v\right) \otimes \hat{e}_{0} +
            \sum_{v \in \src\!\left(G\right)} \left(
              \mathsf{P}_{1}\!\left(\xi\!\left(v\right)\right) \cdot \beta\!\left(v\right)
            \right) \otimes \hat{e}_{1}$
        where $\mathcal{G} = \left(G, \op, A \otimes_{S} \BC_{S}^{1}, \xi\right)$.
\Procedure{Forward-Backward}{$G$, $\op$, $A$, $\xi$}
\State \Call{Forward}{$G, \op, A, \mathsf{P}_{0} \circ \xi$}
       to compute $\alpha\!\left(t\right)$ for every $t \in V \cup E$.
\ForAll{$t \in \snk\!\left(G\right)$}
  \State $\beta\!\left(v\right) \gets 1_{A}$
\EndFor
\State $C \gets \snk\!\left(G\right)$
\While{$C \neq \src\!\left(G\right)$}
  \State Find $x \in C$ s.t. $\exists y \in\; \downarrow\!\!C \setminus C.\, y \prec x$
         and $y$ is a maximal element of $\downarrow\!\!C \setminus C$.
  \State $D' \gets \left\{d' \in V \cup E \relmiddle| d' \prec x \right\}$
  \ForAll{$t \in D'$}
    \If{$t \in V$}
      \State $\beta\!\left(t\right) \gets \sum_{e \in E_{G}^{+}\left(v\right)} \beta\!\left(e\right)$
    \Else
      \If{$\op\!\left(\head\!\left(t\right)\right) = \text{``$+$''}$}
        \State $\beta\!\left(t\right) \gets \beta\!\left(\head\!\left(t\right)\right)$
      \Else
        \State
          $\beta\!\left(t\right) \gets
             \beta\!\left(\head\!\left(t\right)\right) \cdot \left(
               \prod_{e' \in E_{G}^{-}\!\left(\head\!\left(t\right)\right) \setminus \left\{t\right\}}
                 \alpha\!\left(e'\right)
             \right)$
      \EndIf
    \EndIf
  \EndFor
  \State $D \gets \left\{d \in V \cup E \relmiddle| \exists d' \in D.\, d' \prec d \right\}$
  \State $C \gets \left(C \cup D'\right) \setminus D$
\EndWhile
\EndProcedure
\end{algorithmic}
\label{alg:forward_backward}
\end{algorithm}

Moreover, Algorithm \ref{alg:forward} is ``properly scheduled.'' That is, for every $d' \in D'$ where $D'$ is constructed in Line 8, the value of the forward variable on every element covered by $d'$ is guaranteed to be computed before the value of the forward variable on $d'$ is computed. This fact is an immediate consequence of Lemma \ref{lemma:covering_antichain_cutset}.

One may imagine computation of values of the forward variable that is not antichain-cutset-to-antichain-cutset-basis. That is, one could always skip storing the values of the forward variable on every arc and store them as the possibly partial sum or product in every internal node. However, such procedure can be also modeled on an antichain cutset to antichain cutset basis by imposing the restriction that each internal node of computation graphs is binary.

Algorithm \ref{alg:forward} assumes the existence of an oracle in Line 7 in order for the algorithms to be antichain-cutset-to-antichain-cutset-basis. The oracle is required to return, for a given antichain cutset $C$, $x \in C$ such that there exists $y \in\; \uparrow\!\!C \setminus C$ and $x \prec y$. However, this assumption is not strong because, if an order among nodes and/or arcs required for computation of values of the forward variable to be properly scheduled is already known (e.g., a topological order among nodes), then it can be also used to construct the oracle without any remarkably additional computation cost.

Algorithm \ref{alg:forward_backward} is the pseudo-code of forward-backward algorithms. Note that the forward variable $\alpha\!\left(t\right)$ in Algorithm \ref{alg:forward_backward} is that of the computation graph $\left(G, \op, A, \mathsf{P}_{0} \circ \xi\right)$, not $\left(G, \op, A \otimes_{S} \BC_{S}^{1}, \xi\right)$. Computation of the backward variable from Line 11 through 19 is a straightforward implementation of Definition \ref{def:beta} except that computation in Line 17 uses Lemma \ref{lemma:zeroth_projection_of_alpha}. Algorithm \ref{alg:forward_backward} is also described on an antichain cutset to antichain cutset basis. Therefore the discussion about the problems of space complexity and scheduling for Algorithm \ref{alg:forward} also applies to Algorithm \ref{alg:forward_backward} in a straightforward way. Although we omit the proof that the update of antichain cutsets in Line 22 is correct, it can be trivially shown by the order dual of Lemma \ref{lemma:covering_antichain_cutset}. Lemma \ref{lemma:zeroth_and_first_projection} and Theorem \ref{thm:backward_invariants} guarantee that the postcondition holds.

It is obvious that we can compute $\sum_{v \in \snk(G)} \alpha_\calG(v)$ where $\calG = (G, \op, A \otimes_S \BC_S^1, \xi)$ by calling either $\mathsc{Forward}(G, \op, A \otimes_S \BC_S^1, \xi)$ or $\mathsc{Forward-Backward}(G, \op, A, \xi)$. In other words, we can compute $\sum_{v \in \snk(G)} \alpha_\calG(v)$ by $\mathsc{Forward}$ in place of $\mathsc{Forward-Backward}$. Therefore, when we already know details of an instance of forward-backward algorithms to compute a formula at hand, by identifying corresponding $G$, $\op$, $A$, and $\xi$, we can systematically transform the instance of forward-backward algorithms to the corresponding instance of forward-only algorithms. For example, by setting $G$, $\op$, $A$, and $\xi$ to be equal to the computation structure underlying the Baum-Welch algorithm for HMMs (cf. Examples \ref{ex:cg_for_sequence_labeling} and \ref{ex:marginalization_and_expectation_problems}), we can immediately derive the forward-only algorithm for the Baum-Welch algorithm on HMMs.

\subsection{Conditions Favoring Forward-backward Algorithms}
\label{subsec:conditions_favoring_forward_backward_algorithms}

In the previous subsection, it was concluded that what is computed by forward-backward algorithms can be also always computed by forward algorithms. This conclusion naturally raises the following question: "Why or when forward-backward algorithms are necessary?" This question is answered in this subsection.

Actually, there is a set of conditions to make forward-backward algorithms much more favorable than forward algorithms. In addition, the conditions are met in many usual machine learning tasks. We present them below.

\begin{MyCondition}[Conditions Favoring Forward-backward Algorithms]
Let $n$ be a non-negative integer, $S$ a cancellative semiring, $\left(G, \op, M, \phi, A, f\right)$ specify the $f$-parametrized computation graph, and $\left(G, \op, N_i, \psi_i, \BC_S^1, g_i\right) \; \left(i = 1, \dots, n\right)$ specify the $g_i$-parametrized computation graph, respectively. Then the set of the following conditions is called the \termref{conditions favoring forward-backward algorithms}:
\begin{itemize}
  \item $\sum_{v \in \snk(G)} \alpha_{\calG_i}(v)$ is required to be computed for $i = 1, \dots, n$, where $\calG_i$ is the $(f \otimes g_i)$-parametrized computation graph specified by the sextuple $(G, \op, M \times N_i, \phi \times \psi_i, A \otimes_S \BC_S^1, f \otimes g_i)$, and
  \item the values of $\mathsf{P}_{0}\!\left(\left(\left(f \otimes g_{i}\right) \circ \left(\phi \times \psi_{i}\right)\right)\!\left(v\right)\right)$ are independent of $i$ for every $v \in \src\!\left(G\right)$.
\end{itemize}
\label{cond:conditions_favoring_forward_backward_algorithms}
\end{MyCondition}

The values of $\alpha$ computed in Algorithms \ref{alg:forward} depend only on $G$, $\op$, $S$, and $\mathsf{P}_{0}\!\left(\xi\!\left(v\right)\right)$ where $v \in \src\!\left(G\right)$. Therefore, when Condition \ref{cond:conditions_favoring_forward_backward_algorithms} is met, we need to execute Algorithm \ref{alg:forward} only once no matter how many computation graphs are required to be computed. Likewise, because the values of $\beta$ computed in Algorithm \ref{alg:forward_backward} depend only on $G$, $\op$, $A$, and $\alpha$ computed by Algorithm \ref{alg:forward}, when Condition \ref{cond:conditions_favoring_forward_backward_algorithms} is met, we need to compute $\beta$ in Algorithm \ref{alg:forward_backward} only once no matter how many computation graphs are required to be computed. Moreover, the values of $\sum_{v \in \snk\!\left(G\right)} \alpha_{\mathcal{G}}\!\left(v\right)$ depend only on the values of $\alpha\!\left(v\right)$ on every $v \in \snk\!\left(G\right)$, the values of $\beta\!\left(v\right)$ on every $v \in \src\!\left(G\right)$, and the values of $\mathsf{P}_{1}\!\left(\xi\!\left(v\right)\right)$ on every $v \in \src\!\left(G\right)$. Therefore, in order to compute the values of $\sum_{v \in \snk\!\left(G\right)} \alpha_{\mathcal{G}_{i}}\!\left(v\right)$ of $n$ computation graphs $\mathcal{G}_{i}$ ($i = 1, \dots, n$) satisfying Condition \ref{cond:conditions_favoring_forward_backward_algorithms}, we need exactly one forward pass and exactly one backward pass on the computation graph, where computation cost is irrelevant of $n$, and evaluation of $\mathsf{P}_{1}\!\left(\left(\left(f \otimes g_{i}\right) \circ \left(\phi \times \psi_{i}\right)\right)\!\left(v\right)\right)$ for $i = 1, \dots, n$ and only every source node $v \in \src\!\left(G\right)$.

As a result, in the case where $n$ computation graphs ($n > 1$) satisfying Condition \ref{cond:conditions_favoring_forward_backward_algorithms} are required to be computed, forward-backward algorithms are much more efficient than forward algorithms. When such computation is done only by forward algorithms, we need to execute a forward pass for each computation graph and totally $n$ forward passes, so computation cost is proportional to $n \!\left(\left|\src\!\left(G\right)\right| + \left|E\right|\right)$. On the other hand, computation cost by forward-backward algorithms is proportional to $2\!\left(\left|\src\!\left(G\right)\right| + \left|E\right|\right) + n\!\left|\src\!\left(G\right)\right|$. Because $\left|\src\!\left(G\right)\right|$ is usually much smaller than $\left|E\right|$, we can conclude that forward-backward algorithms are much more favorable than forward algorithms in such case.

Next, we show that Condition \ref{cond:conditions_favoring_forward_backward_algorithms} is met in many usual machine learning tasks by some examples.

\begin{MyExample}[Baum-Welch Algorithm]
Recall Example \ref{ex:marginalization_and_expectation_problems}. In order to perform the training of the parameters of HMMs by the Baum-Welch algorithm, we need to calculate \eqref{eq:unnormalized_first_order_moment} multiple times with different definitions of $\psi$. They take on the following form
\begin{equation}
\sum_{\veci \in \bbNzero^n} c_{t, \veci}\!\left((\phi(s_1))^{i_1} \cdots (\phi(s_n))^{i_n}\right)(i_1 \psi_j(s_1) + \cdots + i_n \psi_j(s_n)) \; ,
\label{eq:multiple_first_order_moments}
\end{equation}
where $j \in \left\{1, \dots, m\right\}$. $\psi_j$ is defined by either of following: an indicator function that is equal to $1$ if and only if the argument source node corresponds to the residence at each state (in this case, \eqref{eq:multiple_first_order_moments} becomes equal to the expectation of the residence of the state), an indicator function that is equal to $1$ if and only if the argument source node corresponds to each kind of transitions (in this case, \eqref{eq:multiple_first_order_moments} becomes equal to the expectation of the kind of the transitions), or an indicator function that is equal to $1$ if and only if the argument source node corresponds to each kind of observation emissions from a state to an observation (in this case, \eqref{eq:multiple_first_order_moments} becomes equal to the expectation of the kind of the observation emissions). As explained in Example \ref{ex:marginalization_and_expectation_problems}, the values of the form \eqref{eq:multiple_first_order_moments} is obtained by the parametrized computation graph $\calG_j = (G, \op, \bbR \otimes_\bbR \BC_\bbR^1, (\id_\bbR \otimes \mathcal{P}_\bbR^1) \circ (\phi \times \psi_j))$ specified by the sextuple $(G, \op, (\bbR, \cdot, 1) \times (\bbR, +, 0), \phi \times \psi_j, \bbR \otimes_{\bbR} \BC_{\bbR}^{1}, \id_{\bbR} \otimes \mathcal{P}_{\bbR}^{1})$ and the application of the extension by linearity of the mapping $L \colon \bbR \otimes_\bbR \bbR^{2} \to \bbR$ satisfying $L(1 \otimes \hat{e}_0) = 0$ and $L(1 \otimes \hat{e}_1) = 1$. Because $\zerothproj(((\id_\bbR \otimes \mathcal{P}_\bbR^1) \circ (\phi \times \psi_j))\!\left(v\right)) = \phi(v)$ for every $j \in \left\{1, \dots, m\right\}$ and $v \in \src(G)$, the set of $\calG_j$ satisfies Condition \ref{cond:conditions_favoring_forward_backward_algorithms}. Therefore, forward-backward algorithms are much more favorable than forward algorithms in this case.
\label{ex:conditions_in_baum_welch}
\end{MyExample}

With a slight modification to Example \ref{ex:conditions_in_baum_welch}, we also obtain the instances of forward-backward algorithms to compute feature expectations for CRFs on trellises and factor graphs and the log-linear model on various data structures. See the paragraph following Example \ref{ex:marginalization_and_expectation_problems}.

Note that forward-backward algorithms are favored for the cases cited in Example \ref{ex:conditions_in_baum_welch} only from the point of view of time complexity. If space rather than time complexity is the main constraint, forward-only algorithms are acceptable even when Condition \ref{cond:conditions_favoring_forward_backward_algorithms} is met because space complexity of forward-backward algorithms is proportional to the size of $G$ while that of forward-only algorithms is proportional to the maximum size of antichain cutsets during execution of $\mathsc{Forward}$ (see the discussion of space complexity of Algorithm \ref{alg:forward} in the previous subsection). In the context of the forward-only computation of the Baum-Welch algorithm, the analysis described above is completely consistent with the investigation in \citet{khreich2010memory}.

\begin{MyExample}[Reverse Mode of AD (Back Propagation)]
Let $m$ be a positive integer, $\mathcal{F}_m$ the set of all differentiable functions having $m$ independent variables from an open subset of $D$ of $\bbR^{m}$ to $\bbR$, $1_{\mathcal{F}_m} \in \mathcal{F}_m$ the constant function whose value is always $1$, $\mathcal{F}_m = \left(\mathcal{F}_m, \cdot, 1_{\mathcal{F}_m}\right)$ the commutative monoid equipped with the pointwise multiplication $\cdot$ on $\mathcal{F}_m$, $\psi$ a mapping from $\src\!\left(G\right)$ to $\mathcal{F}_m$, $\mathcal{G} = \left(G, \op, \mathcal{F}_m, \psi\right)$ a computation graph, and $\src\!\left(G\right) = \left\{s_{1}, \dots, s_{n}\right\}$. Let us define $\chi \colon \src\!\left(G\right) \to \bbNzero\left[x_{0}, \dots, x_{n}\right]$ mapping $s_{i}$ to $x_{i}$ for every $s_{i} \in \src\!\left(G\right)$, and obtain the following form of the free forward variable of $\left(G, \op\right)$ w.r.t. $\chi$
\begin{equation*}
\alpha_{\left(G, \op, \chi\right)}\!\left(t\right)
  = \sum_{\veci \in \bbNzero^{n}} c_{t, \veci} {x_{1}}^{i_{1}} \cdots {x_{n}}^{i_{n}}
\end{equation*}
for every node and arc $t \in V \cup E$. Then
\begin{equation*}
\alpha_{\mathcal{G}}\!\left(t\right)
  = \sum_{\veci \in \bbNzero^n}
      c_{t, \veci} \left(\psi\!\left(\vecx; s_1\right)\right)^{i_1} \cdots
                   \left(\psi\!\left(\vecx; s_n\right)\right)^{i_n} \; .
\end{equation*}
Further let us calculate the value of the gradient of $\alpha_{\mathcal{G}}$ evaluated at a point $\vecx = \vecx_0$, which is of the following form
\begin{equation}
\begin{aligned}
&\left.\frac{\partial}{\partial \vecx} \left(
         \sum_{\veci \in \bbNzero^{n}} c_{t, \veci}
           \left(\psi\!\left(\vecx; s_{1}\right)\right)^{i_{1}} \cdots
           \left(\psi\!\left(\vecx; s_{n}\right)\right)^{i_{n}}
 \right)\right|_{\vecx = \vecx_{0}} \\
&\qquad=
 \left(\left.\frac{\partial}{\partial x_{k}} \left(
               \sum_{\veci \in \bbNzero^{n}} c_{t, \veci}
                 \left(\psi\!\left(\vecx; s_{1}\right)\right)^{i_{1}} \cdots
                 \left(\psi\!\left(\vecx; s_{n}\right)\right)^{i_{n}}
 \right)\right|_{\vecx = \vecx_{0}}\right)_{k = 1, \dots, m} \; ,
\end{aligned}
\label{eq:differentiation_of_multivariate_function}
\end{equation}
where $\vecx = \left(x_1, \dots, x_m\right)$. The value of the form \eqref{eq:differentiation_of_multivariate_function} is obtained by the computation graphs $\mathcal{G}'_k = \left(G, \op, \bbR \otimes_\bbR \BC_\bbR^1, \left(\id_\bbR \otimes \Delta_{k, \vecx_0}^1\right) \circ \left(\phi \times \psi\right)\right) \; \left(k = 1, \dots, m\right)$ that are specified by the sextuple $\left(G, \op, (\bbR, \cdot, 1) \times \mathcal{F}_m, \phi \times \psi, \bbR \otimes_\bbR \BC_\bbR^1, \id_\bbR \otimes \Delta_{k, \vecx_0}^1\right)$, respectively, followed by the $\bbR$-homomorphism $L \colon \bbR \otimes_\bbR \BC_\bbR^1 \to \bbR$ satisfying $L\!\left(p \otimes \left(q, r\right)\right) = pr$, where $\phi \colon \src\!\left(G\right) \to \bbR$ is defined by $\phi\!\left(v\right) = 1$ for every $v \in \src\!\left(G\right)$. Because $\mathsf{P}_0\!\left(\left(\left(\id_\bbR \otimes \Delta_{k, \vecx_0}^1\right) \circ \left(\phi \times \psi\right)\right)\!\left(v\right))\right) = \psi\!\left(\vecx_0; v\right)$ is independent of $k$, the set of the computation graphs $\left\{\mathcal{G}_1, \dots, \mathcal{G}_m\right\}$ meets Condition \ref{cond:conditions_favoring_forward_backward_algorithms}. Therefore, forward-backward algorithms are much more favorable than forward algorithms in this case.
\label{ex:reverse_mode_of_automatic_differentiation}
\end{MyExample}

Note that the forward algorithm for the computation graphs $\mathcal{G}'_k \; \left(k = 1, \dots, m\right)$ in Example \ref{ex:reverse_mode_of_automatic_differentiation} is equivalent to the forward mode of AD for the computation graph $\mathcal{G}$, and the corresponding forward-backward algorithm is equivalent to the reverse mode of AD. The analysis of trade-off between time and space complexity about forward and forward-backward algorithms described in this paper is completely consistent with the one between the forward and reverse modes of AD for the case of Example \ref{ex:reverse_mode_of_automatic_differentiation} \citep[cf.][Sections I.3 and I.4]{griewank2008evaluating}.

\subsection{Checkpoints for Forward-backward Algorithms}
\label{subsec:checkpoints_for_forward_backward_algorithms}

This subsection provides a brief note on technique to adjust trade-off between time and space complexity of forward-backward algorithms. The technique is based on so-called \termref{checkpoints}. In executing an instance of forward-backward algorithms, values of the forward variable of a computation graph are not necessarily stored in every element of $V \cup E$ during the forward pass. Instead, they are stored in the elements of some subsets of $V \cup E$, which are called checkpoints, and then necessary values of the forward variable in the backward pass are recomputed from stored values close to them.

One of obviously good candidates for checkpoints is $V$. Forward variables are defined on $V \cup E$ for technical reasons, but the values of a forward variable on $E$ are not necessarily stored in the forward pass of forward-backward algorithms. Instead, the value of a forward variable on every arc can be recomputed from the one stored in the tail of the arc. By doing so, space complexity required for forward-backward algorithms can be made proportional to $\left|V\right|$ while time complexity is kept proportional to $\left|V\right| + \left|E\right|$.

Another good candidate for checkpoints is antichain cutsets of the poset induced by a computation graph because if values of the forward variable of the computation graph are stored in every element of an antichain cutset $C$ then the values of the forward variable on every element of $\uparrow\!\!C \setminus C$ can be recomputed from the values stored in $C$. By adjusting granularity of antichain cutsets where values of the forward variable are stored, we can adjust trade-off between time and space complexity in forward-backward algorithms to a certain degree. In particular, by repeating the construction of a covering antichain cutset described in Lemma \ref{lemma:covering_antichain_cutset} from $\src\!\left(G\right)$, we obtain a maximal chain of the lattice of antichain cutsets from $\src\!\left(G\right)$ through $\snk\!\left(G\right)$. By considering the succession of these antichain cutsets as an analogue of the sequence of the time slices in trellises for sequence labeling, discussion on variants of technique based on checkpoints on sequence labeling can be also applied to any computation graph.

\section{Conclusion}
\label{sec:conclusion}

In this paper, we propose an algebraic formalization of forward and forward-backward algorithms. The formalization consists of (1) a unified abstraction of any computation consisting of a finite number of additions and/or multiplications, which is enough to support the development of the unified formalization of the algorithms, (2) the elucidation of algebraic structures underlying complicated forward algorithms, (3) a systematic framework to construct complicated and difficult-to-design forward algorithms from simple and easy-to-design forward algorithms, and (4) an algebraic characterization of forward-backward algorithms.

Although we present only two pseudo-codes (i.e., Algorithms \ref{alg:forward} and \ref{alg:forward_backward}) in this paper, they subsume a wide range of existing algorithms due to their versatility. To our knowledge, the formalization described in this paper subsumes (a part of) the following papers: \citet[Section III]{forney1973viterbi}, \citet[Sections III.A and III.B]{rabiner1989tutorial}, \citet[Section 2]{lari1990estimation}, \citet[Chapter 3]{tan1993adaptive}, \citet[Section IV.B]{turin1998unidirectional}, \citet[Sections 2 through 4]{goodman1999semiring} (provided that all derivations are acyclic, and the derivation forest is finite), \citet[Section III]{aji2000generalized}, \citet[Section 4]{kschischang2001factor}, \citet[Section 4]{klein2004parsing}, \citet[Section titled ``Methods and results'']{miklos2005linear}, \citet{mann2007efficient}, \citet[Subsection titled ``Linear memory Baum-Welch using a backward sweep with scaling'']{churbanov2008implementing}, \citet{ishihata2008propositionalizing}, \citet[Section I.3]{griewank2008evaluating} (provided that involved operations are limited to the addition or multiplication of real numbers), \citet[Section 5.1]{huang2008advanced}, \citet[Sections 3 through 5]{li2009first}, \citet[Section 2]{azuma2009generalization}, \citet[Section 4]{kimmig2011algebraic}, \citet[Section titled ``Hessian-vector Products of CRFs'']{tsuboi2011fast}, \citet{ilic2012computation}.

The formalization presented in this paper not only subsumes a wide range of existing algorithms but also extends them to an infinite number of their variants by using Framework \ref{framework} described in Section \ref{subsec:tensor_product_of_semialgebras_for_forward_algorithm}. The extension can be done in a systematic way. One only needs to identify the underlying abstract computation structure $\left(G, \op\right)$, and the computation problem at hand as an instance of formula \eqref{eq:complicated_forward_algorithms}.

The formalization presented in this paper also allows another direction of extension. Some of the algorithms proposed in the above-mentioned papers are of the forward-backward type. However, we build a linking bridge between forward and forward-backward algorithms. Consequently, we also obtain the forward-only version of many existing algorithms of the forward-backward type. This increases choices of algorithms according to trade-off between time and space requirements.

In addition, the unified formalization accelerates the ``synergy'' among discussions and techniques in a variety of existing algorithms. Based on the unified formalization, it is easy to transfer a discussion or technique in the context of a specific algorithm to the context of other algorithms.

For example, consider the discussion in \citet{eisner2016inside}. He has pointed out that the inside-outside algorithm can be derived by back propagation. The derivation is automatic as the reverse mode of AD can be automatically derived from computation on the underlying computation graph. However, his discussion is limited to the context of parsing. Moreover, the probability distribution is assumed to be log-linear because his discussion is based on the fact that the partition function of a log-linear distribution is also its moment generating function.

In contrast, our formalization appears to show the possibility that the discussion in \citet{eisner2016inside} can be generalized to much wider contexts. That is, Algorithm \ref{alg:forward_backward} can be automatically derived from Algorithm \ref{alg:forward}. This idea naturally arises from the fact that the relationship between the forward and reverse modes of AD is an instance of the relationship between Algorithms \ref{alg:forward} and \ref{alg:forward_backward}. For such transformation, we can use a long history and accumulation of discussions and experiences in the context of AD~\citep[e.g., see][Chapter 6]{griewank2008evaluating}.




\appendix
\section{Proof of Lemma \ref{lemma:bcs}}

(This proof uses some notions defined after Lemma \ref{lemma:bcs} because using these notions significantly simplifies the proof. Of course, these forward references never lead to circular reasoning.)
\begin{proof}
Let $S = \left(S, +, \cdot, 0_S, 1_S\right)$. It is obvious that $\left(S^{n + 1}, +, 0_{\BC_S^n}\right)$ is a commutative monoid from the definition of the addition \eqref{eq:bc_add}, and is also an $S$-semimodule by being equipped with scalar multiplication $\sigma\!\left(\left(s_i\right)_{i = 0, \dots, n}\right) = \left(\sigma \cdot s_i\right)_{i = 0, \dots, n}$ for every $\sigma \in S$ and $(s_i)_{i = 0, \dots, n} \in S^{n + 1}$. Let $\hat{e}_i \in S^{n + 1} \; (i = 0, \dots, n)$ be defined by $\hat{e}_0 = \left(1_S, 0_S, 0_S, \dots\right), \hat{e}_1 = \left(0_S, 1_S, 0_S, \dots\right)$, and so on. Clearly, $U = \left\{\hat{e}_0, \dots, \hat{e}_n\right\}$ is a basis of the $S$-semimodule $\left(S^{n + 1}, +, 0_{\BC_S^n}\right)$. For every $a = \left(a_i\right)_{i = 0, \dots, n}, b = \left(b_i\right)_{i = 0, \dots, n}, c = \left(c_i\right)_{i = 0, \dots, n} \in S^{n + 1}$, one obtains
\begin{align*}
&a \diamond \left(b + c\right) = \left(\left(a_i\right)_{i = 0, \dots, n}\right) \diamond
  \left(\left(b_i\right)_{i = 0, \dots, n} + \left(c_i\right)_{i = 0, \dots, n}\right) \\
&\qquad = \left(\left(a_i\right)_{i = 0, \dots, n}\right) \diamond
          \left(\left(b_i + c_i\right)_{i = 0, \dots, n}\right) \\
&\qquad = \left(
            \sum_{j \in \left\{k \in \bbNzero \relmiddle| k \leq i\right\}}
              \binom{i}{j} \left(a_{j} \cdot \left(b_{i - j} + c_{i - j}\right)\right)
          \right)_{\!\!\!\!i = 0, \dots, n} \\
&\qquad = \left(
            \left(
              \sum_{j \in \left\{k \in \bbNzero \relmiddle| k \leq i\right\}}
                \binom{i}{j} \left(a_{j} \cdot b_{i - j}\right)
            \right) +
            \left(
              \sum_{j \in \left\{k \in \bbNzero \relmiddle| k \leq i\right\}}
              \binom{i}{j} \left(a_{j} \cdot c_{i - j}\right)
            \right)
          \right)_{\!\!\!\!i = 0 \dots, n} \\
&\qquad = \left(
            \sum_{j \in \left\{k \in \bbNzero \relmiddle| k \leq i\right\}}
              \binom{i}{j} \left(a_{j} \cdot b_{i - j}\right)
          \right)_{\!\!\!\!i = 0, \dots, n} +
          \left(
            \sum_{j \in \left\{k \in \bbNzero \relmiddle| k \leq i\right\}}
              \binom{i}{j} \left(a_{j} \cdot c_{i - j}\right)
          \right)_{\!\!\!\!i = 0, \dots, n} \\
&\qquad = \left(\left(a_i\right)_{i = 0, \dots, n}\right) \diamond
          \left(\left(b_i\right)_{i = 0, \dots, n}\right) +
          \left(\left(a_i\right)_{i = 0, \dots, n}\right) \diamond
          \left(\left(c_i\right)_{i = 0, \dots, n}\right) \\
&\qquad = \left(a \diamond b\right) + \left(a \diamond c\right) \; .
\end{align*}
Likewise, one also easily obtains $\left(a + b\right) \diamond c = \left(a \diamond c\right) + \left(b \diamond c\right)$, so $\diamond$ distributes over $+$. Moreover, from the definition of the multiplication \eqref{eq:bc_mult},
\begin{equation*}
(\sigma \hat{e}_i) \diamond (\tau \hat{e}_j) = (\sigma \cdot \tau) (\hat{e}_i \diamond \hat{e}_j) =
\begin{cases}
  \binom{i + j}{i} \!\left(\left(\sigma \cdot \tau\right) \!\hat{e}_{i + j}\right)
                & \text{if $i + j \leq n$,} \\
  0_{\BC_{S}^n} & \text{otherwise,}
\end{cases}
\end{equation*}
for every $\sigma, \tau \in S$ and $\hat{e}_{i}, \hat{e}_{j} \in U$. Therefore $\left(S^{n + 1}, +, \diamond, 0_{\BC_S^n}\right)$ is an $S$-semialgebra with a semialgebra basis $U$. Because
\begin{equation*}
\hat{e}_{i} \diamond \hat{e}_{j} = \hat{e}_{j} \diamond \hat{e}_{i} = \begin{cases}
\binom{i + j}{i} \hat{e}_{i + j} = \binom{i + j}{j} \hat{e}_{i + j} & \text{if $i + j \leq n$,} \\
0_{\BC_{S}^{n}}                                                     & \text{otherwise,}
\end{cases}
\end{equation*}
for every $\hat{e}_i, \hat{e}_j \in U$, the $S$-semialgebra $(S^{n + 1}, +, \diamond, 0_{\BC_S^n})$ is commutative by \citet[Theorem V.2.4]{hebisch1998algebraic}. The structure constants of the $S$-semialgebra $\bigl(S^{n + 1}, +, \diamond,\allowbreak 0_{\BC_S^n}\bigr)$ w.r.t. $U$ is
\begin{equation*}
\sigma_{\hat{e}_{i}, \hat{e}_{j}}^{\hat{e}_{k}} = \begin{cases}
\binom{i + j}{i} 1_S & \text{if $i + j = k$,} \\
0_S                  & \text{otherwise,}
\end{cases}
\end{equation*}
for every $\hat{e}_{i}, \hat{e}_{j}, \hat{e}_{k} \in U$. Now one obtains
\begin{equation*}
\sum_{\hat{e}_{l} \in U}
  \sigma_{\hat{e}_{i}, \hat{e}_{j}}^{\hat{e}_{l}}
    \cdot \sigma_{\hat{e}_{l}, \hat{e}_{k}}^{\hat{e}_{m}} =
\begin{cases}
  \sigma_{\hat{e}_{i}, \hat{e}_{j}}^{\hat{e}_{i + j}}
    \cdot \sigma_{\hat{e}_{m - k}, \hat{e}_{k}}^{\hat{e}_{m}} =
  \left(\binom{i + j}{i} 1_S\right) \cdot
    \left(\binom{m}{m - k} 1_S\right) & \text{if $i + j + k = m$,} \\
  0_S                                   & \text{otherwise,}
\end{cases}
\end{equation*}
\begin{equation*}
\sum_{\hat{e}_{l} \in U}
  \sigma_{\hat{e}_{j}, \hat{e}_{k}}^{\hat{e}_{l}}
    \cdot \sigma_{\hat{e}_{i}, \hat{e}_{l}}^{\hat{e}_{m}} =
\begin{cases}
  \sigma_{\hat{e}_{j}, \hat{e}_{k}}^{\hat{e}_{j + k}}
    \cdot \sigma_{\hat{e}_{i}, \hat{e}_{m - i}}^{\hat{e}_{m}} =
  \left(\binom{j + k}{j} 1_S\right) \cdot
    \left(\binom{m}{m - i} 1_S\right) & \text{if $i + j + k = m$,} \\
  0_S                                   & \text{otherwise,}
\end{cases}
\end{equation*}
and if $i + j + k = m$, using the identity of binomial coefficients $\binom{n}{h}\binom{n - h}{k} = \binom{n}{k}\binom{n - k}{h}$,
\begin{equation*}
\begin{aligned}
&\left(\binom{i + j}{i} 1_S\right) \cdot \left(\binom{m}{m - k} 1_S\right)
  = \left(\binom{i + j}{i} \binom{m}{m - k}\right) 1_S \\
& \qquad = \left(\binom{i + j}{i} \binom{i + j + k}{i + j}\right) 1_S
  = \left(\binom{i + j + k - k}{i} \binom{i + j + k}{k}\right) 1_S \\
& \qquad = \left(\binom{i + j + k - i}{k} \binom{i + j + k}{i}\right) 1_S
  = \left(\binom{j + k}{k} \binom{i + j + k}{j + k}\right) 1_S \\
& \qquad = \left(\binom{j + k}{j} \binom{m}{m - i}\right) 1_S
  = \left(\binom{j + k}{j} 1_S\right) \cdot \left(\binom{m}{m - i} 1_S\right) \; .
\end{aligned}
\end{equation*}
Therefore, $\sum_{\hat{e}_l \in U} \sigma_{\hat{e}_i, \hat{e}_j}^{\hat{e}_l} \cdot \sigma_{\hat{e}_l, \hat{e}_k}^{\hat{e}_m} = \sum_{\hat{e}_l \in U} \sigma_{\hat{e}_j, \hat{e}_k}^{\hat{e}_l} \cdot \sigma_{\hat{e}_i, \hat{e}_l}^{\hat{e}_m}$ for every $\hat{e}_i, \hat{e}_j, \hat{e}_k, \hat{e}_m \in U$, and thus the $S$-semialgebra $\left(S^{n + 1}, +, \diamond, 0_{\BC_S^n}\right)$ is associative by \citet[Theorem V.2.4]{hebisch1998algebraic}. Finally, $1_{\BC_S^n} \diamond \hat{e}_i = \hat{e}_i \diamond 1_{\BC_S^n} = \hat{e}_i$ for every $\hat{e}_i \in U$, and the $S$-semialgebra $\left(S^{n + 1}, +, \diamond, 0_{\BC_S^n}\right)$ is unital by \citet[Exercise V.2.2]{hebisch1998algebraic}. Thus $\bigl(S^{n + 1}, +,\allowbreak \diamond, 0_{\BC_S^n}, 1_{\BC_S^n}\bigr)$ is a commutative unital associative $S$-semialgebra, and $\BC_S^n = \bigl(S^{n + 1}, +, \diamond,\allowbreak 0_{\BC_S^n}, 1_{\BC_S^n}\bigr)$ is itself a commutative semiring.
\end{proof}

\section{Proof of Lemma \ref{lemma:basis_of_tensor_product_of_semimodules}}

\begin{proof}
Let $S = \left(S, +, \cdot, 0_S, 1_S\right)$. From Lemma \ref{lemma:existence_of_tensor_product_of_semimodules}, every element $t \in M \otimes_{S} N$ can be written as a finite sum $t = \sum_{i} \rho_{i} \!\left(m_{i} \otimes n_{i}\right)$ for some $m_{i} \in M$, $n_{i} \in N$, and $\rho_{i} \in S$. Moreover, $m_{i}$ and $n_{i}$ can be written as unique linear combinations of the bases $U$ and $V$, respectively, so we have $m_{i} = \sum_{u \in U} \sigma_{i, u}u$ and $n_{i} = \sum_{v \in V} \tau_{i, v}v$ for some $\sigma_{i, u}, \tau_{i, v} \in S$. Thus, by using bilinearity of the tensor product, we obtain $t = \sum_{i} \rho_{i} \!\left(\left(\sum_{u \in U} \sigma_{i, u}u\right) \otimes \left(\sum_{v \in V} \tau_{i, v}v\right)\right) = \sum_{\left(u, v\right) \in U \times V} \left(\sum_{i}\left(\rho_{i} \cdot \sigma_{i, u} \cdot \tau_{i, v}\right)\right) \left(u \otimes v\right)$. Therefore, $\left\{u \otimes v\right\}_{u \in U,\, v \in V}$ generates $M \otimes_{S} N$ by linear combinations.

Next, every element $m \in M$ and $n \in N$ can be written in a unique manner as the linear combination of the elements of the bases $U$ and $V$, respectively, so we have $m = \sum_{u \in U} \sigma_{u} u$ and $n = \sum_{v \in V} \tau_{v} v$ for some $\sigma_{u}, \tau_{v} \in S$. Let $u'$ and $v'$ be elements of $U$ and $V$, respectively. Now consider the function $B_{u', v'} \colon M \times N \to S$ defined by setting $B_{u', v'}\left(m, n\right) = B_{u', v'}\left(\sum_{u \in U} \sigma_{u}u, \sum_{v \in V} \tau_{v}v\right) = \sigma_{u'} \cdot \tau_{v'}$. Clearly $B_{u', v'}$ is bilinear. Thus, by Definition \ref{def:tensor_product_of_semimodules} and Lemma \ref{lemma:existence_of_tensor_product_of_semimodules}, there exists an $S$-homomorphism $L_{u', v'} \colon M \otimes_{S} N \to S$ satisfying $L_{u', v'}\left(m \otimes n\right) = B_{u', v'}\!\left(m, n\right) = \left(\sigma_{u'} \cdot \tau_{v'}\right)$. In particular, for every $u \in U$ and $v \in V$,
\begin{equation*}
L_{u', v'}\left(u \otimes v\right) = \begin{cases}
1_S & \text{$u = u'$ and $v = v'$,} \\
0_S & \text{otherwise.}
\end{cases}
\end{equation*}
Assume that
\begin{equation}
\sum_{\left(u, v\right) \in U \times V} \sigma_{u, v}\left(u \otimes v\right) =
\sum_{\left(u, v\right) \in U \times V} \tau_{u, v}\left(u \otimes v\right)
\label{eq:linear_independence_of_basis_of_tensor_product}
\end{equation}
for $\sigma_{u, v}, \tau_{u, v} \in S$ but only finitely many of the coefficients $\sigma_{u, v}$ and $\tau_{u, v}$ are different from $0_S$. Applying $L_{u', v'}$ to both sides of Eq. \eqref{eq:linear_independence_of_basis_of_tensor_product} tells us $\sigma_{u', v'} = \tau_{u', v'}$. Since $u'$ and $v'$ are arbitrarily chosen from $U$ and $V$, respectively, we obtain $\sigma_{u, v} = \tau_{u, v}$ for every $u \in U$ and $v \in V$. It follows that $\left\{u \otimes v\right\}_{u \in U,\, v \in V}$ is linearly independent. Therefore $\left\{u \otimes v\right\}_{u \in U,\, v \in V}$ is a basis of $M \otimes_{S} N$.
\end{proof}

\section{Proof of Lemma \ref{lemma:tensor_product_of_semialgebras}}

\begin{proof}
From Definition \ref{def:tensor_product_of_semimodules} and Lemma \ref{lemma:existence_of_tensor_product_of_semimodules}, $A \otimes_{S} A' = \left(A \otimes_{S} A', +, 0_{A} \otimes 0_{A'}\right)$ is an $S$-semimodule. Moreover, by Lemma \ref{lemma:basis_of_tensor_product_of_semimodules}, $W = \left\{u \otimes v\right\}_{u \in U,\, v \in V}$ is a basis of $A \otimes_{S} A'$.

For every three elements of the tensor product of $S$-semimodules $A \otimes_{S} A'$ written as linear combinations of the basis $W$, say, $t = \sum_{\left(u, v\right) \in U \times V} \rho_{u, v} \left(u \otimes v\right),$ $t' = \sum_{\left(u, v\right) \in U \times V} \rho'_{u, v} \left(u \otimes v\right),$ and $t'' = \sum_{\left(u, v\right) \in U \times V} \rho''_{u, v} \left(u \otimes v\right),$ where $\rho_{u, v}, \rho'_{u, v}, \rho''_{u, v} \in S$, using the definition of the operation \eqref{eq:multiplication_of_tensor_product_of_semialgebras}, we have
\begin{align*}
&t \cdot \left(t' + t''\right) \\
&\qquad= \left(\sum_{\left(u, v\right) \in U \times V} \rho_{u, v} \left(u \otimes v\right)\right) \cdot
         \left(\sum_{\left(u, v\right) \in U \times V}
                 \left(\rho'_{u, v} + \rho''_{u, v}\right) \left(u \otimes v\right)\right) \\
&\qquad=\sum_{\left(u, v\right) \in U \times V}
          \sum_{\left(u', v'\right) \in U \times V}
            \sum_{\left(u'', v''\right) \in U \times V}
              \left(
                \rho_{u, v} \cdot \left(\rho'_{u', v'} + \rho''_{u', v'}\right) \cdot
                \sigma_{u, u'}^{u''} \cdot \tau_{v, v'}^{v''}
              \right) \left(u'' \otimes v''\right) \\
&\qquad=\sum_{\left(u, v\right) \in U \times V}
          \sum_{\left(u', v'\right) \in U \times V}
            \sum_{\left(u'', v''\right) \in U \times V}
              \left(\rho_{u, v} \cdot \rho'_{u', v'} \cdot
                    \sigma_{u, u'}^{u''} \cdot \tau_{v, v'}^{v''}\right) \left(u'' \otimes v''\right) \\
&\qquad\phantom{=}\qquad+
        \sum_{\left(u, v\right) \in U \times V}
          \sum_{\left(u', v'\right) \in U \times V}
          \sum_{\left(u'', v''\right) \in U \times V}
            \left(\rho_{u, v} \cdot \rho''_{u', v'} \cdot
                  \sigma_{u, u'}^{u''} \cdot \tau_{v, v'}^{v''}\right) \left(u'' \otimes v''\right) \\
&\qquad=\left(\sum_{\left(u, v\right) \in U \times V} \rho_{u, v} \left(u \otimes v\right)\right) \cdot
        \left(\sum_{\left(u, v\right) \in U \times V} \rho'_{u, v} \left(u \otimes v\right)\right) \\
&\qquad\phantom{=}\qquad+
        \left(\sum_{\left(u, v\right) \in U \times V} \rho_{u, v} \left(u \otimes v\right)\right) \cdot
        \left(\sum_{\left(u, v\right) \in U \times V} \rho''_{u, v} \left(u \otimes v\right)\right) \\
&\qquad=t \cdot t' + t \cdot t'' \; .
\end{align*}
Likewise, one easily obtains $\left(t + t'\right) \cdot t'' = t \cdot t'' + t' \cdot t''$, so the operation is distributive.

For every $\rho, \rho' \in S$ and every two elements of the basis $u \otimes v, u' \otimes v' \in W$, we have, using \eqref{eq:multiplication_of_tensor_product_of_semialgebras},
\begin{equation}
\begin{aligned}
\left(\rho\!\left(u \otimes v\right)\right) \cdot \left(\rho' \!\left(u' \otimes v'\right)\right) &=
\sum_{\left(u'', v''\right) \in U \times V}
  \left(\rho \cdot \rho' \cdot \sigma_{u, u'}^{u''} \cdot \tau_{v, v'}^{v''}\right)
  \left(u'' \otimes v''\right) \\
&= \sum_{\left(u'', v''\right) \in U \times V}
     \left(\rho \cdot \rho'\right)
     \left(\left(\sigma_{u, u'}^{u''} \cdot \tau_{v, v'}^{v''}\right)
           \left(u'' \otimes v''\right)\right) \\
&= \left(\rho \cdot \rho'\right)
   \sum_{\left(u'', v''\right) \in U \times V}
     \left(\left(\sigma_{u, u'}^{u''} \cdot \tau_{v, v'}^{v''}\right)
           \left(u'' \otimes v''\right)\right) \\
&= \left(\rho \cdot \rho'\right)
   \left(\left(u \otimes v\right) \cdot \left(u' \otimes v'\right)\right) \; .
\end{aligned}
\label{eq:proof_of_semialgebra_basis_of_tensor_product}
\end{equation}
Therefore $W$ is a semialgebra basis of $A \otimes_{S} A'$, and $A \otimes_{S} A' = \left(A \otimes_{S} A', +, \cdot, 0_{A} \otimes 0_{A'}\right)$ is an $S$-semialgebra.

For every $a, b \in A$ and $a', b' \in A'$, we have
\begin{equation}
\left(a \otimes a'\right) \cdot \left(b \otimes b'\right) =
\left(a \cdot a'\right) \otimes \left(b \cdot b'\right) \; ,
\label{eq:multiplication_between_elementary_tensors_in_proof}
\end{equation}
because, by using unique linear combinations of the basis $W$, $a = \sum_{u \in U} \alpha_{u} u$, $b = \sum_{u \in U} \beta_{u} u$, $a' = \sum_{v \in V} \alpha'_{v} v$, and $b' = \sum_{v \in V} \beta'_{v} v$, bilinearity of tensor product, \eqref{eq:multiplication_of_tensor_product_of_semialgebras}, \eqref{eq:proof_of_semialgebra_basis_of_tensor_product}, the definition of structure constants, and the fact that $U$ and $V$ are semialgebra bases of $A$ and $A'$, respectively, we obtain
\begin{align*}
&\left(a \otimes a'\right) \cdot \left(b \otimes b'\right) \\
&\qquad=\left(\left(\sum_{u \in U} \alpha_{u} u\right) \otimes
\left(\sum_{v \in V} \alpha'_{v} v\right)\right) \cdot
\left(\left(\sum_{u \in U} \beta_{u} u\right) \otimes
\left(\sum_{v \in V} \beta'_{v} v\right)\right) \\
&\qquad=\left(\sum_{\left(u, v\right) \in U \times V} \left(\alpha_{u} \cdot \alpha'_{v}\right)
\left(u \otimes v\right)\right) \cdot
\left(\sum_{\left(u, v\right) \in U \times V} \left(\beta_{u} \cdot \beta'_{v}\right)
\left(u \otimes v\right)\right) \\
&\qquad=\sum_{\left(u, v\right) \in U \times V}
\sum_{\left(u', v'\right) \in U \times V}
\sum_{\left(u'', v''\right) \in U \times V}
\left(\alpha_{u} \cdot \alpha'_{v} \cdot \beta_{u'} \cdot \beta'_{v'} \cdot
\sigma_{u, u'}^{u''} \cdot \tau_{v, v'}^{v''}\right)
\left(u'' \otimes v''\right) \\
&\qquad=\sum_{\left(u, v\right) \in U \times V} \sum_{\left(u', v'\right) \in U \times V}
\left(\sum_{u'' \in U} \left(\alpha_{u} \cdot \beta_{u'} \cdot
\sigma_{u, u'}^{u''}\right) u''\right) \otimes
\left(\sum_{v'' \in V} \left(\alpha'_{v} \cdot \beta'_{v'} \cdot
\tau_{v, v'}^{v''}\right) v''\right) \\
&\qquad=\sum_{\left(u, v\right) \in U \times V} \sum_{\left(u', v'\right) \in U \times V}
\left(\left(\alpha_{u} \cdot \beta_{u'}\right)\left(u \cdot u'\right)\right) \otimes
\left(\left(\alpha'_{v} \cdot \beta'_{v'}\right)\left(v \cdot v'\right)\right) \\
&\qquad=\sum_{\left(u, v\right) \in U \times V} \sum_{\left(u', v'\right) \in U \times V}
\left(\left(\alpha_{u} u\right) \cdot \left(\beta_{u'} u'\right)\right) \otimes
\left(\left(\alpha'_{v} v\right) \cdot \left(\beta'_{v'} v'\right)\right) \\
&\qquad=\left(\left(\sum_{u \in U} \alpha_{u} u\right) \cdot
\left(\sum_{u \in U} \beta_{u} u\right)\right) \otimes
\left(\left(\sum_{v \in V} \alpha'_{v} v\right) \cdot
\left(\sum_{v \in V} \beta'_{v} v\right)\right) \\
&\qquad=\left(a \cdot a'\right) \otimes \left(b \cdot b'\right) \; .
\end{align*}

For every two elements of the basis $u \otimes v, u' \otimes v' \in W$, using \eqref{eq:multiplication_between_elementary_tensors_in_proof} and commutativity of $A$ and $A'$, $\left(u \otimes v\right) \cdot \left(u' \otimes v'\right) = \left(u \cdot u'\right) \otimes \left(v \cdot v'\right) = \left(u' \cdot u\right) \otimes \left(v' \cdot v\right) = \left(u' \otimes v'\right) \cdot \left(u \otimes v\right)$. From \citet[Theorem V.2.4]{hebisch1998algebraic}, this equation is sufficient condition for $A \otimes_{S} A'$ to be commutative.

For every $t, t' \in A \otimes_{S} A'$ and $\mu \in S$,
\begin{equation}
\left(\mu t\right) \cdot t' = t \cdot \left(\mu t'\right) = \mu \!\left(t \cdot t'\right)
\label{eq:scalar_multiplication_for_semialgebra}
\end{equation}
holds because, using the unique linear combinations of the basis $W$ for $t$ and $t'$, say, $t = \sum_{\left(u, v\right) \in U \times V} \rho_{u, v} \left(u \otimes v\right)$ and $t' = \sum_{\left(u, v\right) \in U \times V} \rho'_{u, v} \left(u \otimes v\right)$, and \eqref{eq:multiplication_of_tensor_product_of_semialgebras},
\begin{equation*}
\begin{aligned}
&\left(\mu t\right) \cdot t' \\
&\qquad=\left(\mu \left(\sum_{\left(u, v\right) \in U \times V}
                             \rho_{u, v} \left(u \otimes v\right)\right)\right) \cdot
        \left(\sum_{\left(u, v\right) \in U \times V} \rho'_{u, v} \left(u \otimes v\right)\right) \\
&\qquad=\left(\sum_{\left(u, v\right) \in U \times V}
                \left(\mu \cdot \rho_{u, v}\right)\left(u \otimes v\right)\right) \cdot
        \left(\sum_{\left(u, v\right) \in U \times V} \rho'_{u, v} \left(u \otimes v\right)\right) \\
&\qquad=\sum_{\left(u, v\right) \in U \times V} \sum_{\left(u', v'\right) \in U \times V}
        \sum_{\left(u'', v''\right) \in U \times V}
          \left(\mu \cdot \rho_{u, v} \cdot \rho'_{u', v'} \cdot
                \sigma_{u, u'}^{u''} \cdot \tau_{v, v'}^{v''}\right)
          \left(u'' \otimes v''\right) \\
&\qquad=\mu \sum_{\left(u, v\right) \in U \times V} \sum_{\left(u', v'\right) \in U \times V}
            \sum_{\left(u'', v''\right) \in U \times V}
              \left(\rho_{u, v} \cdot \rho'_{u', v'} \cdot
                    \sigma_{u, u'}^{u''} \cdot \tau_{v, v'}^{v''}\right)
              \left(u'' \otimes v''\right) \\
&\qquad=\mu \!\left(\left(\sum_{\left(u, v\right) \in U \times V}
                            \rho_{u, v} \left(u \otimes v\right)\right) \cdot
                    \left(\sum_{\left(u, v\right) \in U \times V}
                            \rho'_{u, v} \left(u \otimes v\right)\right)\right) \\
&\qquad=\mu \!\left(t \cdot t'\right)
\end{aligned}
\end{equation*}
Likewise, one easily obtains $t \cdot \left(\mu t'\right) = \mu \!\left(t \cdot t'\right)$.

For every element of $A \otimes_S A'$ written in the unique linear combinations of the basis $W$, say, $t = \sum_{(u, v) \in U \times V} \rho_{u, v} (u \otimes v)$, using \eqref{eq:scalar_multiplication_for_semialgebra} and \eqref{eq:multiplication_between_elementary_tensors_in_proof},
\begin{align*}
&(1_A \otimes 1_{A'}) \cdot t
 = (1_A \otimes 1_{A'}) \cdot
     \left(\sum_{(u, v) \in U \times V} \rho_{u, v} (u \otimes v)\right)
 = \sum_{(u, v) \in U \times V}
    \rho_{u, v} (\left(1_A \otimes 1_{A'}\right) \cdot (u \otimes v)) \\
&\qquad= \sum_{(u, v) \in U \times V} \rho_{u, v} (\left(1_A \cdot u\right) \otimes (1_{A'} \cdot v)) = \sum_{(u, v) \in U \times V} \rho_{u, v} (u \otimes v) = t \; .
\end{align*}
Therefore, $1_A \otimes 1_{A'}$ is the identity element of $A \otimes_S A'$, and $A \otimes_S A'$ is unital.

From \eqref{eq:multiplication_of_tensor_product_of_semialgebras}, the structure constants of $A \otimes_{S} A'$ with respect to the semialgebra basis $W$ is clearly $\omega_{u \otimes v, u' \otimes v'}^{u'' \otimes v''} = \sigma_{u, u'}^{u''} \cdot \tau_{v, v'}^{v''}$.

Since $A$ and $A'$ are associative semialgebras, by using \citet[Theorem V.2.4]{hebisch1998algebraic}, $\sum_{u''' \in U} \sigma_{u, u'}^{u'''} \cdot \sigma_{u''', u''}^{u''''} = \sum_{u''' \in U} \sigma_{u', u''}^{u'''} \cdot \sigma_{u, u'''}^{u''''}$ holds for every $u, u', u'', u'''' \in U$, and $\sum_{v''' \in V} \tau_{v, v'}^{v'''} \cdot \tau_{v''', v''}^{v''''} = \sum_{v''' \in V} \tau_{v', v''}^{v'''} \cdot \tau_{v, v'''}^{v''''}$ for every $v, v', v'', v'''' \in V$. Therefore,
\begin{align*}
& \sum_{\left(u''', v'''\right) \in U \times V}
    \omega_{u \otimes u, u' \otimes v'}^{u''' \otimes v'''} \cdot
    \omega_{u''' \otimes v''', u'' \otimes v''}^{u'''' \otimes v''''}
 =\sum_{\left(u''', v'''\right) \in U \times V}
    \sigma_{u, u'}^{u'''} \cdot \tau_{v, v'}^{v'''} \cdot
    \sigma_{u''', u''}^{u''''} \cdot \tau_{v''', v''}^{v''''} \\
&\qquad=\left(\sum_{u''' \in U}
                \sigma_{u, u'}^{u'''} \cdot \sigma_{u''', u''}^{u''''}\right) \cdot
        \left(\sum_{v''' \in V} \tau_{v, v'}^{v'''} \cdot \tau_{v''', v''}^{v''''}\right) \\
&\qquad=\left(\sum_{u''' \in U}
                \sigma_{u', u''}^{u'''} \cdot \sigma_{u, u'''}^{u''''}\right) \cdot
        \left(\sum_{v''' \in V} \tau_{v', v''}^{v'''} \cdot \tau_{v, v'''}^{v''''}\right) \\
&\qquad=\sum_{\left(u''', v'''\right) \in U \times V} \sigma_{u', u''}^{u'''} \cdot \tau_{v', v''}^{v'''} \cdot \sigma_{u''', u''}^{u''''} \cdot \tau_{v''', v''}^{v''''} = \sum_{\left(u''', v'''\right) \in U \times V} \omega_{u' \otimes v', u'' \otimes v''}^{u''' \otimes v'''} \cdot \omega_{u \otimes v, u''' \otimes v'''}^{u'''' \otimes v''''}
\end{align*}
holds for every $\left(u, v\right), \left(u', v'\right), \left(u'', v''\right), \left(u'''', v''''\right) \in U \times V$. Again by \citet[Theorem V.2.4]{hebisch1998algebraic}, this equation is sufficient condition for $A \otimes_{S} A'$ to be associative.

Therefore, $A \otimes_{S} A'$ is a commutative unital associative $S$-semialgebra.
\end{proof}

\section{Proof of Theorem \ref{thm:alpha_for_tensor_product_of_semialgebras}}

\begin{proof}
For every $\left(m, n\right), \left(m', n'\right) \in M \times N$, using the homomorphism of $f$ and $g$, and the equation \eqref{eq:multiplication_between_elementary_tensors}, we have
\begin{align*}
&\left(f \otimes g\right)\left(\left(m, n\right) \cdot \left(m, n'\right)\right)
       = \left(f \otimes g\right)\left(m \cdot m', n \cdot n'\right)
       = \left(f\!\left(m \cdot m'\right)\right) \otimes \left(g\!\left(n \cdot n'\right)\right) \\
&\qquad= \left(f\!\left(m\right) \cdot f\!\left(m'\right)\right) \otimes
         \left(g\!\left(m\right) \cdot g\!\left(n'\right)\right)
       = \left(f\!\left(m\right) \otimes g\!\left(n\right)\right) \cdot
         \left(f\!\left(m'\right) \otimes g\!\left(n'\right)\right) \\
&\qquad= \left(\left(f \otimes g\right)\left(m, n\right)\right) \cdot
         \left(\left(f \otimes g\right)\left(m', n'\right)\right) \; ,
\end{align*}
so $f \otimes g$ is a monoid homomorphism from $M \otimes N$ to the multiplicative monoid $(A \otimes_S A', \cdot, 1_A \otimes 1_{A'})$ of the tensor product of the semialgebras $A \otimes_S A'$. Thus the sextuple $\calG = (G, \op, M \times N, \phi \times \psi, A \otimes_S A', f \otimes g)$ specifies the $(f \otimes g)$-parametrized computation graph $(G, \op, A \otimes_S A', (f \otimes g) \circ (\phi \times \psi))$. Therefore, using Theorem \ref{thm:parametrized_alpha}, we have
\begin{equation*}
\begin{aligned}
\alpha_{\mathcal{G}}\!\left(t\right)
=& \sum_{\veci \in \bbNzero^n} c_{t, \veci}\!\left(\left(f \otimes g\right)
     \left(
       \left(\left(\left(\phi \times \psi\right)\left(s_{1}\right)\right)^{i_{1}}\right) \cdots
       \left(\left(\left(\phi \times \psi\right)\left(s_{n}\right)\right)^{i_{n}}\right)
     \right)\right) \\
=& \sum_{\veci \in \bbNzero^n} c_{t, \veci}\!\left(\left(f \otimes g\right)
     \left(
       \left(\phi\!\left(s_{1}\right), \psi\!\left(s_{1}\right)\right)^{i_{1}} \cdots
       \left(\phi\!\left(s_{n}\right), \psi\!\left(s_{n}\right)\right)^{i_{n}}
     \right)\right) \\
=& \sum_{\veci \in \bbNzero^n} c_{t, \veci}\!\left(\left(f \otimes g\right)
     \left(
       \left(\phi\!\left(s_{1}\right)\right)^{i_{1}} \cdots
       \left(\phi\!\left(s_{n}\right)\right)^{i_{n}},
       \left(\psi\!\left(s_{1}\right)\right)^{i_{1}} \cdots
       \left(\psi\!\left(s_{n}\right)\right)^{i_{n}}
     \right)\right) \\
=& \sum_{\veci \in \bbNzero^n} c_{t, \veci}\!\left(
     f\!\left(\left(\phi\!\left(s_{1}\right)\right)^{i_{1}} \cdots
              \left(\phi\!\left(s_{n}\right)\right)^{i_{n}}\right) \otimes
     g\!\left(\left(\psi\!\left(s_{1}\right)\right)^{i_{1}} \cdots
              \left(\psi\!\left(s_{n}\right)\right)^{i_{n}}\right)\right) \; .
\end{aligned}
\end{equation*}
Finally, Definition \ref{def:tensor_product_of_semimodules} guarantees the existence of $L$. In fact, a function $L$ that maps each element of $A \otimes_S A'$ written as the linear combination of the basis $\left\{u \otimes u'\right\}_{u \in U,\, u' \in U'}$, say, $\sum_{u \in U, u' \in U'} \rho_{u, u'} \!\left(u \otimes u'\right)$ where $\rho_{u, u'} \in S$, and $U$ and $U'$ are bases of $A$ and $A'$, respectively, to $\sum_{u \in U, u' \in U'} \rho_{u, u'} B\!\left(u, u'\right)$ is an $S$-homomorphism and satisfies
\begin{equation*}
L\!\left(\alpha_{\mathcal{G}}\!\left(t\right)\right) =
\sum_{\veci \in \bbNzero^n} c_{t, \veci}
B\!\left(f\!\left(\left(\phi\!\left(s_1\right)\right)^{i_1} \cdots
\left(\phi\!\left(s_n\right)\right)^{i_n}\right),
g\!\left(\left(\psi\!\left(s_1\right)\right)^{i_1} \cdots
\left(\psi\!\left(s_n\right)\right)^{i_n}\right)\right) \; .
\end{equation*}
\end{proof}

\section{Proof of Lemma \ref{lemma:zeroth_and_first_projection}}

\begin{proof}
For every element of $A \otimes_{S} \BC_{S}^{1}$ written as the unique linear combination of elements of the basis $\left\{u \otimes \hat{e}_{i}\right\}_{u \in U,\, i \in \left\{0, 1\right\}}$, say, $x = \sum_{u \in U} \sum_{i \in \left\{0, 1\right\}} \sigma_{u, i} \left(u \otimes \hat{e}_{i}\right)$, we have
\begin{equation*}
\begin{split}
x = \sum_{u \in U} \sum_{i \in \left\{0, 1\right\}} \sigma_{u, i} \left(u \otimes \hat{e}_{i}\right)
  = \sum_{u \in U} \sigma_{u, 0} \left(u \otimes \hat{e}_{0}\right)
  + \sum_{u \in U} \sigma_{u, 1} \left(u \otimes \hat{e}_{1}\right) \\
  = \sum_{u \in U} \left(\sigma_{u, 0} u\right) \otimes \hat{e}_{0}
  + \sum_{u \in U} \left(\sigma_{u, 1} u\right) \otimes \hat{e}_{1}
  = \mathsf{P}_{0}\!\left(x\right) \otimes \hat{e}_{0}
  + \mathsf{P}_{1}\!\left(x\right) \otimes \hat{e}_{1} \; .
\end{split}
\end{equation*}
\end{proof}

\section{Proof of Lemma \ref{lemma:linearity_of_projection}}

\begin{proof}
For every two elements of $A \otimes_S \BC_S^1$ written as linear combinations of elements of a basis $\{u \otimes \hat{e}_i\}_{u \in U,\, i \in \{0, 1\}}$, say, $x = \sum_{u \in U} \sum_{i \in \{0, 1\}} \sigma_{u, i}(u \otimes \hat{e}_i)$ and $y = \sum_{u \in U} \sum_{i \in \{0, 1\}} \tau_{u, i}(u \otimes \hat{e}_i)$, we have
\begin{align*}
&\mathsf{P}_{0}\!\left(x + y\right) \\
&\qquad=\mathsf{P}_{0}\!\left(
          \sum_{u \in U} \sum_{i \in \left\{0, 1\right\}}
            \left(\sigma_{u, i} + \tau_{u, i}\right) \left(u \otimes \hat{e}_{i}\right)
        \right) \\
&\qquad=\sum_{u \in U} \left(\sigma_{u, 0} + \tau_{u, 0}\right) u \\
&\qquad=\sum_{u \in U} \sigma_{u, 0} u + \sum_{u \in U} \tau_{u, 0} u \\
&\qquad=\mathsf{P}_{0}\!\left(
          \sum_{u \in U}
            \sum_{i \in \left\{0, 1\right\}}
              \sigma_{u, i}\!\left(u \otimes \hat{e}_{i}\right)\right) +
        \mathsf{P}_{0}\!\left(
          \sum_{u \in U}
            \sum_{i \in \left\{0, 1\right\}}
              \tau_{u, i}\!\left(u \otimes \hat{e}_{i}\right)\right) \\
&\qquad=\mathsf{P}_{0}\!\left(x\right) + \mathsf{P}_{0}\!\left(y\right) \; .
\end{align*}
Likewise, it is easy to show $\mathsf{P}_{1}\!\left(x + y\right) = \mathsf{P}_{1}\!\left(x\right) + \mathsf{P}_{1}\!\left(y\right)$ for every two elements $x, y \in A \otimes_{S} \BC_{S}^{1}$.
\end{proof}

\section{Proof of Lemma \ref{lemma:zeroth_projection_of_alpha}}

\begin{proof}
If $t \in \src\!\left(G\right)$, from Definition \ref{def:alpha}, we have $\mathsf{P}_{0}\!\left(\alpha_{\mathcal{G}}\!\left(t\right)\right) = \mathsf{P}_{0}\!\left(\xi\!\left(t\right)\right) = \left(\mathsf{P}_{0} \circ \xi\right)\left(t\right) = \alpha_{\mathcal{G}'}\!\left(t\right)$.

Consider the case where $t$ is an element of $V$ such that $t \notin \src\!\left(G\right)$. Assume the induction hypothesis that \eqref{eq:zeroth_projection_of_alpha} holds for every element of $E_{G}^{-}\!\left(v\right)$. If $\op\!\left(v\right) = \text{``$+$''}$ then, by Definition \ref{def:alpha}, Lemma \ref{lemma:linearity_of_projection}, and the induction hypothesis, we have
\begin{equation*}
\mathsf{P}_{0}\!\left(\alpha_{\mathcal{G}}\!\left(t\right)\right) =
\mathsf{P}_{0}\!\left(
  \sum_{e \in E_{G}^{-}\!\left(t\right)} \alpha_{\mathcal{G}}\!\left(e\right)\right) =
\sum_{e \in E_{G}^{-}\!\left(t\right)}
  \mathsf{P}_{0}\!\left(\alpha_{\mathcal{G}}\!\left(e\right)\right) =
\sum_{e \in E_{G}^{-}\!\left(t\right)} \alpha_{\mathcal{G}'}\!\left(t\right) \; .
\end{equation*}
Otherwise (i.e., $\op(t) = \text{``$\cdot$''}$), by using Definition \ref{def:alpha}, Lemma \ref{lemma:linearity_of_projection}, the linear combination of elements of a basis $\{u \otimes \hat{e}_i\}_{u \in U,\, i \in \{0, 1\}}$ of $A \otimes_S \BC_S^1$ for $\alpha_\calG(t)$, say, $\alpha_\calG(t) = \sum_{u \in U,\, i \in \{0, 1\}} \eta_{t, u, i}(u \otimes \hat{e}_i)$ where $\eta_{t, u, i} \in S$, and the induction hypothesis, and by noting that $x = \left(\mathsf{P}_{0}\!\left(x\right)\right) \otimes \hat{e}_{0} + \left(\mathsf{P}_{1}\!\left(x\right)\right) \otimes \hat{e}_{1}$, $\mathsf{P}_{0}\!\left(\left(\mathsf{P}_{0}\!\left(x\right)\right) \otimes \hat{e}_{0}\right) = \mathsf{P}_{0}\!\left(x\right)$, and $\mathsf{P}_{0}\!\left(\left(\mathsf{P}_{1}\!\left(x\right)\right) \otimes \hat{e}_{1}\right) = 0_{A}$ for every $x \in A \otimes_{S} \BC_{S}^{1}$, and $\hat{e}_{i_{1}} \cdots \hat{e}_{i_{n}} = \hat{e}_{0}$ if and only if $i_{j} = 0$ for all $j$, we have
\begin{equation*}
\begin{aligned}
\mathsf{P}_{0}\!\left(\alpha_{\mathcal{G}}\!\left(t\right)\right)
&=
\mathsf{P}_{0}\!\left(
  \prod_{e \in E_{G}^{-}\!\left(t\right)} \alpha_{\mathcal{G}}\!\left(e\right)\right) =
\mathsf{P}_{0}\!\left(
  \prod_{e \in E_{G}^{-}\!\left(t\right)}
    \sum_{u \in U,\, i \in \left\{0, 1\right\}}
      \eta_{e, u, i} \left(u \otimes \hat{e}_{i}\right)\right) \\
&=
\prod_{e \in E_{G}^{-}\!\left(t\right)} \left(
  \sum_{u \in U} \eta_{e, u, 0} \left(u \otimes \hat{e}_{0}\right)\right) =
\prod_{e \in E_{G}^{-}\!\left(t\right)}
  \mathsf{P}_{0}\!\left(\alpha_{\mathcal{G}}\!\left(e\right)\right) =
\prod_{e \in E_{G}^{-}\!\left(t\right)} \alpha_{\mathcal{G}'}\!\left(e\right) =
\alpha_{\mathcal{G}'}\!\left(t\right) \; .
\end{aligned}
\end{equation*}

Consider the case where $t$ is an element of $E$. Assume the induction hypothesis that \eqref{eq:zeroth_projection_of_alpha} holds for $\tail\!\left(t\right)$. Then, by using Definition \ref{def:alpha} and the induction hypothesis, we have $\mathsf{P}_{0}\!\left(\alpha_{\mathcal{G}}\!\left(t\right)\right) = \mathsf{P}_{0}\!\left(\alpha_{\mathcal{G}}\!\left(\tail\!\left(t\right)\right)\right) = \alpha_{\mathcal{G}'}\!\left(\tail\!\left(t\right)\right) = \alpha_{\mathcal{G}'}\!\left(t\right)$.

Therefore, we have proven that the equation \eqref{eq:zeroth_projection_of_alpha} holds for every $t \in V \cup E$ by induction on the finite dag $G$.
\end{proof}

\section{Proof of Lemma \ref{lemma:first_projection_of_prod_alpha}}

\begin{proof}
Let $a_{e, 0} = \mathsf{P}_{0}\!\left(\alpha_{\mathcal{G}}\!\left(e\right)\right)$ and $a_{e, 1} = \mathsf{P}_{1}\!\left(\alpha_{\mathcal{G}}\!\left(e\right)\right)$, and let $E_{G}^{-}\!\left(v\right) = \left\{e_{1}, \dots, e_{n}\right\}$. Note that, for the elements $\hat{e}_{0} = \left(1_{S}, 0_{S}\right), \hat{e}_{1} = \left(0_{S}, 1_{S}\right)$ of a basis of $\BC_{S}^{1}$, $\hat{e}_{i_1} \cdot \hat{e}_{i_2} \cdots \hat{e}_{i_n} = \hat{e}_{0}$ if and only if $i_{1} = i_{2} = \cdots = i_{n} = 0$, $\hat{e}_{i_1} \cdot \hat{e}_{i_2} \cdots \hat{e}_{i_n} = \hat{e}_{1}$ if and only if only one of $i_{j}$ is equal to $1$ and all the others are equal to $0$, and $\hat{e}_{i_1} \cdot \hat{e}_{i_2} \cdots \hat{e}_{i_n} = 0_{\BC_{S}^{1}}$ otherwise, because of the definition of the multiplication of $\BC_{S}^{1}$ (cf. Definition \ref{def:binomial_convolution_semiring}). By noting above-mentioned facts, and using the linear combination of elements of the basis $\left\{u, \hat{e}_{i}\right\}_{u \in U,\, i \in \left\{0, 1\right\}}$ for an element of $A \otimes_{S} \BC_{S}^{1}$, and the equation \eqref{eq:multiplication_between_elementary_tensors}, we have
\begin{align*}
&\prod_{e \in E_{G}^{-}\left(v\right)} \alpha_{\mathcal{G}}\!\left(e\right) \\
&\quad=\prod_{e \in E_{G}^{-}\left(v\right)}
          \left(a_{e, 0} \otimes \hat{e}_{0} + a_{e, 1} \otimes \hat{e}_{1}\right) \\
&\quad=\left(a_{e_1, 0} \otimes \hat{e}_{0}\right) \cdot \left(a_{e_2, 0} \otimes \hat{e}_{0}\right)
        \cdots \left(a_{e_{n}, 0} \otimes \hat{e}_{0}\right) \\
&\qquad\phantom{=}\qquad+
        \left(a_{e_1, 0} \otimes \hat{e}_{0}\right) \cdot \left(a_{e_2, 0} \otimes \hat{e}_{0}\right)
        \cdots \left(a_{e_{n}, 1} \otimes \hat{e}_{1}\right) \\
&\quad\phantom{=}\qquad+ \cdots \\
&\quad\phantom{=}\qquad+
        \left(a_{e_1, 1} \otimes \hat{e}_{1}\right) \cdot \left(a_{e_2, 1} \otimes \hat{e}_{1}\right)
        \cdots \left(a_{e_{n}, 1} \otimes \hat{e}_{1}\right) \\
&\quad=\left(a_{e_{0}, 0} \cdot a_{e_{1}, 0} \cdots a_{e_{n}, 0}\right)
          \otimes \left(\hat{e}_{0} \cdot \hat{e}_{0} \cdots \hat{e}_{0}\right) \\
&\quad\phantom{=}\qquad+
        \left(a_{e_{0}, 0} \cdot a_{e_{1}, 0} \cdots a_{e_{n}, 1}\right)
          \otimes \left(\hat{e}_{0} \cdot \hat{e}_{0} \cdots \hat{e}_{1}\right) \\
&\quad\phantom{=}\qquad+ \cdots \\
&\quad\phantom{=}\qquad+
        \left(a_{e_{0}, 1} \cdot a_{e_{1}, 1} \cdots a_{e_{n}, 1}\right)
          \otimes \left(\hat{e}_{1} \cdot \hat{e}_{1} \cdots \hat{e}_{1}\right) \\
&\quad=\left(a_{e_{0}, 0} \cdot a_{e_{1}, 0} \cdots a_{e_{n}, 0}\right) \otimes \hat{e}_{0} \\
&\quad\phantom{=}\qquad+
        \left(a_{e_{1}, 1} \cdot a_{e_{1}, 0} \cdots a_{e_{n}, 0} +
              a_{e_{1}, 0} \cdot a_{e_{1}, 1} \cdots a_{e_{n}, 0} +
              \cdots +
              a_{e_{1}, 0} \cdot a_{e_{1}, 0} \cdots a_{e_{n}, 1}\right) \otimes \hat{e}_{1} \\
&\quad=\left(\prod_{e \in E_{G}^{-}\left(v\right)} a_{e, 0}\right) \otimes \hat{e}_{0} +
        \left(
          \sum_{e \in E_{G}^{-}\left(v\right)} a_{e, 1} \cdot \left(
            \prod_{e' \in E_{G}^{-}\left(v\right) \setminus \left\{e\right\}} a_{e', 0}\right)
        \right) \otimes \hat{e}_{1} \\
&\quad=\left(
          \prod_{e \in E_{G}^{-}\left(v\right)}
            \mathsf{P}_{0}\left(\alpha_{\mathcal{G}}\!\left(e\right)\right)\right) \otimes \hat{e}_{0} +
        \left(
          \sum_{e \in E_{G}^{-}\left(v\right)}
            \mathsf{P}_{1}\!\left(\alpha_{\mathcal{G}}\!\left(e\right)\right) \cdot \left(
              \prod_{e' \in E_{G}^{-}\left(v\right) \setminus \left\{e\right\}}
                \mathsf{P}_{0}\!\left(\alpha_{\mathcal{G}}\!\left(e'\right)\right)\right)
        \right) \otimes \hat{e}_{1} \; .
\end{align*}
Therefore, we finally obtain
\begin{equation*}
\mathsf{P}_{1}\!\left(
  \prod_{e \in E_{G}^{-}\left(v\right)} \alpha_{\mathcal{G}}\!\left(e\right)\right) =
\sum_{e \in E_{G}^{-}\left(v\right)} \mathsf{P}_{1}\!\left(\alpha_{\mathcal{G}}\!\left(e\right)\right)
  \cdot \left(
    \prod_{e' \in E_{G}^{-}\left(v\right) \setminus \left\{e\right\}}
      \mathsf{P}_{0}\!\left(\alpha_{\mathcal{G}}\!\left(e'\right)\right)\right) \; .
\end{equation*}
\end{proof}

\section{Proof of Lemma \ref{lemma:covering_antichain_cutset}}

\begin{proof}
Let $x \nparallel y$ denote that $x$ and $y$ are comparable for every $x, y \in \mathcal{O}$.

Since $C \neq \snk\!\left(G\right)$, it is obvious that $\uparrow\!\!C \setminus C \neq \emptyset$. Since $G$ is finite (and thus so are $\mathcal{O}$ and $\uparrow\!\!C \setminus C$), there exists at least one minimal element $y$ in $\uparrow\!\!C \setminus C$. Because a maximal chain of $\mathcal{O}$ passing through $y$ intersects the antichain cutset $C$, there exists at least one element $x \in C$ such that $x < y$. If there exists $z \in V \cup E$ such that $x < z < y$ then $z$ is an element of $\uparrow\!\!C \setminus C$, but this contradicts the fact that $y$ is a minimal element in $\uparrow\!\!C \setminus C$. Therefore $x \prec y$ is shown by contradiction, and one can obtain $x \in C$ such that there exists $y \in V \cup E$ satisfying $x \prec y$ and $y$ is a minimal element of $\uparrow\!\!C \setminus C$.

\paragraph{1)}
If $x \in V$, it is obvious that $D' = E_{G}^{+}\!\left(x\right)$ and $D = \left\{x\right\}$. For every $d' \in D'$, the covered set of $d'$ is the singleton set $\left\{x\right\}$, so the covered set of $d'$ is a subset of $C$. $C' = \left(C \cup D'\right) \setminus D$ is clearly an antichain cutset. $C < C'$ is obvious. By using Definitions \ref{def:alpha} and \ref{def:beta}, and noting that $\tail\!\left(e\right) = x$ holds in the summand of $\sum_{e \in E_{G}^{+}\!\left(x\right)}$, we have
\begin{align*}
&\sum_{e \in E_{G}^{+}\!\left(x\right)}
  \mathsf{P}_{1}\!\left(\alpha_{\mathcal{G}}\!\left(e\right)\right) \cdot
  \beta_{\mathcal{G}}\!\left(e\right) =
 \sum_{e \in E_{G}^{+}\!\left(x\right)}
   \mathsf{P}_{1}\!\left(\alpha_{\mathcal{G}}\!\left(\tail\!\left(e\right)\right)\right) \cdot
   \beta_{\mathcal{G}}\!\left(e\right) \\
&\qquad=\sum_{e \in E_{G}^{+}\!\left(x\right)}
          \mathsf{P}_{1}\!\left(\alpha_{\mathcal{G}}\!\left(x\right)\right) \cdot
          \beta_{\mathcal{G}}\!\left(e\right)
       =\mathsf{P}_{1}\!\left(\alpha_{\mathcal{G}}\!\left(x\right)\right)
          \cdot \sum_{e \in E_{G}^{+}\!\left(x\right)} \beta_{\mathcal{G}}\!\left(e\right)
       =\mathsf{P}_{1}\!\left(\alpha_{\mathcal{G}}\!\left(x\right)\right) \cdot
        \beta_{\mathcal{G}}\!\left(x\right)
\end{align*}
By noting $D \subseteq C$ and $C \cap D' = \emptyset$, and using the above equation, we finally obtain
\begin{align*}
&\sum_{c \in C'} \mathsf{P}_{1}\!\left(\alpha_{\mathcal{G}}\!\left(c\right)\right) \cdot
                 \beta_{\mathcal{G}}\!\left(c\right) \\
&\qquad=
  \sum_{c \in C} \mathsf{P}_{1}\!\left(\alpha_{\mathcal{G}}\!\left(c\right)\right)
                   \cdot \beta_{\mathcal{G}}\!\left(c\right) +
  \sum_{d' \in D'} \mathsf{P}_{1}\!\left(\alpha_{\mathcal{G}}\!\left(d'\right)\right)
                     \cdot \beta_{\mathcal{G}}\!\left(d'\right) -
  \sum_{d \in D} \mathsf{P}_{1}\!\left(\alpha_{\mathcal{G}}\!\left(d\right)\right)
                   \cdot \beta_{\mathcal{G}}\!\left(d\right) \\
&\qquad=
  \sum_{c \in C} \mathsf{P}_{1}\!\left(\alpha_{\mathcal{G}}\!\left(c\right)\right)
                   \cdot \beta_{\mathcal{G}}\!\left(c\right) +
  \sum_{e \in E_{G}^{+}\!\left(x\right)}
    \mathsf{P}_{1}\!\left(\alpha_{\mathcal{G}}\!\left(e\right)\right)
      \cdot \beta_{\mathcal{G}}\!\left(e\right) -
  \mathsf{P}_{1}\!\left(\alpha_{\mathcal{G}}\!\left(x\right)\right)
    \cdot \beta_{\mathcal{G}}\!\left(x\right) \\
&\qquad=
  \sum_{c \in C} \mathsf{P}_{1}\!\left(\alpha_{\mathcal{G}}\!\left(c\right)\right)
    \cdot \beta_{\mathcal{G}}\!\left(c\right) \; .
\end{align*}

\paragraph{2)}
If $x \in E$, it is obvious that $D' = \left\{\head\!\left(x\right)\right\}$ and $D = E_{G}^{-}\!\left(\head\!\left(x\right)\right)$.

The covered set by an element of $D'$ is $D = E_{G}^{-}\!\left(\head\!\left(x\right)\right)$. Assume that there exists $d \in D$ such that $d \notin C$. Because $C$ intersects every maximal chain of $\mathcal{O}$, there exists $c \in C$ such that $c \nparallel d$. Further assume that $d < c$ then $d \prec \head\!\left(x\right) \leq c$ thus $x \prec d' \leq c$, but this contradicts the fact that $C$ is an antichain. Further assume that $c < d$ then $c < d \prec \head\!\left(x\right)$ but this also contradicts the fact that $\head\!\left(x\right)$ is a minimal element of $\uparrow\!\!C \setminus C$. Therefore we conclude that $d \in D \implies d \in C$ by contradiction and thus $D \subseteq C$.

Because $C$ intersects every maximal chain of $\mathcal{O}$, and every maximal chain of $\mathcal{O}$ intersecting $D$ also intersects $D'$ (i.e., passes through $\head\!\left(x\right)$), $C' = \left(C \cup D'\right) \setminus D$ intersects every maximal chain of $\mathcal{O}$. Assume that $C' = \left(C \cup D'\right) \setminus D$ is not an antichain. Then, because $C$ is an antichain, and $D'$ is also an antichain since it is a singleton set, there exists $x' \in C$ and $y' \in D'$ such that $x' \neq y'$ and $x' \nparallel y'$. $y' < x'$ is not possible, because the covering set of $y'$ is the singleton set consisting only of $y$, and it follows that $x \prec y \leq x'$ holds, but this contradicts the fact that $C$ is an antichain. $x' < y'$ is neither possible, because it follows that $x' < y' \prec y$, but this contradicts the fact that $y$ is a minimal element of $\uparrow\!\!C \setminus C$. Hence, we have proven $C'$ is an antichain by contradiction. Therefore, $C' = \left(C \cup D'\right) \setminus D$ is an antichain cutset. $C < C'$ is obvious.

\paragraph{2-a)}
If $\op\!\left(\head\!\left(x\right)\right) = \text{``$+$''}$, by using Definitions \ref{def:alpha} and \ref{def:beta}, and Lemma \ref{lemma:linearity_of_projection}, and noting that $\head\!\left(x\right) = \head\!\left(e\right)$ holds in the summand of $\sum_{e \in E_{G}^{-}\!\left(\head\!\left(x\right)\right)}$, we have
\begin{align*}
&\mathsf{P}_{1}\!\left(\alpha_{\mathcal{G}}\!\left(\head\!\left(x\right)\right)\right)
  \cdot \beta_{\mathcal{G}}\!\left(\head\!\left(x\right)\right)
 = \mathsf{P}_{1}\!\left(
     \sum_{e \in E_{G}^{-}\!\left(\head\left(x\right)\right)}
       \alpha_{\mathcal{G}}\!\left(e\right)\right)
     \cdot \beta_{\mathcal{G}}\!\left(\head\!\left(x\right)\right) \\
&\qquad= \left(\sum_{e \in E_{G}^{-}\!\left(\head\left(x\right)\right)}
                 \mathsf{P}_{1}\!\left(\alpha_{\mathcal{G}}\!\left(e\right)\right)\right)
         \cdot \beta_{\mathcal{G}}\!\left(\head\!\left(x\right)\right)
 = \sum_{e \in E_{G}^{-}\!\left(\head\left(x\right)\right)} \left(
     \mathsf{P}_{1}\!\left(\alpha_{\mathcal{G}}\!\left(e\right)\right)
     \cdot \beta_{\mathcal{G}}\!\left(\head\!\left(x\right)\right)\right) \\
&\qquad= \sum_{e \in E_{G}^{-}\!\left(\head\left(x\right)\right)} \left(
           \mathsf{P}_{1}\!\left(\alpha_{\mathcal{G}}\!\left(e\right)\right)
         \cdot \beta_{\mathcal{G}}\!\left(\head\!\left(e\right)\right)\right)
 = \sum_{e \in E_{G}^{-}\!\left(\head\left(x\right)\right)}
     \mathsf{P}_{1}\!\left(\alpha_{\mathcal{G}}\!\left(e\right)\right)
     \cdot \beta_{\mathcal{G}}\!\left(e\right)
\end{align*}
By noting $D \subseteq C$ and $C \cap D' = \emptyset$, and using the above equation, we finally obtain
\begin{equation*}
\begin{aligned}
&\sum_{c \in C'} \mathsf{P}_{1}\!\left(\alpha_{\mathcal{G}}\!\left(c\right)\right)
                   \cdot \beta_{\mathcal{G}}\!\left(c\right) \\
&\qquad=\sum_{c \in C} \mathsf{P}_{1}\!\left(\alpha_{\mathcal{G}}\!\left(c\right)\right)
                         \cdot \beta_{\mathcal{G}}\!\left(c\right) +
        \sum_{d' \in D'} \mathsf{P}_{1}\!\left(\alpha_{\mathcal{G}}\!\left(d'\right)\right)
                           \cdot \beta_{\mathcal{G}}\!\left(d'\right) -
        \sum_{d \in D} \mathsf{P}_{1}\!\left(\alpha_{\mathcal{G}}\!\left(d\right)\right)
                         \cdot \beta_{\mathcal{G}}\!\left(d\right) \\
&\qquad=\sum_{c \in C} \mathsf{P}_{1}\!\left(\alpha_{\mathcal{G}}\!\left(c\right)\right)
                         \cdot \beta_{\mathcal{G}}\!\left(c\right) \\
&\qquad\qquad+\mathsf{P}_{1}\!\left(
                \alpha_{\mathcal{G}}\!\left(\head\!\left(x\right)\right)\right)
                \cdot \beta_{\mathcal{G}}\!\left(\head\!\left(x\right)\right) -
              \sum_{e \in E_{G}^{-}\!\left(\head\!\left(x\right)\right)}
                \mathsf{P}_{1}\!\left(\alpha_{\mathcal{G}}\!\left(e\right)\right)
                \cdot \beta_{\mathcal{G}}\!\left(e\right) \\
&\qquad=\sum_{c \in C} \mathsf{P}_{1}\!\left(\alpha_{\mathcal{G}}\!\left(c\right)\right)
                        \cdot \beta_{\mathcal{G}}\!\left(c\right) \; .
\end{aligned}
\end{equation*}

\paragraph{2-b)}
If $\op\!\left(\head\!\left(x\right)\right) = \text{``$\cdot$''}$, by using Definitions \ref{def:alpha} and \ref{def:beta}, and Lemma \ref{lemma:first_projection_of_prod_alpha}, and noting that $\head\!\left(x\right) = \head\!\left(e\right)$ holds in the summand of $\sum_{e \in E_{G}^{-}\!\left(\head\!\left(x\right)\right)}$, we have
\begin{align*}
&\mathsf{P}_{1}\!\left(\alpha_{\mathcal{G}}\!\left(\head\!\left(x\right)\right)\right)
   \cdot \beta_{\mathcal{G}}\!\left(\head\!\left(x\right)\right)
 = \mathsf{P}_{1}\!\left(
     \prod_{e \in E_{G}^{-}\!\left(\head\left(x\right)\right)}
       \alpha_{\mathcal{G}}\!\left(e\right)\right)
     \cdot \beta_{\mathcal{G}}\!\left(\head\!\left(x\right)\right) \\
&\qquad= \left(
           \sum_{e \in E_{G}^{-}\!\left(\head\left(x\right)\right)}
             \mathsf{P}_{1}\!\left(\alpha_{\mathcal{G}}\!\left(e\right)\right)
             \cdot \left(
               \prod_{e' \in E_{G}^{-}\!\left(\head\left(x\right)\right)
                               \setminus \left\{e\right\}}
                 \mathsf{P}_{0}\!\left(\alpha_{\mathcal{G}}\!\left(e'\right)\right)
             \right)
         \right) \cdot \beta_{\mathcal{G}}\!\left(\head\!\left(x\right)\right) \\
&\qquad= \sum_{e \in E_{G}^{-}\!\left(\head\left(x\right)\right)}\left(
           \mathsf{P}_{1}\!\left(\alpha_{\mathcal{G}}\!\left(e\right)\right)
           \cdot \left(
             \prod_{e' \in E_{G}^{-}\!\left(\head\left(x\right)\right)
                             \setminus \left\{e\right\}}
               \mathsf{P}_{0}\!\left(\alpha_{\mathcal{G}}\!\left(e'\right)\right)
           \right) \cdot \beta_{\mathcal{G}}\!\left(\head\!\left(x\right)\right)\right) \\
&\qquad= \sum_{e \in E_{G}^{-}\!\left(\head\left(x\right)\right)}\left(
           \mathsf{P}_{1}\!\left(\alpha_{\mathcal{G}}\!\left(e\right)\right)
           \cdot \left(
             \prod_{e' \in E_{G}^{-}\!\left(\head\left(e\right)\right)
                             \setminus \left\{e\right\}}
               \mathsf{P}_{0}\!\left(\alpha_{\mathcal{G}}\!\left(e'\right)\right)
           \right) \cdot \beta_{\mathcal{G}}\!\left(\head\!\left(e\right)\right)\right) \\
&\qquad= \sum_{e \in E_{G}^{-}\!\left(\head\left(x\right)\right)}
           \mathsf{P}_{1}\!\left(\alpha_{\mathcal{G}}\!\left(e\right)\right)
           \cdot \beta_{\mathcal{G}}\!\left(e\right) \; .
\end{align*}
By noting $D \subseteq C$ and $C \cap D' = \emptyset$, and using the above equation, we finally obtain
\begin{equation*}
\begin{aligned}
&\sum_{c \in C'} \mathsf{P}_{1}\!\left(\alpha_{\mathcal{G}}\!\left(c\right)\right)
                   \cdot \beta_{\mathcal{G}}\!\left(c\right) \\
&\qquad=\sum_{c \in C} \mathsf{P}_{1}\!\left(\alpha_{\mathcal{G}}\!\left(c\right)\right)
                         \cdot \beta_{\mathcal{G}}\!\left(c\right) +
        \sum_{d' \in D'} \mathsf{P}_{1}\!\left(\alpha_{\mathcal{G}}\!\left(d'\right)\right)
                           \cdot \beta_{\mathcal{G}}\!\left(d'\right) -
        \sum_{d \in D} \mathsf{P}_{1}\!\left(\alpha_{\mathcal{G}}\!\left(d\right)\right)
                         \cdot \beta_{\mathcal{G}}\!\left(d\right) \\
&\qquad=\sum_{c \in C} \mathsf{P}_{1}\!\left(\alpha_{\mathcal{G}}\!\left(c\right)\right)
                         \cdot \beta_{\mathcal{G}}\!\left(c\right) \\
&\qquad\qquad+\mathsf{P}_{1}\!\left(
                \alpha_{\mathcal{G}}\!\left(\head\!\left(x\right)\right)\right)
              \cdot \beta_{\mathcal{G}}\!\left(\head\!\left(x\right)\right) -
              \sum_{e \in E_{G}^{-}\!\left(\head\!\left(x\right)\right)}
                \mathsf{P}_{1}\!\left(\alpha_{\mathcal{G}}\!\left(e\right)\right)
                  \cdot \beta_{\mathcal{G}}\!\left(e\right) \\
&\qquad=\sum_{c \in C} \mathsf{P}_{1}\!\left(\alpha_{\mathcal{G}}\!\left(c\right)\right)
                         \cdot \beta_{\mathcal{G}}\!\left(c\right) \; .
\end{aligned}
\end{equation*}
\end{proof}

\section{Proof of Theorem \ref{thm:backward_invariants}}

\begin{proof}
Let $C$ be an antichain cutset of $\mathcal{O}$. Then, by repeating the construction of a covering antichain cutset described in Lemma \ref{lemma:covering_antichain_cutset}, one can obtain a strictly ascending chain of the lattice of antichain cutsets $\left(\AC\left(\mathcal{O}\right), \leq\right)$ starting at $C$, say, $C = C_1 < C_2 < \cdots < C_n = \snk\!\left(G\right)$ where $C_{i + 1}$ is a covering antichain cutset of $C_i$ for $i = 1, \dots, n - 1$, and, as it is written, this ascending chain eventually terminates at $C_n = \snk\!\left(G\right)$ since $G$ is finite. Moreover, again by Lemma \ref{lemma:covering_antichain_cutset}, we have $\sum_{c \in C_i} \mathsf{P}_{1}\!\left(\alpha_{\mathcal{G}}\!\left(c\right)\right) \cdot \beta_{\mathcal{G}}\!\left(c\right) = \sum_{c \in C_{i + 1}} \mathsf{P}_{1}\!\left(\alpha_{\mathcal{G}}\!\left(c\right)\right) \cdot \beta_{\mathcal{G}}\!\left(c\right)$ for $i = 1, \dots, n - 1$. By induction on the chain $C = C_1 < \cdots < C_n = \snk\!\left(G\right)$, one obtains $\sum_{c \in C} \mathsf{P}_{1}\!\left(\alpha_{\mathcal{G}}\!\left(c\right)\right) \cdot \beta_{\mathcal{G}}\!\left(c\right) = \sum_{c \in \snk\!\left(G\right)} \mathsf{P}_{1}\!\left(\alpha_{\mathcal{G}}\!\left(c\right)\right) \cdot \beta_{\mathcal{G}}\!\left(c\right)$. The above discussion can be repeated on another antichain cutset $C' \in \AC\!\left(\mathcal{O}\right)$. Therefore, we finally obtain $\sum_{c \in C} \mathsf{P}_{1}\!\left(\alpha_{\mathcal{G}}\!\left(c\right)\right) \cdot \beta_{\mathcal{G}}\!\left(c\right) = \sum_{c \in \snk\!\left(G\right)} \mathsf{P}_{1}\!\left(\alpha_{\mathcal{G}}\!\left(c\right)\right) \cdot \beta_{\mathcal{G}}\!\left(c\right) = \sum_{c \in C'} \mathsf{P}_{1}\!\left(\alpha_{\mathcal{G}}\!\left(a\right)\right) \cdot \beta_{\mathcal{G}}\!\left(c\right)$. It is obvious that $\src\!\left(G\right)$ and $\snk\!\left(G\right)$ are an antichain cutset of $\mathcal{O}$, so $\sum_{c \in \src\!\left(G\right)} \mathsf{P}_{1}\!\left(\xi\!\left(c\right)\right) \cdot \beta_{\mathcal{G}}\!\left(c\right) = \sum_{c \in \src\!\left(G\right)} \mathsf{P}_{1}\!\left(\alpha_{\mathcal{G}}\!\left(c\right)\right) \cdot \beta_{\mathcal{G}}\!\left(c\right) = \sum_{c \in \snk\!\left(G\right)} \mathsf{P}_{1}\!\left(\alpha_{\mathcal{G}}\!\left(c\right)\right) \cdot \beta_{\mathcal{G}}\!\left(c\right) = \sum_{c \in \snk\!\left(G\right)} \mathsf{P}_{1}\!\left(\alpha_{\mathcal{G}}\!\left(c\right)\right) = \mathsf{P}_{1}\!\left(\sum_{c \in \snk\!\left(G\right)} \alpha_{\mathcal{G}}\!\left(c\right)\right)$ clearly holds.
\end{proof}

\bibliography{azuma17a}

\end{document}